\documentclass{article}

    \PassOptionsToPackage{numbers, compress}{natbib}

     \usepackage[final]{neurips_2022}

\usepackage[utf8]{inputenc} 
\usepackage[T1]{fontenc}    
\usepackage{url}            
\usepackage{booktabs}       
\usepackage{amsfonts}       
\usepackage{nicefrac}       
\usepackage{microtype}      
\usepackage{xcolor}         
\usepackage{amsmath}
\usepackage{amsthm,amsfonts,amssymb}
\usepackage{enumerate}
\usepackage{graphicx,bm,amsthm,bbm}
\usepackage{multirow}
\usepackage[colorlinks, anchorcolor=white, linkcolor=black, urlcolor=black, citecolor=black]{hyperref}

\usepackage{algorithm,algorithmicx}
\usepackage[noend]{algpseudocode}

\usepackage{amsmath,amsfonts,bm}

\def\eqref#1{equation~\ref{#1}}

\def\1{\bm{1}}

\newcommand{\diff}{\mathrm{d}}

\def\vzero{{\bm{0}}}

\def\mI{{\bm{I}}}

\DeclareMathAlphabet{\mathsfit}{\encodingdefault}{\sfdefault}{m}{sl}
\SetMathAlphabet{\mathsfit}{bold}{\encodingdefault}{\sfdefault}{bx}{n}

\def\sR{{\mathbb{R}}}

\newcommand{\E}{\mathbb{E}}

\newcommand{\R}{\mathbb{R}}

\usepackage{amssymb}
\usepackage{amsthm}
\usepackage{mathrsfs}
\usepackage{blindtext}
\newtheorem{theorem}{Theorem}[section]
\newtheorem{lemma}[theorem]{Lemma}
\newtheorem{definition}[theorem]{Definition}

\newtheorem{remark}[theorem]{Remark}

\usepackage{bbding}

\title{Is $L^2$ Physics-Informed Loss Always Suitable for Training Physics-Informed Neural Network?}

\author{%
Chuwei Wang$^{1*}$,\ Shanda Li$^{2,5}$\thanks{Equal contribution.}\ , \textbf{Di He$^{3\dagger}$}, 
\ \textbf{Liwei Wang$^{3,4}$}\thanks{Correspondence to: Liwei Wang <\texttt{wanglw@pku.edu.cn}> and Di He <\texttt{dihe@pku.edu.cn}>.}\ \\
$^1$School of Mathematical Sciences, Peking University \\
$^2$Machine Learning Department, School of Computer Science, Carnegie Mellon University\\
\textls[-39]{\spaceskip=0.145em\relax $^3$National Key Laboratory of General Artificial Intelligence,}\\
{ 
School of Intelligence Science and Technology, Peking University}\\
$^4$ Center for Data Science, Peking University $\ \ $ 
$^5$ Zhejiang Lab\\
\texttt{chuwei.wang@pku.edu.cn, shandal@cs.cmu.edu}\\
\texttt{dihe@pku.edu.cn, wanglw@pku.edu.cn} \\
}

\begin{document}

\maketitle

\begin{abstract}
The Physics-Informed Neural Network (PINN) approach is a new and promising way to solve partial differential equations using deep learning. The $L^2$ Physics-Informed Loss is the de-facto standard in training Physics-Informed Neural Networks. In this paper, we challenge this common practice by investigating the relationship between the loss function and the approximation quality of the learned solution. In particular, we leverage the concept of stability in the literature of partial differential equation to study the asymptotic behavior of the learned solution as the loss approaches zero. With this concept, we study an important class of high-dimensional non-linear PDEs in optimal control, the Hamilton-Jacobi-Bellman (HJB) Equation, and prove that for general $L^p$ Physics-Informed Loss, a wide class of HJB equation is stable only if $p$ is sufficiently large. Therefore, the commonly used $L^2$ loss is not suitable for training PINN on those equations, while $L^{\infty}$ loss is a better choice. Based on the theoretical insight, we develop a novel PINN training algorithm to minimize the $L^{\infty}$ loss for HJB equations which is in a similar spirit to adversarial training. The effectiveness of the proposed algorithm is empirically demonstrated through experiments. Our code is released at \texttt{https://github.com/LithiumDA/L\_inf-PINN}.
\end{abstract}

\section{Introduction}
Recently, with the explosive growth of available data and computational resources, there have been growing interests in developing machine learning approaches to solve partial differential equations (PDEs)~\cite{khoo2017solving,han2018solving,sirignano2018dgm,raissi2019physics}. One seminal work in this direction is the Physics-Informed Neural Network (PINN) approach~\cite{raissi2019physics} which parameterizes the PDE's solution as a neural network. By defining differentiable loss functionals that measure how well the model fits the PDE and boundary conditions, the network parameters can be efficiently optimized using gradient-based approaches. $L^2$ distance is one of the most popularly used measures, which calculates the $L^2$ norm of the PDE and boundary residual on the domain and boundary, respectively. Previous works demonstrated that PINN could solve a wide range of PDE problems using the $L^2$ Physics-Informed Loss, such as Poisson equation, Burgers' equation, and Navier-Stokes equation \cite{raissi2019physics,de2022error}.

Although previous works empirically demonstrated promising results using $L^2$ Physics-Informed Loss, we argue the plausibility of using this loss for (high-dimensional) non-linear PDE problems. We know the trivial fact that the learned solution will equal the exact solution when its $L^2$ loss equals zero. However, the quality of a learned solution with a small but non-zero loss, which is a more realistic scenario in practice, remains unknown to have any approximation guarantees. In this work, we aim at answering a fundamental question:
\begin{center}
    \emph{Can we guarantee that a learned solution with a small Physics-Informed Loss always corresponds to a good approximator of the exact solution?}
\end{center}
To thoroughly investigate the problem,
we advocate analyzing the stability of PDE \cite{evans1998partial} in the PINN framework. Stability characterizes the asymptotic behavior of the distance between the learned solution and the exact solution when the Physics-Informed Loss approaches zero. If the PDE is not stable with respect to certain loss functions, we may not obtain good approximate solutions by minimizing the loss. To show the strength of the theory, we perform a comprehensive study on the stability of an important class of high-dimensional non-linear PDEs in optimal control, the Hamilton-Jacobi-Bellman (HJB) equation, which establishes a necessary and sufficient condition for a control's optimality with regard to the cost function. Interestingly, we prove that for general $L^p$ Physics-Informed Loss, the HJB equation is stable only if $p$ is sufficiently large. This finding suggests that the most widely used $L^2$ loss may not be suitable for training PINN on HJB equations as the learned solution can be arbitrarily distant from the exact solution. Empirical observation verifies the theoretical results.

We further show our theory can serve as a principled way to design loss functions for training PINN. For the high-dimensional HJB equation we target in the paper, the theoretical result suggests that $L^{\infty}$ loss may be a better choice to learn approximate solutions. Motivated by this insight, we propose a new algorithm for training PINN, which adopts a min-max optimization procedure to minimize the $L^{\infty}$ loss. Our approach resembles the well-known adversarial training framework. In each iteration, we first fix the network parameters and learn adversarial data points to approximate $L^{\infty}$ loss, and then optimize the network parameters to minimize the loss. When the training finishes, the learned network will converge to a solution with small $L^{\infty}$ losses and is close to the exact solution. We conduct experiments to demonstrate the effectiveness of the proposed algorithm. All empirical results show that our method can indeed learn accurate solutions for HJB equations and is much better than several baseline methods.

The contribution of the paper is summarized as follows.
\begin{itemize}
    \item We make the first step towards theoretically studying the loss design in PINN, and formally introduce the concept of stability in the literature of PDE to characterize the quality of a learned solution with small but non-zero Physics-Informed Loss. 
    
    \item We provide rigorous investigations on an important class of high-dimensional non-linear PDEs in optimal control, the HJB equation. Our results suggest that the widely used $L^2$ loss is not a suitable choice for training PINN on HJB equations.
    
    \item Based on the theoretical insight, we develop a novel PINN training algorithm to minimize the $L^{\infty}$ loss for HJB equations. We empirically demonstrate that the proposed algorithm can significant improve the accuracy of PINN in solving the optimal control problems.
\end{itemize}

\section{Related Works}
\label{sec:related}
Physics-Informed Neural Network approaches \cite{sirignano2018dgm, raissi2019physics} learn to find parametric solutions to satisfy equations and boundary conditions with gradient descent. There has been a notable scarcity of papers that rigorously justify why PINNs work. Important works include \cite{shin2020error}, which prove the convergence of PINN for second-order elliptic and parabolic equations. In \cite{lu2021machine}, the authors study the statistical limit of learning a PDE solution from sampled observations for elliptic equations. In \cite{DBLP:journals/corr/abs-2203-09346}, the convergence of PINN with $L^2$ loss is established for Navier-Stokes equations. At the same time, several works observed different failure modes for training PINN in other PDE problems. In \cite{krishnapriyan2021characterizing}, researchers discover that PINN sometimes fails to learn accurate solutions to a class of convection and reaction equations. Their analysis shows that this may be attributed to the complicated loss landscape. \cite{wang2021understanding} observed PINN failed to learn the Helmholtz equation due to the incommensurability between PDE and boundary losses.

In this work, we mainly experiment with the Hamilton-Jacobi-Bellman (HJB) equation in optimal control. Previously, there were several works aiming at solving the HJB equation using deep learning methods \cite{halperin2021distributional,han2018solving,pereira2019learning,yu2020backward,pereira2020feynman,beck2021deep,pham2021neural,davey2022deep}. \cite{han2018solving} is among the first to leverage neural networks to solve HJB equations. In particular, \cite{han2018solving} targets constructing an approximation to a solution value $u$ at $T=0$, which is further transformed into a backward stochastic differential equation and learned by neural networks. The main difference between \cite{han2018solving} and ours is that \cite{han2018solving} only learns the solution on a pre-defined time frame, while with our method, the obtained solution can be evaluated for any time frame. Recently, \cite{halperin2021distributional} tackled the offline reinforcement learning problem and developed a soft relaxation of the classical HJB equation, which can be learned using offline behavior data. The main difference between \cite{halperin2021distributional} and ours is that no additional data is required in our setting.  

Stability is one of the most fundamental concepts in studying the well-posedness of PDE problems. Formally speaking, it characterizes the behavior of the solution to a PDE problem when a small perturbation modifies the operator, initial condition, boundary condition, or force term. We say the equation is stable if the solution of the perturbed PDE converges to the exact solution as the perturbations approach zero \cite{evans1998partial}.
The problem regarding whether a PDE is stable has been intensively studied \cite{evans1998partial,lieberman1996second,gilbarg1977elliptic} in literature. 
There are also some works studying how (regularity) conditions affect stability. \cite{lin2011inviscid} and \cite{deng2018long} investigate in which topology Couette Flow is asymptotic stable. \cite{christ2003asymptotics} answers in which Sobolev space defocusing nonlinear Schrodinger equation, real Korteweg-de Vries (KdV), and modified KdV are locally well-posed. Our main focus is akin to the latter works but is settled in the machine learning framework.
\section{Preliminary}
\label{sec:pre}
In this section, we introduce basic background on Physics-Informed Neural Networks and stochastic optimal control problems. Without loss of generality, we formulate any partial differential equation as: 
\begin{equation}
\label{eq:general-pde}
    \begin{cases}
    \mathcal{L}u(x)=\varphi(x)& \quad x\in\Omega\subset\R^n\\
    \mathcal{B}u(x)=\psi(x)& \quad x\in\partial\Omega,
    \end{cases}
\end{equation}

where $\mathcal{L}$ is the partial differential operator and $\mathcal{B}$ is the boundary condition. We use $x$ to denote the spatiotemporal-dependent variable, and use $\Omega$ and $\partial\Omega$ to denote the domain and boundary. 

\paragraph{Physics-Informed Neural Networks (PINN)} PINN~\cite{raissi2019physics} is a popular choice to learn the function $u(x)$ automatically by minimizing the loss function induced by the PDE (\ref{eq:general-pde}). To be concrete, given $p\in(1,+\infty)$, we define the $L^p$ Physics-Informed Loss as
\begin{align}\label{eq:pinn-loss}
    &\ell_{\Omega,p}(u)=\| \mathcal{L}u(x)-\varphi(x)\|_{L^p(\Omega)}^p, \\
    &\ell_{\partial\Omega,p}(u)=\| \mathcal{B}u(x)-\psi(x)\|_{L^p(\partial\Omega)}^p\label{eq:pinn-loss2}.
\end{align}
The loss term $\ell_{\Omega,p}(u)$ in Eq. (\ref{eq:pinn-loss}) corresponds to the PDE residual, which evaluates how $u(x)$ fits the partial differential equation on $\Omega$; and $\ell_{\partial\Omega,p}(u)$ in Eq. (\ref{eq:pinn-loss2}) corresponds to the boundary residual, which measures how well $u(x)$ satisfies the boundary condition on $\partial\Omega$. $L^p$ denotes $p$-norm, where $p$ is usually set to 2, leading to a ``mean squared error'' interpretation of the loss function \cite{sirignano2018dgm, raissi2019physics}. The goal is to find $u^*$ that minimizes a linear combination of the two losses defined above. The function $u(x)$ is usually parameterized by neural network $u_{\theta}(x)$ with parameter $\theta\in\Theta$. To find $\theta^*$ efficiently, PINN approaches use gradient-based optimization methods. Note that computing the loss involves integrals over $\Omega$ and $\partial\Omega$. Thus, Monte Carlo methods are commonly used to approximate $\ell_{\Omega,p}(u)$ and $\ell_{\partial\Omega,p}(u)$ in practice.

\paragraph{Stochastic Optimal Control} Stochastic control~\cite{fleming2012deterministic, bertsekas1996stochastic} is an important sub-field in optimal control theory. In stochastic control, the state function $\{X_t\}_{0\leq t \leq T}$ is a stochastic process, where $T$ is the time horizon of the control problem. The evolution of the state function is governed by the following stochastic differential equation:
\begin{equation}
\label{eq:general-control-main}
\left\{
\begin{array}{ll}
    \diff X_s = m(s,X_s)\diff s+\sigma \diff W_s & s\in[t,T] \\
    X_t=x
\end{array}
\right.,
\end{equation}
where $m:[t, T]\times\mathbb{R}^n\to \mathbb{R}^n$ is the control function and $\{W_s\}$ is a standard $n$-dimensional Brownian motion.

Given a control function $m$, its total cost is defined as $J_{x,t}(m)=\E\int_t^T r(X_s,m,s)\diff s+ g(X_T)$,
where $r:\mathbb{R}^n\times\mathbb{R}^n\times[0,T]\to\mathbb{R}$ measures the cost rate during the process and $g:\mathbb{R}^n\to\mathbb{R}$ measures the final cost at the terminal state. The expectation is taken over the randomness of the trajectories. 

We are interested in finding a control function that minimizes the total cost for a given initial state. Formally speaking, we define the \textit{value function} of the control problem (\ref{eq:general-control-main}) as $u(x,t)=\min\limits_{m\in \mathcal{M}}J_{x,t}(m)$, where $\mathcal{M}$ denotes the set of possible control functions that we take into consideration. It can be obtained that the value function will follow a particular partial differential equation as stated below.

\begin{definition}[\cite{yong1999stochastic}]
The value function $u(x,t)$ is the unique solution to the following partial differential equation, which is called \textbf{Hamilton-Jacobi-Bellman Equation}:
\begin{equation}
\label{eq:general-hjb}
\begin{cases}
\partial_t u(x,t)+\frac 1 2 \sigma^2 \Delta u(x,t)+\min\limits_{m\in \mathcal{M}}\left[ r(x,m(t,x),t)+\nabla u\cdot m_t\right]=0\\
u(x,T)=g(x).
\end{cases}
\end{equation}
\end{definition}

Hamilton-Jacobi-Bellman (HJB) equation establishes a necessary and sufficient condition for a control's optimality with regard to the cost functions. It is one of the most important high-dimensional PDEs \cite{kirk2004optimal} in optimal control with tremendous applications in physics \cite{sieniutycz2000hamilton}, biology \cite{li2011inverse}, and finance \cite{pham2009continuous}. Many well-known equations, including Riccati equation, Linear–Quadratic–Gaussian control problem \cite{willems1971least}, Merton's portfolio problem \cite{merton1975optimum} are special cases of HJB equation \cite{wonham1968matrix}.

Conventionally, the solution to the HJB equation, i.e., the value function $u(x,t)$, can be computed using dynamic programming \cite{bellman1966dynamic}. 
However, the computational complexity of dynamic programming will grow exponentially with the dimension of state function.
Considering that the state function in many applications is high-dimensional, solving such HJB equations is notoriously difficult in practice using conventional solvers. As neural networks have shown impressive power in learning high-dimensional functions, it's natural to resort to neural-network-based approaches for solving high-dimensional HJB equations.

\section{Failure Mode of PINN on High-Dimensional Stochastic Optimal Control}
\label{sec:theory}

Note that $u(x)$ is the exact solution to the PDE (\ref{eq:general-pde}) if and only if both loss terms $\ell_{\Omega,p}(u)$ and $\ell_{\partial\Omega,p}(u)$ are zero. However, in practice, we usually can only obtain small but non-zero loss values due to the randomness in the optimization procedure or the capacity of the neural network. In such cases, a natural question arises: whether a learned $u(x)$ with a small loss will correspond to a good approximator to the exact solution $u^*(x)$? Such a property is highly related to the concept \emph{stability} in PDE literature, which can be defined as below in our learning scenario:
\begin{definition}
\label{def_stb}
Suppose $Z_1,Z_2,Z_3$ are three Banach spaces. We say a PDE defined as Eq. (\ref{eq:general-pde}) is $(Z_1,Z_2,Z_3)$-stable, if  $\|u^*(x)-u(x)\|_{Z_3}=O(\|\mathcal{L}u(x)-\varphi(x)\|_{Z_1}+\|\mathcal{B}u(x)-\psi(x)\|_{Z_2})$ as $\|\mathcal{L}u(x)-\varphi(x)\|_{Z_1},\|\mathcal{B}u(x)-\psi(x)\|_{Z_2}\to 0$ for any function $u$.
\end{definition}


By definition, if a PDE is $(L^2(\Omega), L^2(\partial \Omega), Z)$-stable with a suitable Banach space $Z$, we can minimize the widely used $L^2$ Physics-Informed Losses $\| \mathcal{L}u(x)-\varphi(x)\|_{L^2(\Omega)}^2$ and $\|\mathcal{B}u(x)-\psi(x)\|_{L^2(\partial\Omega)}^2$, and the learned solution is guaranteed to be close to the exact solution when the loss terms approach zero. However, stability is not always an obvious property for PDEs. There are tremendous equations that are unstable, such as the inverse heat equation. Moreover, even if an equation is stable, it is possible that the equation is not $(L^2(\Omega), L^2(\partial \Omega), Z)$-stable, which suggests that the original $L^2$ Physics-Informed Loss might not be a good choice for solving it. We will show later that for control problems, some practical high-dimensional HJB equations are stable but not $(L^2(\Omega), L^2(\partial \Omega), Z)$-stable, and using $L^2$ Physics-Informed Loss will fail to find an approximated solution in practice. 

We consider a class\footnote{The form of cost function we investigate in the paper is representative in optimal control. For example, in financial markets, we often face power-law trading cost in optimal execution problems \cite{forsyth2012optimal,schied2009risk}. The cost function in Linear–Quadratic–Gaussian control and Merton's portfolio model (constant relative risk aversion utility function in \cite{merton1975optimum}) is also of this form. Therefore, we believe our theoretical analysis for this class of HJB equation is relevant for practical applications. } of HJB equations in which the cost rate function is formulated as $r(x,m)= a_1|m_1|^{\alpha_1}+\cdots+a_n|m_n|^{\alpha_n}-\varphi(x,t)$. The corresponding Hamilton-Jacobi-Bellman equation can be reformulated as:
\begin{equation}
\label{eq:hjb-study}
\begin{cases}
\displaystyle{
\mathcal{L}_{\mathrm{HJB}}u:=\partial_t u(x,t)+\frac 1 2 \sigma^2 \Delta u(x,t)-\sum_{i=1}^n A_i |\partial_{x_i}u|^{c_i}=\varphi(x,t)}&\; (x,t)\in\sR^n\times[0,T]\\
\mathcal{B}_{\mathrm{HJB}}u:=u(x,T)=g(x) &\; x\in\sR^n
\end{cases},
\end{equation}
where 
$A_i=(a_i\alpha_i)^{-\frac{1}{\alpha_i-1}}-a_i(a_i\alpha_i)^{-\frac{\alpha_i}{\alpha_i-1}}\in(0,+\infty)$ and $c_i={\frac{\alpha_i}{\alpha_i-1}}\in(1,\infty)$. See Appendix \ref{app:derive-HJB} for the detailed derivation. For a function $f:X\to \mathbb{R}$, where $X$ is a measurable space, we denote by $\mathrm{supp}f$ the support set of $f$, i.e. the closure of ${\{x\in X:f(x)\neq 0\}}$. 

An important concept for analyzing PDEs is the Sobolev space, which is defined as follows:

\begin{definition}
For $m\in \mathbb{N}$, $p\in[1,+\infty)$ and an open set $\Omega\subset\mathbb{R}^n$, 
the Sobolev space $W^{m,p}(\Omega)$ is defined as $\{f(x)\in L^p(\Omega):D^{\alpha}f\in L^p(\Omega),\forall \alpha\in \mathbb{N}^n,|\alpha|\leq m\}$. The function space $W^{m,p}(\Omega)$ is equipped with Sobolev norm, which is defined as $\|f\|_{W^{m,p}(\Omega)}=\left(\sum\limits_{|\alpha|\leq m}\|D^{\alpha}f\|^p_{L^p(\Omega)}\right)^{\frac 1 p}$.
\end{definition}

The definition above can be extended to functions defined on a spatiotemporal domain $Q\subseteq \sR^n \times [0,T]$. With a slight abuse of notation, we define $W^{m,p}(Q)=\{f(x,t)\in L^p(Q):D^{\alpha}f\in L^p(Q),\forall \alpha\in \mathbb{N}^n,|\alpha|\leq m\}$,
where the differential $D^{\alpha}$ is only operated over spatial variable $x$. The norm $\|\cdot\|_{W^{m,p}(Q)}$ can also be defined accordingly. 

\paragraph{Stability of the HJB Equation}
We present our main theoretical result which characterizes the stability of the HJB equation (Eq. (\ref{eq:hjb-study})). In particular, we show that the HJB equation is $(L^p(\sR^n\times[0,T]),L^q(\sR^n),W^{1,r}(\sR^n\times[0,T]))$-stable when $p$, $q$ and $r$ satisfies certain conditions.
We take the Banach space $Z_3$ in Definition \ref{def_stb} as $W^{1,r}$ here because it captures the properties of both the value and the derivatives of a function, but $L^p$ spaces do not. However, as could be seen from Appendix \ref{app:derive-HJB}, for optimal control problems, it is essential to obtain an accurate approximator for both the value and the gradient of the value function $u$ (the solution of  (Eq. (\ref{eq:hjb-study})). Thus, it is appropriate to analyze the quality of the approximate solution in $W^{1,r}$ space.


\begin{theorem}
\label{thm:stb0}
For $p,q\geq 1$, let $r_0=\frac{(n+2)q}{n+q}$. Assume the following inequalities hold for $p,q$ and $r_0$:
\begin{equation}
\label{eq:main-thm-cond}
    p\geq \max\left\{2, \left(1-\frac{1}{\bar{c}}\right)n\right\};~
    q> \frac{(\bar{c}-1)n^2}{(2-\bar{c})n+2};~
    \frac{1}{r_0}\geq \frac{1}{p}-\frac{1}{n},
\end{equation}
where $\bar{c}=\max\limits_{1\leq i\leq n} c_i$ in Eq. (\ref{eq:hjb-study}). Then for any $r\in[1,r_0)$ and any bounded open set $Q\subset \mathbb{R}^n\times[0,T]$, Eq. (\ref{eq:hjb-study}) is $(L^p(\mathbb{R}^n\times[0,T]),L^q(\mathbb{R}^n),W^{1,r}(Q))$-stable for $\bar{c}\leq 2$.
\end{theorem}

The proof of Theorem \ref{thm:stb0} can be found in Appendix \ref{app:proof-upper-bound} and an improved theorem with relaxed dependency on $\bar{c}$ can be found in Appendix \ref{new-thm}.
Intuitively, Theorem \ref{thm:stb0} states that $(L^p,L^q,W^{1,r})$-stability of Eq. (\ref{eq:hjb-study}) can be achieved when $p,q=\Omega(n)$. We further show that this linear dependency on $n$ cannot be relaxed in the following theorem:

\begin{theorem}
\label{thm:lower-bound}
There exists an instance of Eq. (\ref{eq:hjb-study}), whose exact solution is $u^*$, such that for any $\varepsilon>0,A>0,r\geq 1,m\in\mathbb{N}$ and $p\in\left[1,\frac n 4\right]$, there exists a function $u\in C^{\infty}(\mathbb{R}^n\times(0,T])$ which satisfies the following conditions:
\begin{itemize}
    \item $\|\mathcal{L}_{\mathrm{HJB}}u-\varphi\|_{L^p(\sR^n\times[0,T])}<\varepsilon$,  $\mathcal{B}_{\mathrm{HJB}}u=\mathcal{B}_{\mathrm{HJB}}u^{*}$, and $\mathrm{supp}(u-u^{*})$ is compact, where $\mathcal{L}_{\mathrm{HJB}}$ and $\mathcal{B}_{\mathrm{HJB}}$ are defined in Eq. (\ref{eq:hjb-study}).
    \item ${\|u-u^{*}\|_{W^{m,r}(\sR^n\times[0,T])}>A}.$
\end{itemize}
\end{theorem}

The proof of Theorem \ref{thm:lower-bound} can be found in Appendix \ref{app:proof-lower-bound}.

\paragraph{Discussion.} Theorem \ref{thm:stb0} and \ref{thm:lower-bound} together state that when the dimension of the state function $n$ is large, the HJB equation in Eq. (\ref{eq:hjb-study}) cannot be $(L^p,L^q,W^{1,r})$-stable 
if $p$ and $q$ are small. Furthermore, since $L^r$=$W^{0,r}$ by definition, Theorem \ref{thm:lower-bound} also implies that Eq. (\ref{eq:hjb-study}) is not even $(L^p,L^q,L^{r})$-stable. Therefore, for high-dimensional HJB problems, if we use classic $L^2$ Physics-Informed Loss for training PINN, the learned solution may be arbitrarily distant from $u^*$ even if the loss is very small. Such theoretical results are verified in our empirical studies in Section \ref{sec:exp}. 

More importantly, our theoretical results indicate that the design choice of the Physics-Informed Loss plays a significant role in solving PDEs using PINN. In this work, we shed light upon this problem using HJB equations. We believe the relationship between PDE's stability and the Physics-Informed Loss should be carefully investigated in the future, especially for high-dimensional non-linear PDEs whose stability are more complicated than low-dimensional and linear ones \cite{evans1998partial,gilbarg1977elliptic,lieberman1996second}. 
Given the above observations, we further propose a new algorithm for training PINN to solve HJB Equations, which will be presented in the subsequent sections.

\section{Solving HJB Equations with Adversarial Training}
\label{sec:alg}

\begin{algorithm}[b]
    \caption{$L^{\infty}$ Training for Physics-Informed Neural Networks}
    \label{alg:main}
    \hspace*{0.02in} \textbf{Input:} Target PDE (Eq. (\ref{eq:general-pde})); neural network $u_{\theta}$; initial model parameters $\theta$\\
    \hspace*{0.02in} \textbf{Output:} Learned PDE solution $u_{\theta}$\\
    \hspace*{0.02in} \textbf{Hyper-parameters:} Number of total training iterations $M$; number of iterations and step size of inner loop $K,\eta$; weight for combining the two loss term $\lambda$
    \begin{algorithmic}[1]
    \For{$i=1, \cdots, M$}
        \State Sample $x^{(1)},\cdots,x^{(N_1)}\in\Omega$ and $\tilde x^{(1)},\cdots,\tilde x^{(N_2)}\in \partial\Omega$
    	\For{$j=1, \cdots, K$}
    		\For{$k=1, \cdots, N_1$}
    		    \State $x^{(k)}\leftarrow \mathrm{Project}_{\Omega}\left( x^{(k)}+\eta~\mathrm{sign}\nabla_x  \left(\mathcal{L}u_{\theta}(x^{(k)})-\varphi(x^{(k)})\right)^2\right)$
    	    \EndFor
    	    \For{$k=1, \cdots, N_2$}
    		    \State $\tilde x^{(k)}\leftarrow \mathrm{Project}_{\partial \Omega}\left(\tilde x^{(k)}+\eta~ \mathrm{sign}\nabla_x \left(\mathcal{B}u_{\theta}(\tilde x^{(k)})-\psi(\tilde x^{(k)})\right)^2\right)$
    	    \EndFor
    	\EndFor
    	\State $\displaystyle{ g\leftarrow \nabla_{\theta}\left(\frac{1}{N_1}\sum_{i=1}^{N_1}\left( \mathcal{L}u_{\theta}(x^{(i)})-\varphi(x^{(i)})\right)^2+\lambda\cdot\frac{1}{N_2}\sum_{i=1}^{N_2}\left( \mathcal{B}u_{\theta}(\tilde x^{(i)}) -\psi(\tilde x^{(i)})\right)^2 \right)}$
	\State $\theta\leftarrow \mathrm{Optimizer}\left(\theta, g \right)$
	\EndFor
    \State \Return $u_{\theta}$
    \end{algorithmic}
\end{algorithm}

The above results suggest that we should use a large value of $p$ and $q$ in the loss $\ell_{\Omega,p}(u)$ and $\ell_{\partial\Omega,q}(u)$ to guarantee a learned solution $u$ is close to $u^*$ for high-dimensional HJB problems. Note that $L^p$-norm and $L^\infty$-norm behave similarly when $p$ is large. We can substitute $L^p$-norm by $L^\infty$-norm and directly optimize $\ell_{\Omega,\infty}(u)$ and $\ell_{\partial\Omega,\infty}(u)$. Overall, the training objective can be formulated as:
\begin{equation}\label{eq:pinn-loss-inf}
    \min_{u}~\ell_{\infty}(u)=\sup\limits_{x\in\Omega}| \mathcal{L}u(x)-\varphi(x)|+\lambda\sup\limits_{x\in\partial\Omega}| \mathcal{B}u(x)-\psi(x)|,
\end{equation}
where $\lambda>0$ is a hyper-parameter to trade off the two objectives.

It is straightforward to obtain that setting $p$ and $q$ to infinity satisfies the conditions in Theorem \ref{thm:stb0}, and thus the quality of the learned solution enjoys theoretical guarantee. Furthermore,  Eq.~(\ref{eq:pinn-loss-inf}) can be regarded as a min-max optimization problem. The inner loop is a maximization problem to find data points on $\Omega$ and $\partial\Omega$ where $u$ violates the PDE most, and the outer loop is a minimization problem to find $u$ (i.e., the neural network parameters) that minimizes the loss on those points. 

In deep learning, such a min-max optimization problem has been intensively studied, and adversarial training is one of the most effective learning approaches in many applications. We leverage adversarial training, and the detailed implementation is described in Algorithm \ref{alg:main}. In each training step, the model parameters and data points are iteratively updated. We first fix the model $u$ and randomly sample data points $x^{(1)},\cdots,x^{(N_1)}\in\Omega$ and $\tilde x^{(1)},\cdots,\tilde x^{(N_2)}\in \partial\Omega$, serving as a random initialization of the inner loop optimization. Then we perform gradient-based methods to obtain data points with large point-wise Physics-Informed Losses, which leads to the following inner-loop update rule:
\begin{align}
    &x^{(k)}\leftarrow \mathrm{Project}_{\Omega}\left( x^{(k)}+\eta~\mathrm{sign}\nabla_x  \left(\mathcal{L}u_{\theta}(x^{(k)})-\varphi(x^{(k)})\right)^2\right);\\
    &\tilde x^{(k)}\leftarrow \mathrm{Project}_{\partial \Omega}\left(\tilde x^{(k)}+\eta~ \mathrm{sign}\nabla_x \left(\mathcal{B}u_{\theta}(\tilde x^{(k)})-\psi(\tilde x^{(k)})\right)^2\right),
\end{align}
where $\mathrm{Project}_{\Omega}\left(\cdot \right)$ and $\mathrm{Project}_{\partial \Omega}\left(\cdot \right)$ project the updated data points to the domain. When the inner-loop optimization finishes, we fix the generated data points and calculate the gradient $g$ to the model parameter:
\begin{align}
g\leftarrow \nabla_{\theta}\left(\frac{1}{N_1}\sum_{i=1}^{N_1}\left( \mathcal{L}u_{\theta}(x^{(i)})-\varphi(x^{(i)})\right)^2+\lambda\cdot\frac{1}{N_2}\sum_{i=1}^{N_2}\left( \mathcal{B}u_{\theta}(\tilde x^{(i)}) -\psi(\tilde x^{(i)})\right)^2 \right),    
\end{align}
then the model parameter can be updated using any first-order optimization methods. When the training finishes, the learned neural network will converge to a solution with small $L^{\infty}$ losses and is guaranteed to be close to the exact solution.

\section{Experiments}
\label{sec:exp}

In this section, we conduct experiments to verify the effectiveness of our approach. Ablation studies on the design choices and hyper-parameters are then provided. Our codes are implemented based on \texttt{PyTorch} \cite{paszke2019pytorch}. All the models are trained on one NVIDIA Tesla V100 GPU with 16GB memory. Due to space limitation, we only showcase our methods on the Linear Quadratic Gaussian control problem in the main body of the paper. More experimental results on other PDE problems can be found in Appendix \ref{app:more-exp}.

\subsection{High Dimensional Linear Quadratic Gaussian Control Problem}
\label{sec:hjb}

\begin{table}[tb]
\caption{\textbf{Experimental results of solving the 100/250-dimensional LQG control problems.} $n$ denotes the dimensionality of the problem. Performances are measured by $L^1$, $L^2$, and $W^{1,1}$ relative error in $[0,1]^n\times[0,T]$. The best performances are indicated in \textbf{bold}.}
\label{tab:exp-lqg-main}
\centering
\begin{tabular}{ccccccc}\toprule
\multirow{2}{*}{Method} & \multicolumn{3}{c}{Relative error for $n=100$} & \multicolumn{3}{c}{Relative error for $n=250$}\\
                        &  $L^1$ & $L^2$ & $W^{1,1}$ &  $L^1$  & $L^2$  & $W^{1,1}$\\\midrule
Original PINN \cite{raissi2019physics}     & 3.47\%    & 4.25\%  &11.31\%  & 6.74\% & 7.67\% & 17.51\% \\ 
Adaptive time sampling \cite{wight2020solving}  & 3.05\%	& 3.67\% & 13.63\% & 7.18\% & 7.91\% & 18.38\%\\
Learning rate annealing  \cite{wang2021understanding}   &  11.09\%  &  11.82\%   &  33.61\%  &  6.94\% & 8.04\% & 18.47\% \\
Curriculum regularization \cite{krishnapriyan2021characterizing}   &   3.40\% & 3.91\% &   9.53\%  &   6.72\% &  7.51\% & 17.52\%\\ 
\midrule
Adversarial training (ours)  & \textbf{0.27\%}	& \textbf{0.33\%}  & \textbf{2.22\%} & \textbf{0.95\%}	& \textbf{1.18\%}	& \textbf{4.38\%}	\\ \bottomrule     
\end{tabular}
\end{table}

We follow \cite{han2018solving} to study the classical linear-quadratic Gaussian (LQG) control problem in $n$ dimensions, a special case of the HJB equation:
\begin{equation}
\label{eq:lqg}
\begin{cases}
    \partial_t u(x,t) + \Delta u(x,t) -\mu \|\nabla_x u(x,t)\|^2 = 0 & x \in \mathbb{R}^n, 
    t \in [0, T] \\
    u(x,T) = g(x) & x \in \mathbb{R}^n,
\end{cases}
\end{equation}

As is shown in \cite{han2018solving}, there is a unique solution to Eq. (\ref{eq:lqg}):
\begin{equation}
    u(x,t) = -\frac{1}{\mu} \ln\left(\int_{\R^n} (2\pi)^{-n/2}\mathrm{e}^{-\|y\|^2/2}\cdot \mathrm{e}^{-\mu g(x-\sqrt{2(T-t)}y)} \diff y\right),
\end{equation}

We set $\mu=1$, $T=1$, and the terminal cost function $g(x)=\ln\left(\dfrac{1+\|x\|^2}{2}\right)$.  

\paragraph{Experimental Design} The neural network used for training is a 4-layer MLP with 4096 neurons and $\mathrm{tanh}$ activation in each hidden layer. 
To train the models, we use Adam as the optimizer \cite{kingma2015adam}. The learning rate is set to $7\mathrm{e}-4$ in the beginning and then decays linearly to zero during training. The total number of training iterations is set to 5000/10000 for the 100/250-dimensional problem. In each training iteration, we sample $N_1=100/50$ points from the domain $\mathbb{R}^n\times [0, T]$ and $N_2=100/50$ points from the boundary $\mathbb{R}^n\times \{T\}$ to obtain a mini-batch for the 100/250-dimensional problem. The number of inner-loop iterations $K$ is set to 20, and the inner-loop step size $\eta$ is set to 0.05 unless otherwise specified. 
Evaluations are performed on a hold-out validation set which is unseen during training. We use the $L^1$, $L^2$, and $W^{1,1}$ relative error in $[0,1]^n\times[0,T]$ as evaluation metrics: $L^1$ and $L^2$ relative errors are popular evaluation metrics in literature. We additionally consider $W^{1,1}$ relative error since the gradient of the solution to HJB equations plays an important role in applications, and our theory indicates that Eq. (\ref{eq:lqg}) is $(L^{\infty},L^{\infty}, W^{1,r})$ stable. More detailed descriptions of the experimental setting and evaluation metrics can be found in Appendix \ref{app:exp-settings}.

We compare our method with a few strong baselines: 1) original PINN trained with $L^2$ Physics-Informed Loss \cite{raissi2019physics}; 2) adaptive time sampling for PINN training proposed in \cite{wight2020solving}; 3) PINN with the learning rate annealing algorithm proposed in \cite{wang2021understanding}; 4) curriculum PINN regularization proposed in \cite{krishnapriyan2021characterizing}. The training recipes for the baseline methods, including the neural network architecture, the training iterations, the optimizer, and the learning rate, are the same as those of our method described above. 
It should be noted that although these approaches modifies the data sampler, training algorithms or the loss function, they all keep the $L^2$ norm of the PDE residual and boundary residual unchanged in the training objective.

\begin{figure}[t]
\begin{minipage}{0.35\linewidth}
    \centering
    \includegraphics[width=1\linewidth]{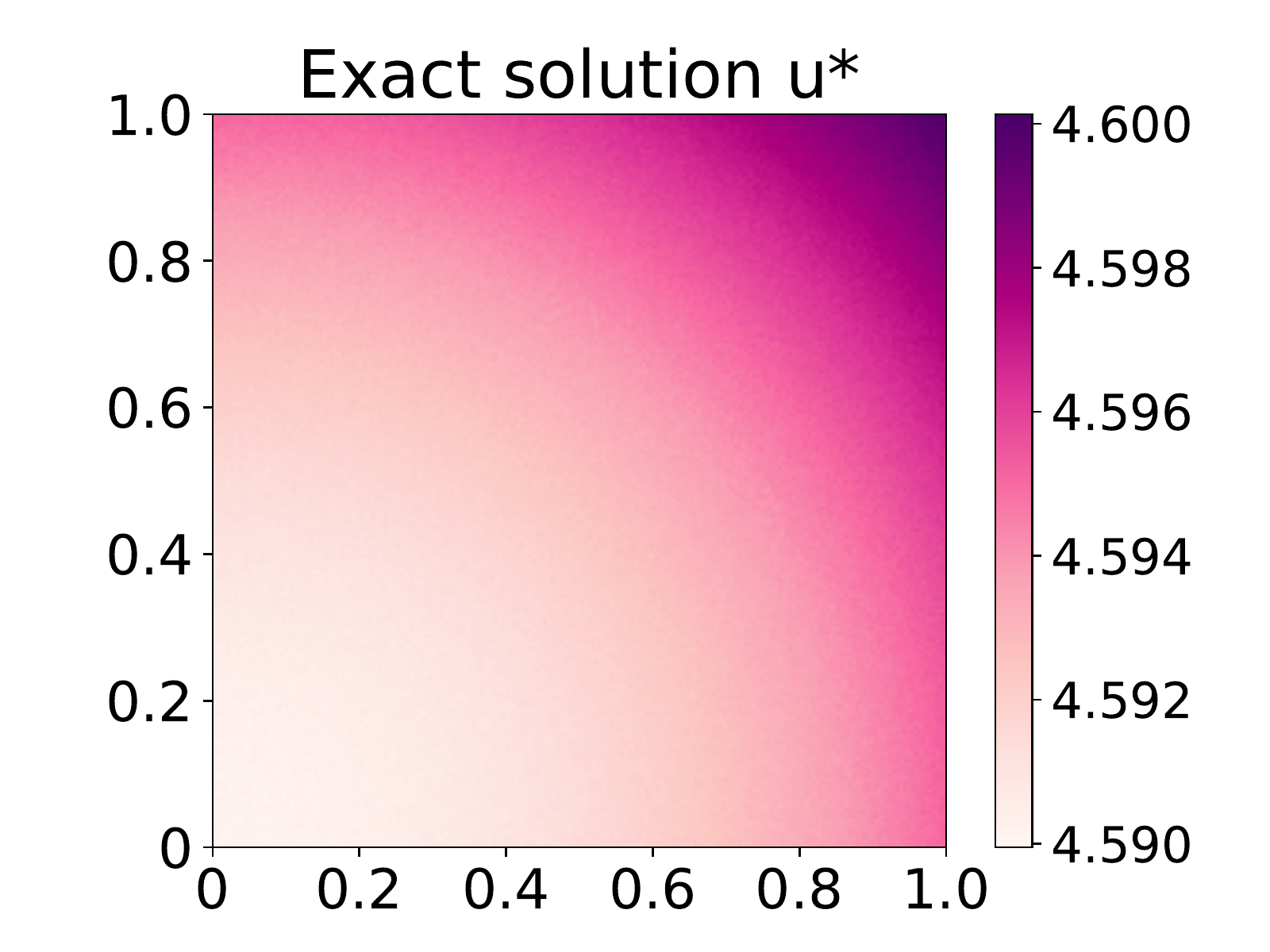}
\end{minipage}\hfill
\begin{minipage}{0.32\linewidth}
    \centering
    \includegraphics[width=1\linewidth]{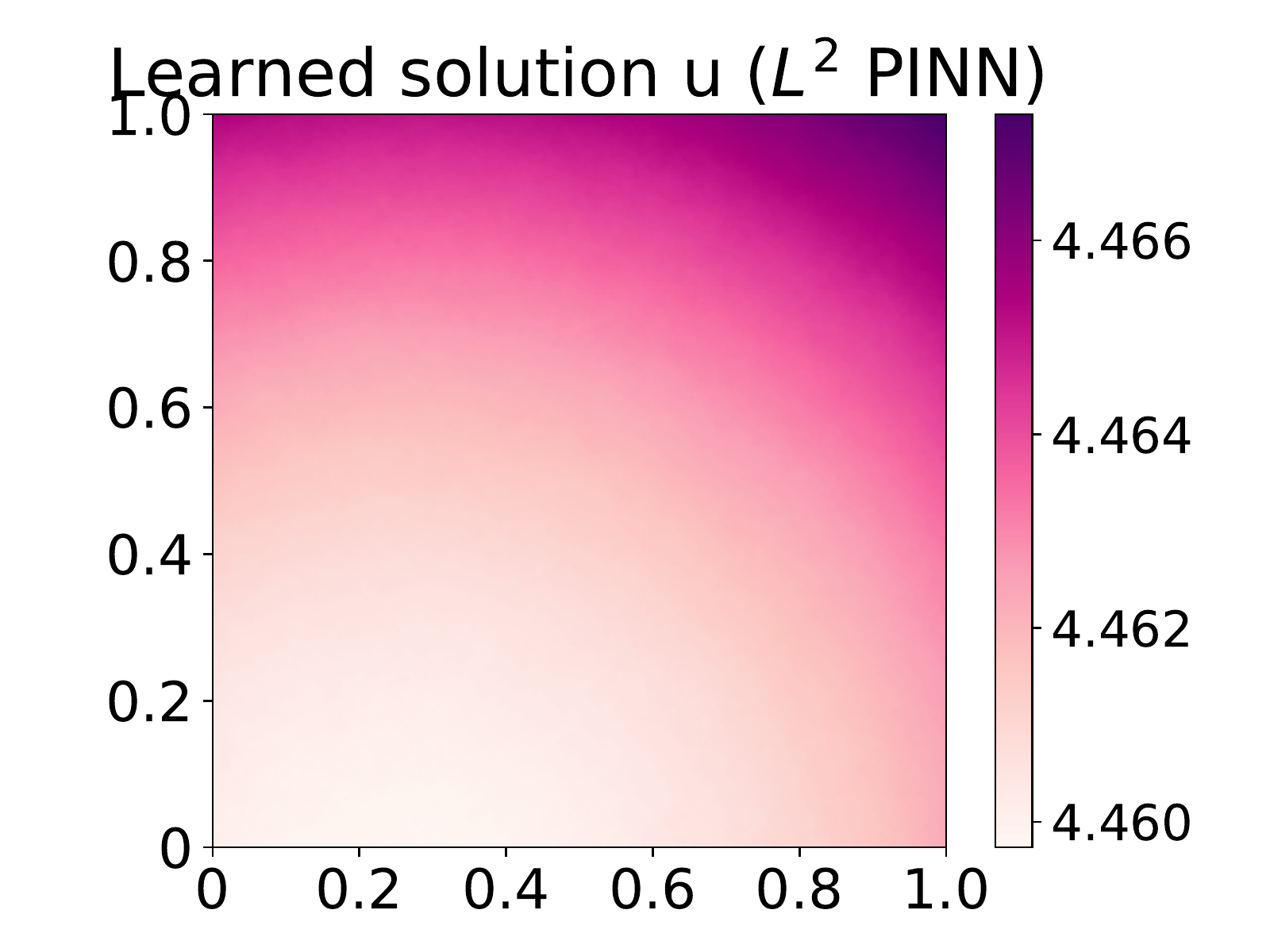}
    \includegraphics[width=1\linewidth]{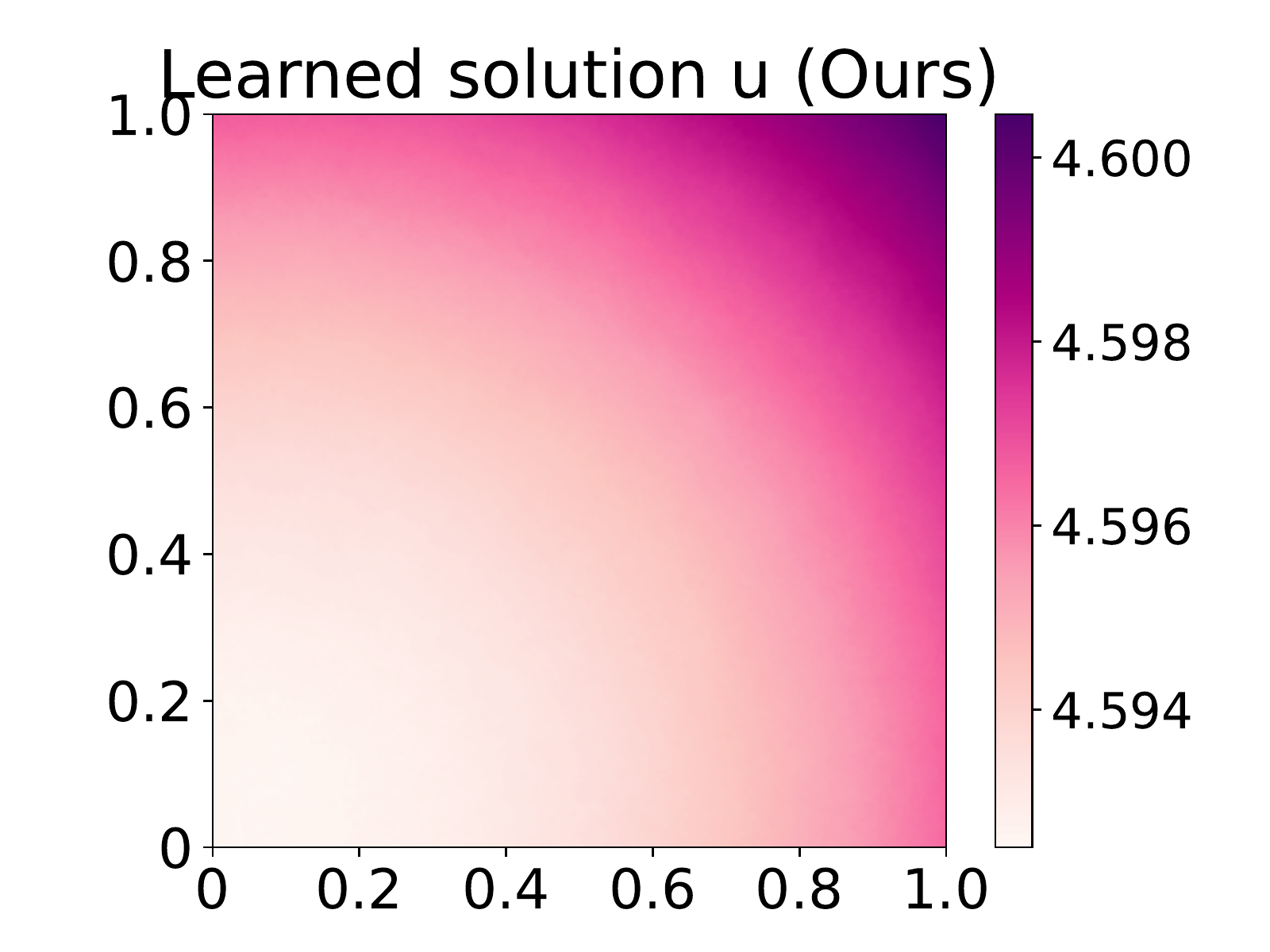}
\end{minipage}\hfill
\begin{minipage}{0.32\linewidth}
    \centering
    \includegraphics[width=1\linewidth]{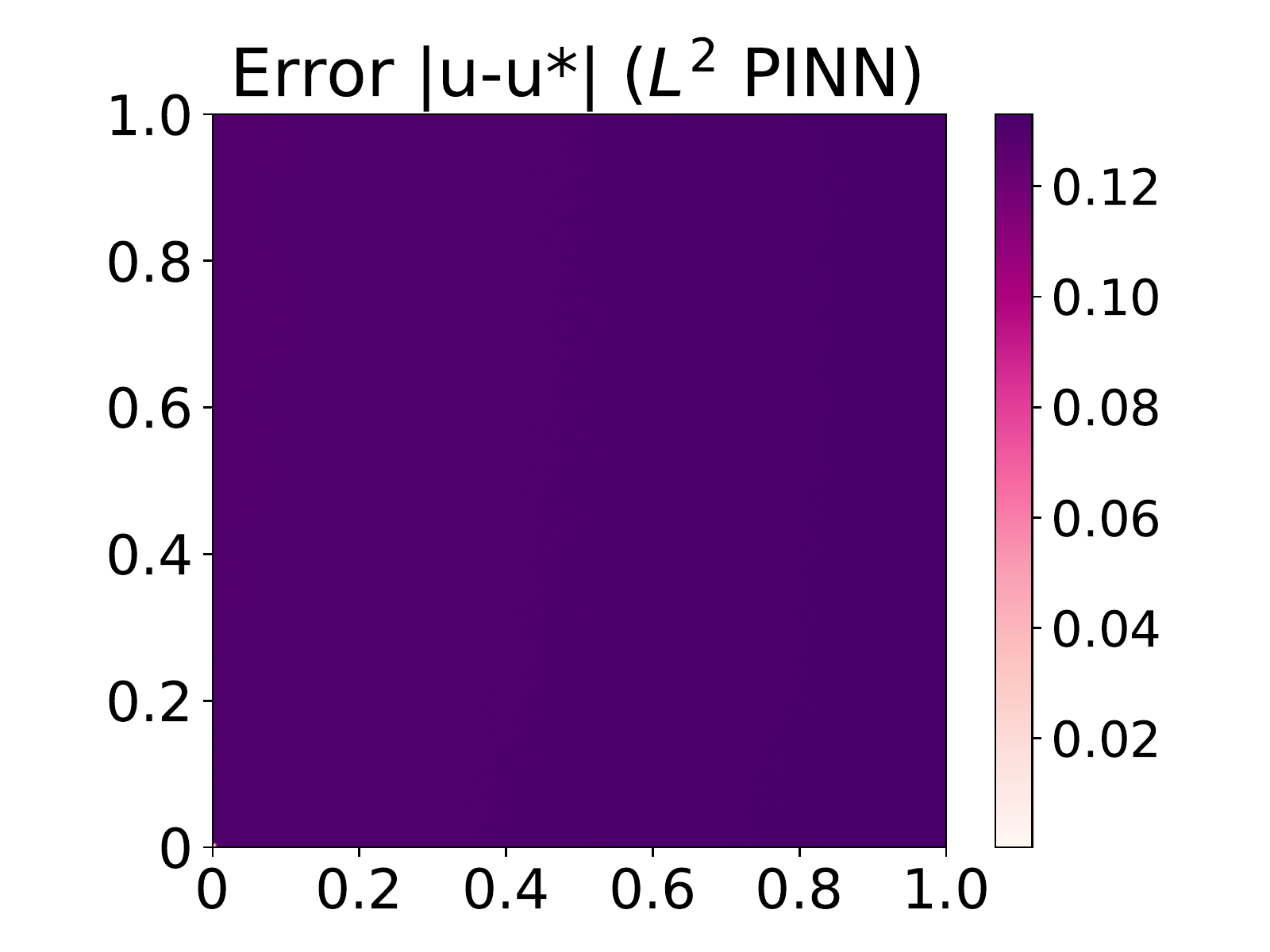}
    \includegraphics[width=1\linewidth]{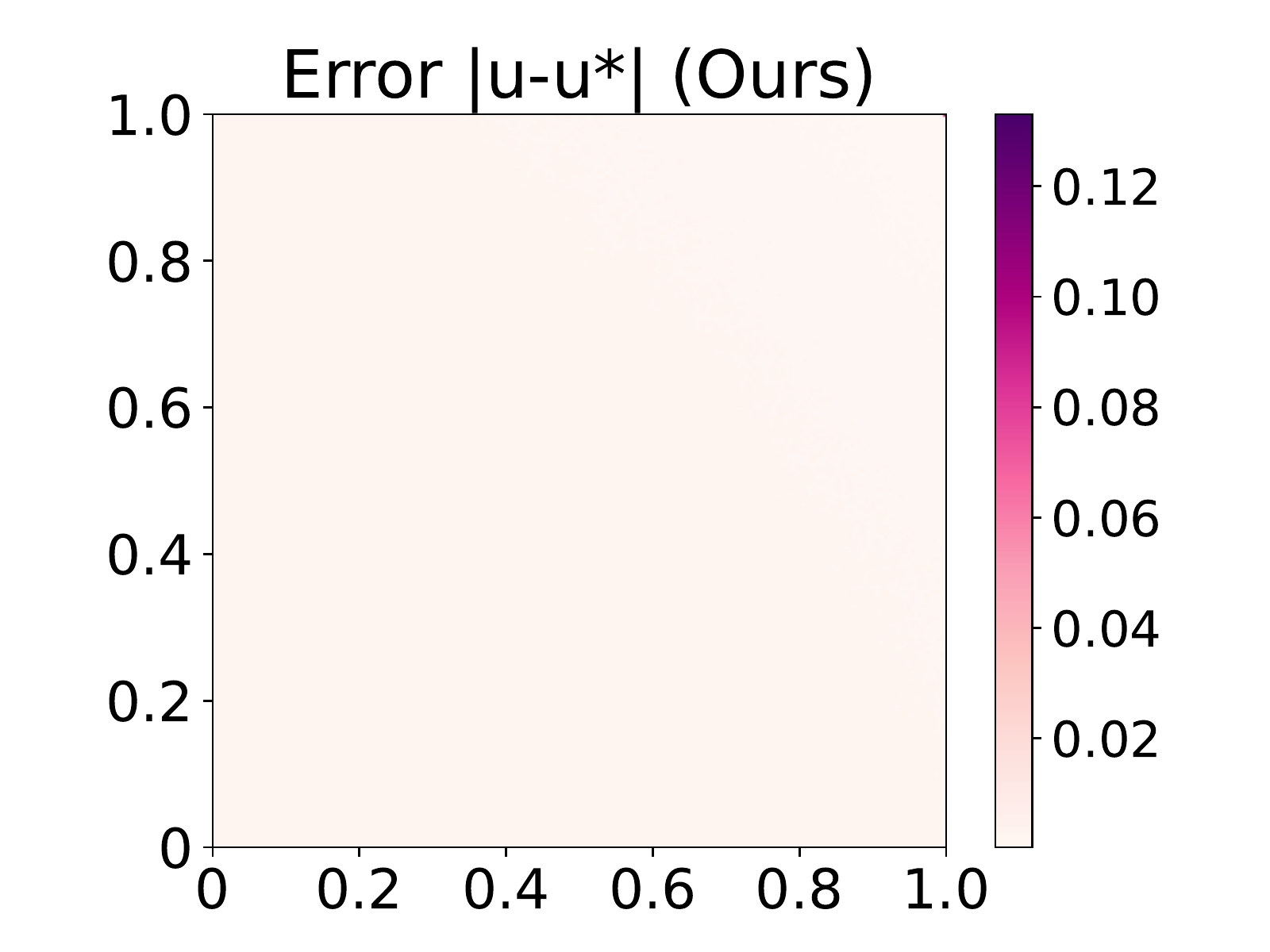}
\end{minipage}
\caption{\textbf{Visualization for the solutions of Eq. (\ref{eq:lqg})}. The left panel shows the exact solution $u^*$; the middle panel shows the learned solutions $u$ of the original PINN method with $L^2$ loss and our method with adversarial training; the right panel show the point-wise absolute error $|u-u^*|$. Note that the solution is a high dimensional function, and we visualize its snapshot on a two-dimensional domain. Specifically, we visualize a bivariate function $u(x_1, x_2, 0, \cdots,0; 0)$ for $x_1,x_2\in[0,1]$ with the horizontal axis and vertical axis corresponding to $x_1$ and $x_2$ respectively.}
\label{fig:lqg}
\end{figure}

\paragraph{Experimental Results} The experimental results are summarized in Table \ref{tab:exp-lqg-main}. It's clear that the relative error of the model trained using the original PINN does not fit the solution well, e.g., the $L^1$ relative error is larger than $6\%$ when $n=250$. This empirical observation aligns well with our theoretical analysis, i.e., minimizing $L^2$ loss cannot guarantee the learned solution to be accurate. Advanced methods, e.g., curriculum PINN regularization \cite{krishnapriyan2021characterizing}, can improve the accuracy of the learned solution but with marginal improvement, which suggests that these methods do not address the key limitation of PINN in solving high-dimensional HJB Equations. By contrast, our proposed method significantly outperforms all the baseline methods in terms of both $L^p$ relative error and Sobolev relative error, which indicates that both the values and the gradients of our learned solutions are more accurate than the baselines.

We also examine the quality of the learned solution $u(x,t)$ by visualization. As the solution is a high-dimensional function, we visualize its snapshot on a two-dimensional space. Specifically, we consider a bivariate function $u(x_1, x_2, 0, \cdots,0; 0)$ and use a heatmap to show its function value given different $x_1$ and $x_2$. Figure \ref{fig:lqg} shows the ground truth $u^*$, the learned solutions $u$ of original PINN and our method, and the point-wise absolute error $|u-u^*|$ for each methods. The two axises correspond $x_1$ and $x_2$, respectively. We can see that the point-wise error of the learned solution using our algorithm is less than $2\mathrm{e}-2$ on average. In contrast, the point-wise error of the learned solution using original PINN method with $L^2$ loss is larger than $1.3\mathrm{e}-1$ for most areas. Therefore, the visualization of the solutions clearly illustrate that PINN learned more accurate solution using our proposed algorithm.

Furthermore, we visualize the gradient norm $|\nabla_x u|$ of the learned solution of both our method and the original PINN in Figure \ref{fig:lqg-grad}, to illustrate that not only can the learned solution of our method accurately approximate the exact solution, but also the \textit{gradient} of the learned solution can accurately approximate the \textit{gradient} of the exact solution. 
Again, since $|\nabla_x u|$ is a high dimensional function, we use a heatmap to show its function value given different $x_1$ and $x_2$, and set the other variables to 0. From the right panel of Figure \ref{fig:lqg-grad}, we can clearly see that the gradient of the learned solution of our method is much more accurate compared with that of the gradient of the learned solution using original PINN. The gradient norm error of the vanilla PINN approach is nearly $1\mathrm{e}-2$ in some areas shown in the visualization, while the error of our method is less than $1\mathrm{e}-3$ for most data points. This empirical observation aligns well with our theory, which states that Eq. (\ref{eq:lqg}) is $(L^{\infty},L^{\infty}, W^{1,r})$ stable.


\begin{figure}[t]
\begin{minipage}{0.35\linewidth}
    \centering
    \includegraphics[width=1\linewidth]{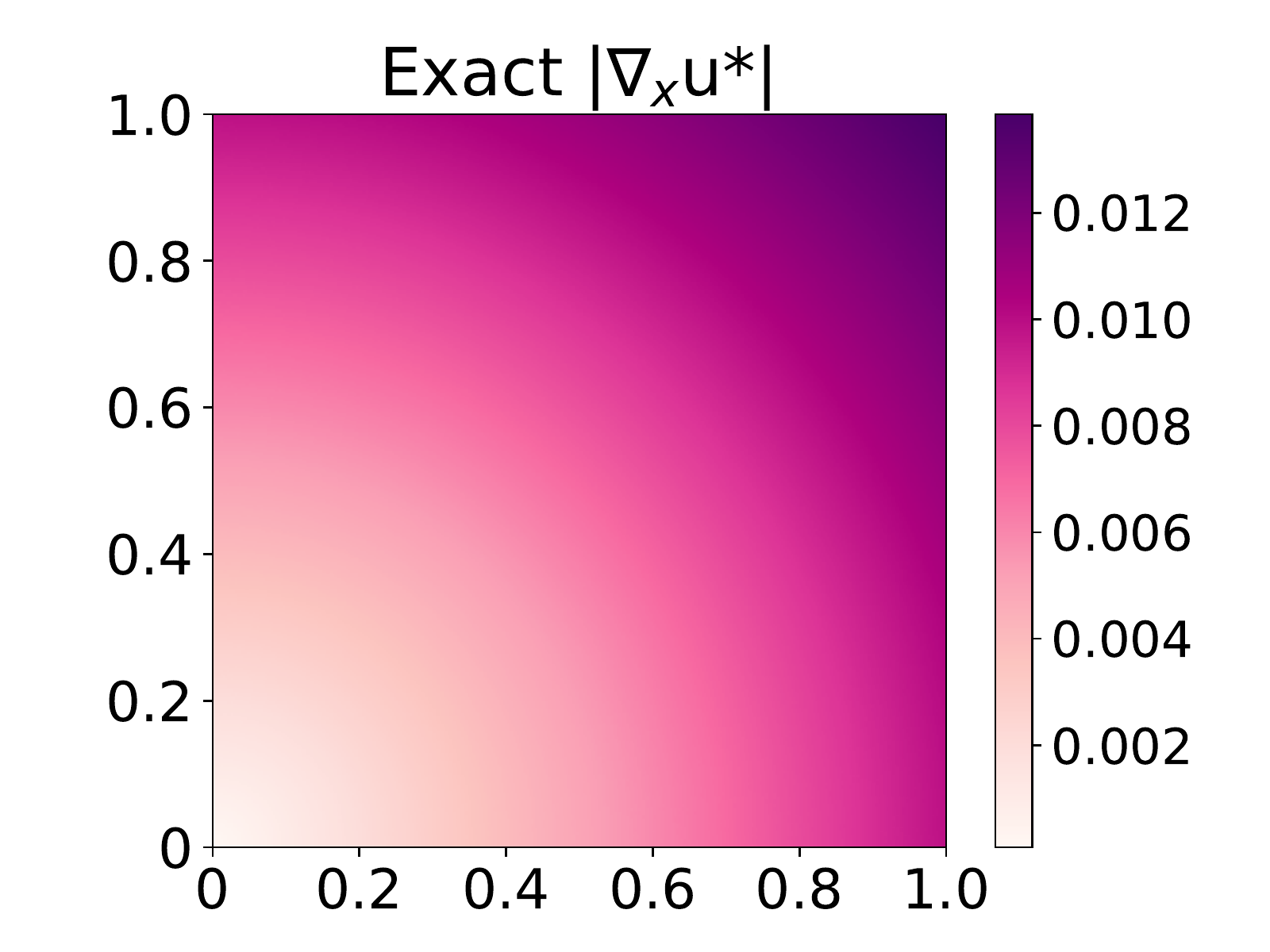}
\end{minipage}\hfill
\begin{minipage}{0.32\linewidth}
    \centering
    \includegraphics[width=1\linewidth]{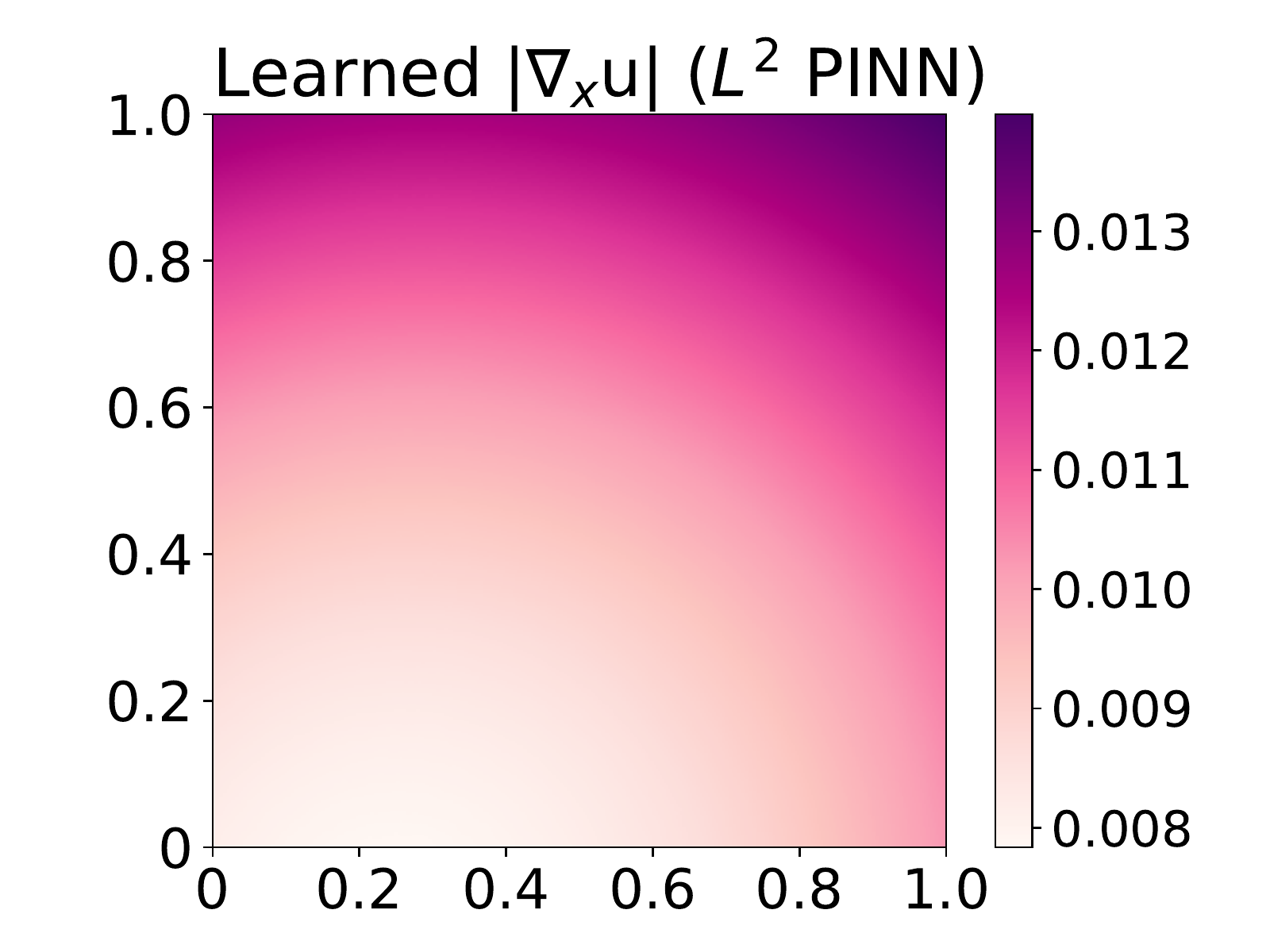}
    \includegraphics[width=1\linewidth]{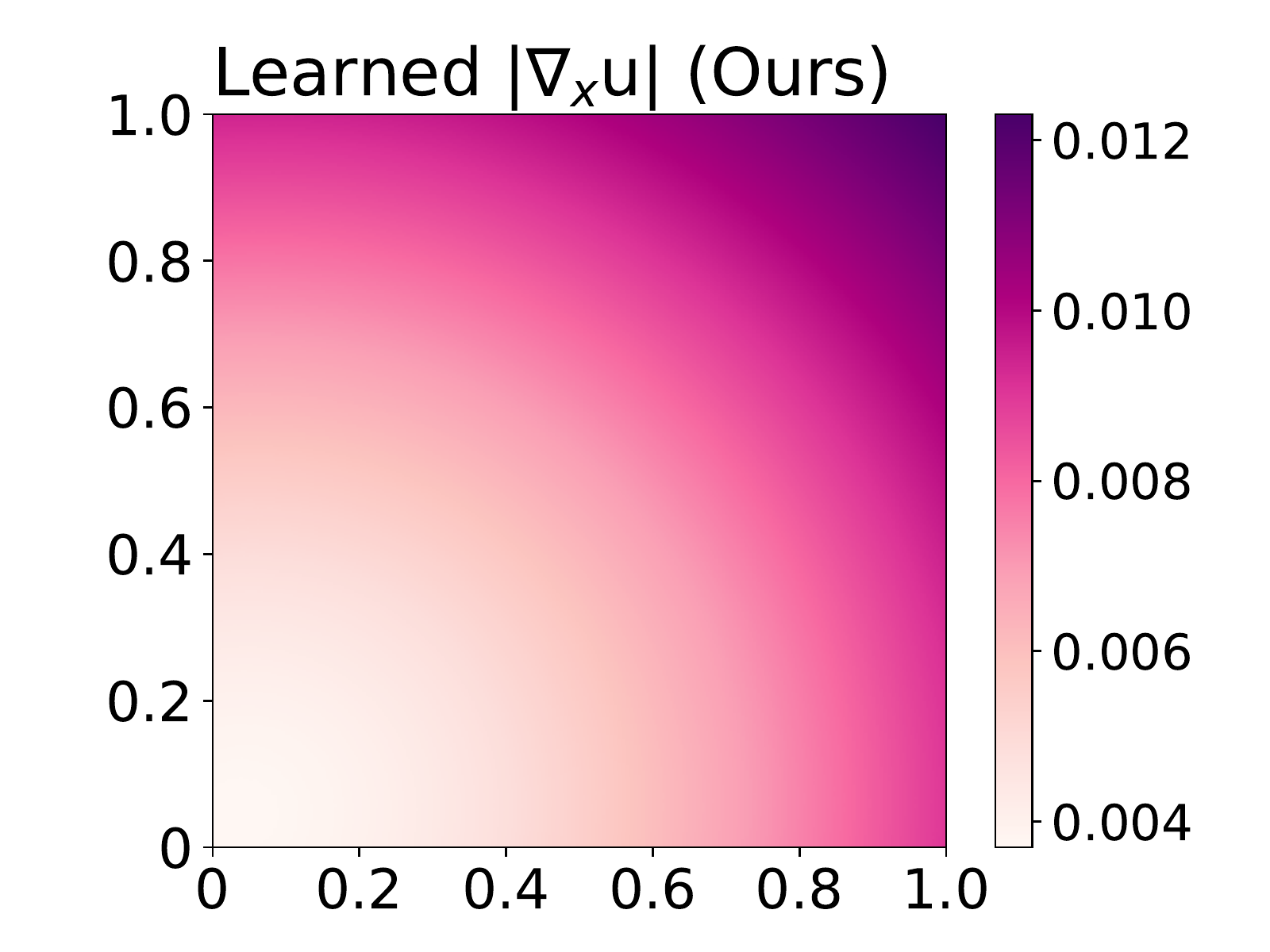}
\end{minipage}\hfill
\begin{minipage}{0.32\linewidth}
    \centering
    \includegraphics[width=1\linewidth]{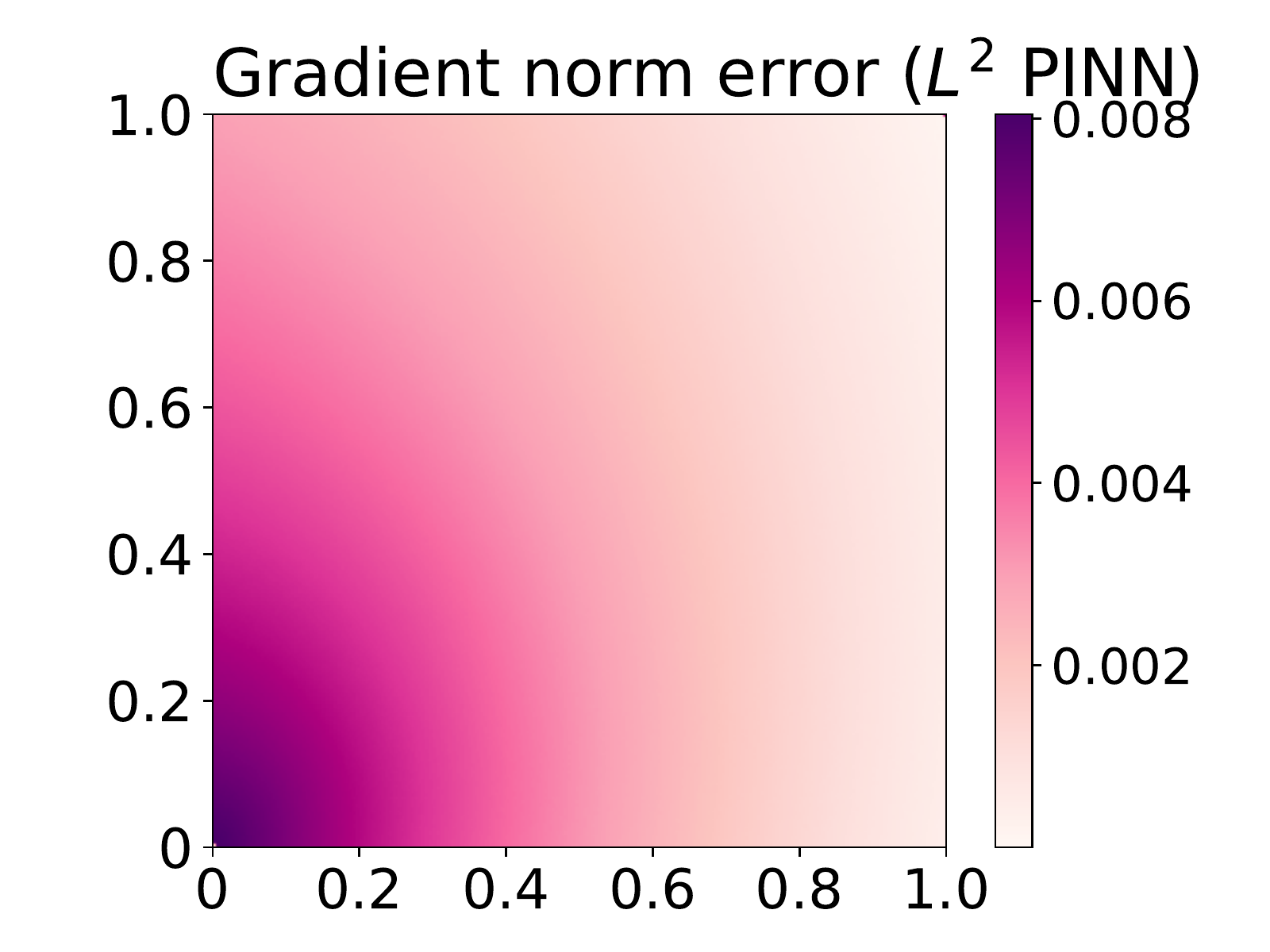}
    \includegraphics[width=1\linewidth]{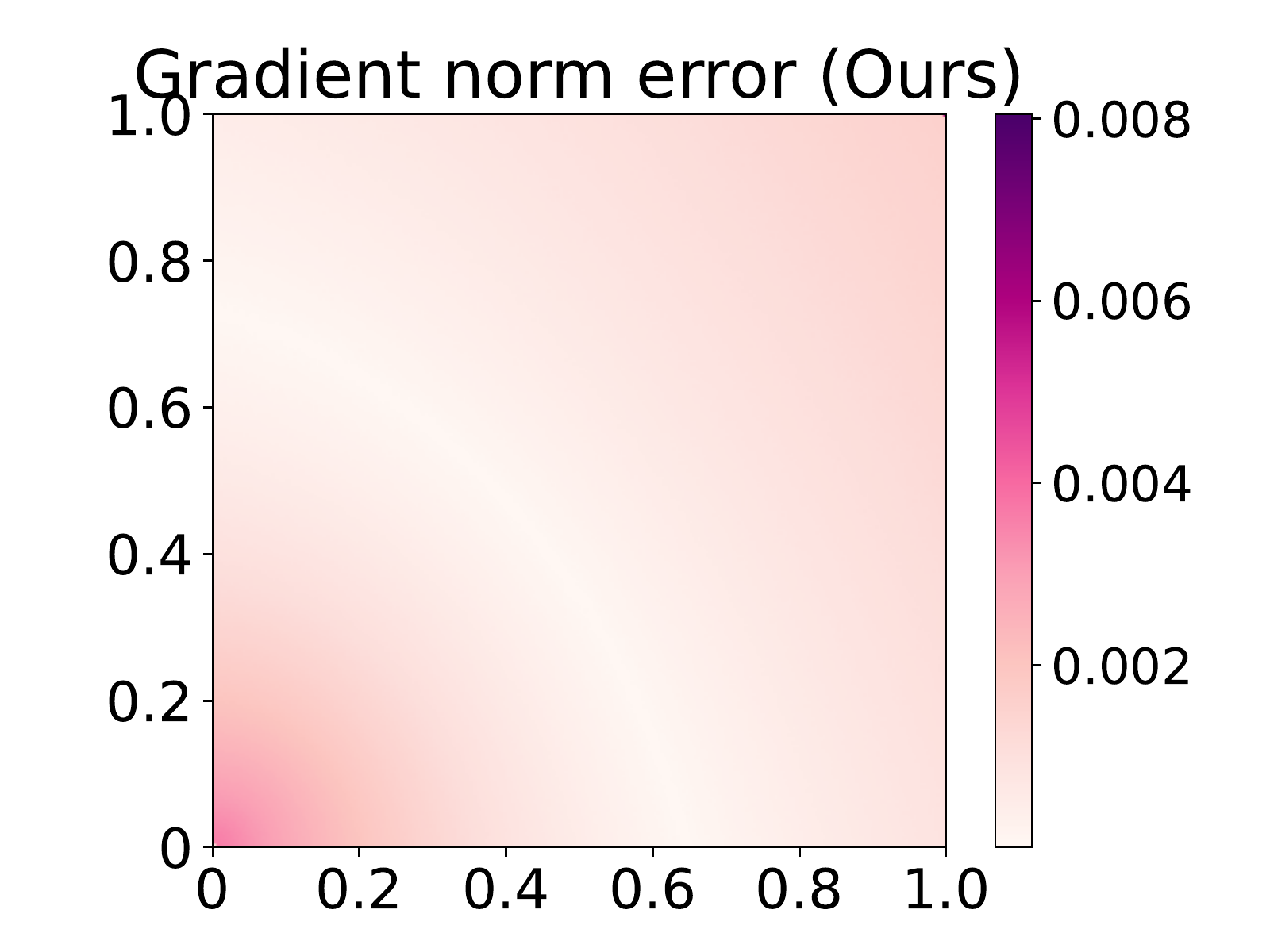}
\end{minipage}
\caption{\textbf{Visualization for the gradient norm of the solutions of Eq. (\ref{eq:lqg})}. The left panel shows the ground truth $|\nabla_x u^*|$; the middle panel shows the learned $|\nabla_x u|$ of the original PINN method with $L^2$ loss and our method with adversarial training; the right panel shows the point-wise absolute error $|\nabla_x (u-u^*)|$. We only visualize a snapshot on a two-dimensional domain. Specifically, we visualize a bivariate function $u(x_1, x_2, 0, \cdots,0; 0)$ for $x_1,x_2\in[0,1]$ with the horizontal axis and vertical axis corresponding to $x_1$ and $x_2$ respectively.}
\label{fig:lqg-grad}
\end{figure}

\subsection{Ablation studies}
\label{sec:ablate}
We conduct ablation studies on the 100-dimensional LQG control problem (Eq. (\ref{eq:lqg})) to ablate the main designs in our algorithm. 
\paragraph{Adversarial training v.s. Directly optimizing $L^p$ physic-informed loss.} Our proposed algorithm introduces a min-max optimization procedure. One may have concerns that such an approach may be unnecessarily complicated, and directly minimizing $L^p$ physic-informed loss with a large $p$ would have the same effect. We use these two methods to solve Eq. (\ref{eq:lqg}), and compare their performance in the left panel of Table \ref{tab:ablate}. It can be seen that directly minimizing $L^p$ physic-informed loss does not lead to satisfactory results. 

We point out that this observation does not contradict our theoretical analysis (Theorem \ref{thm:stb0}). Theorem \ref{thm:stb0} focuses on the \textit{approximation} ability, which indicates that a model with a small $L^p$ loss can approximate the exact solution well. The empirical results in Table \ref{tab:ablate} demonstrate the \textit{optimization} difficulty of learning such a model with $L^p$ loss. By comparison, our proposed adversarial training method is more stable and leads to better performance. More detailed discussions are provided in Appendix \ref{app:supp-table2}.
	
\paragraph{Adversarial training should be applied to both the PDE residual and the boundary residual.} Our theoretical analysis suggests that we should use a large value of $p$ and $q$ in the loss $\ell_{\Omega,p}(u)$ and $\ell_{\partial\Omega,q}(u)$ to guarantee the quality of the learned solution $u$. Thus, in the proposed Algorithm \ref{alg:main}, both the data points inside the domain and the data points on the boundary are learned in the inner-loop maximization. From the right panel of Table \ref{tab:ablate}, we can see that when adversarial training is applied to one loss term, the performance is slightly improved, but its accuracy is still not satisfactory. When both loss terms use adversarial training, the solution is one order of magnitude more accurate, indicating that applying adversarial training to the whole loss function is essential.

\begin{table*}[t]
\caption{\textbf{Experimental results for ablation studies.} The left panel compares PINN trained with $L^p$ Physic Informed Loss and our method; the right panel compares PINN trained with partial or no adversarial training and our method. In the first two columns of the right panel, {\scriptsize\XSolidBrush}  indicates using the original $L^2$ Physics-Informed Loss for the PDE/boundary residual loss term, while {\scriptsize\Checkmark} indicates using the proposed adversarial training method for the corresponding loss term. Performances are measured by $L^1$ relative error. Best performances are indicated in \textbf{bold}.}
\label{tab:ablate}
\begin{minipage}{0.29\linewidth}
\centering
\begin{tabular}{cc}\toprule
Method          & Relative error\\\midrule
$L^4$ Loss      & 2.42\%     \\
$L^8$ Loss      & 53.55\%    \\
$L^{16}$ Loss   & 113.24\%   \\ \midrule
Ours & \textbf{0.27\%}\\ \bottomrule     
\end{tabular}
\end{minipage}\hfill
\begin{minipage}{0.43\linewidth}
\centering
\begin{tabular}{cccc}\toprule
\multicolumn{2}{c}{Adversarial training}& \multirow{2}{*}{Relative error}\\
Domain & Boundary   & \\\midrule
{\scriptsize\XSolidBrush} & {\scriptsize\XSolidBrush}  & 3.47\%     \\
{\scriptsize\XSolidBrush}  & {\scriptsize\Checkmark} & 2.79\%    \\\midrule
{\scriptsize\Checkmark} & {\scriptsize\Checkmark} & \textbf{0.27\%}\\ \bottomrule    
\end{tabular}
\end{minipage}\hspace{-0.33cm}
\begin{minipage}{0.29\linewidth}
    \centering
    \includegraphics[width=1\linewidth]{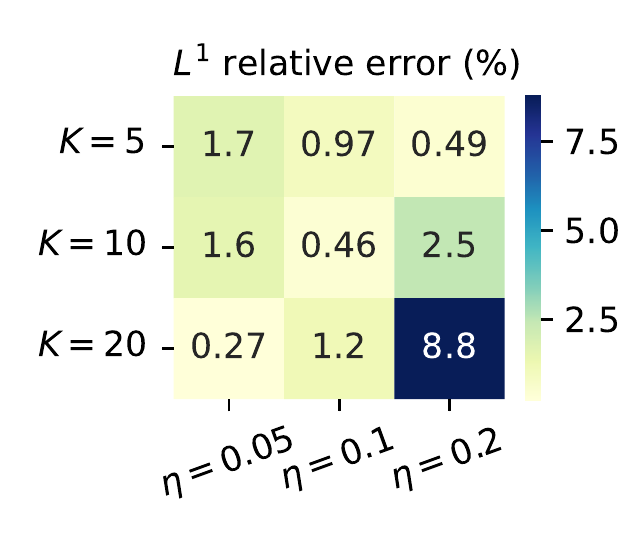}
\end{minipage}
\end{table*}

\paragraph{Hyper-parameters $K$ and $\eta$ for the inner loop maximization.}
Our approach introduces additional hyper-parameters $K$ (the number of inner-loop iterations) and $\eta$ (the inner-loop step size). These two parameters control the accuracy of inner-loop maximization.
We conduct ablation studies to examine the effects of different design choices. Specifically, we experiment with $K=5, 10, 20$ and $\eta=0.05, 0.1, 0.2$, and show the $L^1$ relative error in the right panel of Table \ref{tab:ablate}. Typically we find that setting the product $K\eta=1$ achieves the best performance. When $K\eta$ is fixed, our results suggest that using a larger $K$ and a smaller $\eta$, i.e., more inner-loop iterations and smaller step sizes, will lead to better performance while being more time-consuming.

\section{Conclusions}

\label{sec:conclusion}
In this paper, we theoretically investigate the relationship between the loss function and the approximation quality of the learned solution using the concept of stability in the partial differential equation. We study an important class of high-dimensional non-linear PDEs in optimal control, the Hamilton-Jacobi-Bellman (HJB) equation, and prove that for general $L^p$ Physics-Informed Loss, the HJB equation is stable only if $p$ is sufficiently large. Such a theoretical finding reveals that the widely used $L^2$ loss is not suitable for training PINN on high-dimensional HJB equations, while $L^{\infty}$ loss is a better choice. 
The theory also inspires us to develop a novel PINN training algorithm to minimize the $L^{\infty}$ loss for HJB equations in a similar spirit to adversarial training. One limitation of this work is that we only work on the HJB Equation. Theoretical investigation of other important equations can be an exciting direction for future works.
We believe this work provides important insights into the loss design in Physics-Informed deep learning.

\section*{Acknowledgements}
We thank Weinan E, Bin Dong, Yiping Lu, Zhifei Zhang, and Yufan Chen for the helpful discussions.

This work is supported by National Science Foundation of China (NSFC62276005), The Major Key Project of PCL (PCL2021A12), Exploratory Research Project of Zhejiang Lab (No. 2022RC0AN02), and Project 2020BD006 supported by PKUBaidu Fund.

\newpage
\bibliography{main}
\bibliographystyle{plain}

\newpage
\appendix

\section{Notation and Auxiliary Results}
\label{apdxA}

This section gives an overview of the notations used in the paper and summarizes some basic results in partial differential equation and functional analysis.

\subsection{Basic notations}

For $n\in\mathbb{N},$ we denote $\{1,2,...n\}$ by $[n]$ for simplicity in the paper.

For two Banach spaces $X,Y$, $\mathscr{L}(X,Y)$ refers to the set of continuous linear operator mapping from $X$ to $Y$.

For a mapping $f:X\to Y$, and $u,v\in X$, $df(u,v)$ denotes  Gateaux differential of $f$ at $u$ in the direction of $v$, and $f'(u)$ denotes Fréchet derivative of $f$ at $u$.

For $r>0$, and $x_0\in X$, where $X$ is a Banach space equipped with norm $\|\cdot\|_X$, $B_{r}(x_0)$ refers to $\{x\in X: \|x-x_0\|_X<r\}$.

For a function $f:X\to \mathbb{R}$, where $X$ is a measurable space. We denote by $\mathrm{supp}f$ the support set of $f$, i.e. the closure of ${\{x\in X:f(x)\neq 0\}}$.

For a measurable set $\Omega\subset\mathbb{R}^n$, define the parabolic region $Q_t=Q_t(\Omega)$ as $\Omega\times[0,t]$.

The parabolic boundary $\partial_p Q_t$ is then defined as $\Omega\times\{0\}\cup \partial\Omega\times[0,t]$.

For $R>0$, the parabolic neighborhood of 0 (denoted by $Q(R)$) is defined as
$\{(x,t)\in\mathbb{R}^n\times\mathbb{R}_{\geq 0}:|x|<R,\ t<R^2\}$. 

Its parabolic boundary
$\partial_p(Q(R))$ is defined as $\{(x,t):\ |x|<R,\ t=0\ $or$\ |x|=R,\ t\in[0,R^2]\}$.


\subsection{Multi-index notations}

For $n\in\mathbb{N}$, we call an $n-$tuple of non-negative integers $\alpha\in\mathbb{N}^n$ a multi-index. We use the notation $|\alpha|=\Sigma_{i=1}^n\alpha_i,\ \alpha!=\Pi_{i=1}^n\alpha_i!$. For $x=(x_1,x_2,...x_n)\in\mathbb{R}^n$, we denote by $x^\alpha=\Pi_{i=1}^n x_i^{\alpha_i}$ the corresponding multinomial. Given two multi-indices $\alpha,\beta\in\mathbb{N}^d$, we say $\alpha\leq\beta$ if and only if $\alpha_i\leq\beta_i,\ \forall i\in[n]$.

For an open set $\Omega\subset \mathbb{R}^n,\ T\in \mathbb{R}^{+}$ and a function $f(x):\Omega\to\mathbb{R}$ or $f(x,t):\Omega\times[0,T]\to \mathbb{R}$, we denote by
\begin{align}
D^{\alpha}f=\frac{\partial^{|\alpha|}f}{\partial x_1^{\alpha_1}...\partial x_n^{\alpha_n}}
\end{align}
the classical or weak derivative of $f$.

For $k\in \mathbb{R}^n$, we denote by $D^kf$ (or $\nabla^k f$) the vector whose components are $D^{\alpha}f$ for all $|\alpha|=k$, and we abbreviate $D^1f$ as $Df$.

\subsection{Norm notations}
\label{sblv}

Let $n\in\mathbb{N}^*,\ m\in \mathbb{N},\ T\in \mathbb{R}^{+}$, and $\Omega\subset \mathbb{R}^n,$ 
$Q\subset\mathbb{R}^n\times[0,T]$ be open sets.
We denote by $L^p(\Omega)$ and $L^p(Q)$ the usual Lebesgue space.

The Sobolev space $W^{m,p}(\Omega)$ is defined as
\begin{align}
    \{f(x)\in L^p(\Omega):D^{\alpha}f\in L^p(\Omega),\ \forall \alpha\in \mathbb{N}^n\ with\ |\alpha|\leq m\}.
\end{align}

And we define $W^{m,p}(Q)$ as
\begin{align}
    \{f(x)\in L^p(Q):D^{\alpha}f\in L^p(Q),\ \forall \alpha\in \mathbb{N}^n\ with\ |\alpha|\leq m\}.
\end{align}

We define
\begin{align}
    \|f\|_{W^{m,p}(\Omega)}:=\left(\sum_{|\alpha|\leq m}\|D^{\alpha}f\|^p_{L^p(\Omega)}\right)^{\frac 1 p}
\end{align}
for $1\leq p<\infty$, and
\begin{align}
    \|f\|_{W^{m,\infty}(\Omega)}:=\max_{|\alpha|\leq m}\|D^{\alpha}f\|_{L^\infty(\Omega)}
\end{align}
for $p=\infty$.

We define $\|f\|_{W^{m,p}(Q)}$ for $p\in[1,\infty]$ similarly.

We will use simplified notations $\|f\|_p$ and $\|f\|_{m,p}$ for $L^p-$norm and $W^{m,p}-$norm when the domain is whole space ($\mathbb{R}^n$, $\mathbb{R}^n\times[0,T]$ or $\mathbb{R}^n\times\mathbb{R}_{\geq 0}$) or when it is clear to the reader.

$W_0^{m,p}(\Omega)$ is defined as the completion of $C_0^{\infty}(\Omega)$ under $\|\cdot\|_{m,p}$ norm. In the similar way we define $W_0^{m,p}(Q_t)$.

\subsection{Auxiliary results}
In this section, we list out several fundamental yet important results in the field of PDE and functional analysis.

To begin with, we would like to recall three useful inequalities in real analysis.

\begin{lemma}[Young's convolution inequality]
\label{Lemma A.1}
In $\mathbb{R}^n$, we define the convolution of two functions $f$ and $g$ as $(f*g)(x):=\int_{\mathbb{R}^n}f(y)g(x-y)dy$.
Suppose $f\in L^p(\mathbb{R}^n)$,  $g\in L^q(\mathbb{R}^n)$, and $\frac 1 p+ \frac 1 q=\frac 1 r+1$ with $p,q,r\in[1,\infty]$, then$\|f*g\|_r\leq\|f\|_p\|g\|_q$.
\end{lemma}

\begin{lemma}[Sobolev embedding theorem]
\label{Lemma A.2}
Let $\Omega$ be an open set in $\mathbb{R}^n$,  $p\in[1,\infty],$ and $\ m\leq k$ be a non-negative integer.
\begin{enumerate}[(i)]
    \item If $\frac 1 p-\frac k n>0$, and set $q=\frac {np}{n-pk}$, then $W^{m,p}\subset W^{m-k,q}$ and the embedding is continuous, i.e. there exists a constant $c>0$ such that $\|u\|_{m-k,q}\leq c \|u\|_{m,p},\ \forall u\in W^{m,p}$.
    \item If $\frac 1 p-\frac k n\leq0$, then for any $q\in[1,\infty)$, $W^{m,p}\subset W^{m-k,q}$ and the embedding is continuous.
\end{enumerate}
\end{lemma}

\begin{lemma}[A special case of Gagliardo-Nirenberg inequality]
\label{Lemma A.3}
$\Omega$ is an open set in $\mathbb{R}^n$,
Let $q\in[1,\infty]$ and $\ j,k\in\mathbb{N}$, and suppose $j\neq 0$
and
$$
\begin{cases}
1<r<\infty\\
k-j-\frac n r\notin\mathbb{N}\\
\frac j k\leq \theta<1.
\end{cases}$$
If we set
\begin{align}
    \frac 1 p=\frac j n+\theta\left(\frac 1 r-\frac k n\right)+\frac{1-\theta}{q},
\end{align}
then there exists a constant $C$ independent of $u$ such that
\begin{align}
    \|\nabla^j u\|_p\leq C\|\nabla^k u\|_r^{\theta}\|u\|_q^{1-\theta},\ \forall u\in L^q(\Omega)\cap W^{k,r}(\Omega).
\end{align}
\end{lemma}

Next, we list out several basic results in second order parabolic equations, which will be of great use in Appendix C.

In the following lemmas, we denote by $\mathcal{L}_0$ the operator $\frac{\partial}{\partial t}-\Delta$.

\begin{lemma}
\label{Lemma A.4}
Suppose $\Omega$ is bounded, and $Q_T=\Omega\times[0,T]$. 

Let $u\in W^{2,2}(Q_T)$ be the solution to

\begin{equation}
\label{eqn:lma4}
   \begin{cases}
\mathcal{L}_0u=f(x,t),&\ \ (x,t)\in Q_T\\
u=0,&\ \ (x,t)\in \partial_p Q_T,
\end{cases} 
\end{equation}

then for $2\leq p<\infty$, if $f\in L^p(Q_T)$, we have $u\in W^{2,p}(Q_T)$ and there exists a constant $C$ such that $\|u\|_{2,p}\leq C \|f\|_p$.
\end{lemma}
\begin{proof}
Since $\Omega$ is bounded, we can choose an $R_0>1$ such that $Q_T\subset Q(R_0)$.
Set $\hat{u}(x,t)=u(x,t)\textbf{1}_{Q_T}$ and $\hat{f}(x,t)=f(x,t)\textbf{1}_{Q_T}$, which are extensions of $u$ and $f$ in $\mathbb{R}^n\times\mathbb{R}_{\geq 0}$, respectively. The boundary condition in Eq. (\ref{eqn:lma4}) implies that $\hat{u}\in W^{2,p}(\mathbb{R}^n\times\mathbb{R}_{\geq 0})$.

Furthermore, for any $R>R_0$, it holds that, 
$$
\begin{cases}
\mathcal{L}_0\hat{u}=\hat{f}(x,t),&\ \ (x,t)\in Q(R)\\
\hat{u}=0,&\ \ (x,t)\in \partial_p Q(R).
\end{cases}
$$

 From proposition 7.18 in \cite{lieberman1996second}, and in light of the fact that both $\hat{u}$ and $\hat{f}$ are supported on $Q_T$, we obtain that $\|D^2\hat{u}\|_p\leq C(\|\hat{f}\|_p+\frac 1 R \|D\hat{u}\|_p+\frac 1 {R^2}\|\hat{u}\|_p)$ holds for $\forall R>R_0$. Additionally, Poincaré inequality guarantees that there exists a constant $C'>0$ depending only on $\Omega$ and $ n$, such that $\|\hat{u}\|_{2,p}\leq C' \|D^2\hat{u}\|_p$.

Therefore we have
\begin{align}
    \frac 1 {C'}\|\hat{u}\|_{2,p}\leq C \left(\|\hat{f}\|_p+\frac 1 R \|\hat{u}\|_{2,p}+\frac 1 {R^2}\|\hat{u}\|_{2,p}\right)&\leq C \left(\|\hat{f}\|_p+\frac 2 R \|\hat{u}\|_{2,p}\right)\\
    \left(\frac 1 {C'}-\frac {2C}R\right) \|\hat{u}\|_{2,p}&\leq C\|\hat{f}\|_p.
\end{align}
Let $R\to\infty$ and we derive $\|\hat{u}\|_{2,p}\leq CC'\|\hat{f}\|_p$. Since $\|\hat{u}\|_{W^{2,p}(\mathbb{R}^n\times\mathbb{R}_{\geq 0})}=\|u\|_{W^{2,p}(Q_T)}$ and $\|\hat{f}\|_{L^p(\mathbb{R}^n\times\mathbb{R}_{\geq 0})}=\|f\|_{L^p(Q_T)}$, we completes the proof.
\end{proof}

\begin{lemma}
\label{Lemma A.5}
Let $u$ be the solution to
$$
\begin{cases}
\mathcal{L}_0u(x,t)=0,\quad & (x,t)\in \mathbb{R}^n\times[0,T]\\
u(x,0)=g(x),\quad & x\in \mathbb{R}^n.
\end{cases}$$
For any compact set $Q\subset\mathbb{R}^n\times[0,T]$ and $p\geq 1$, let $r<\frac{(n+2)p}{n+p}$, then
\begin{enumerate}[(i)]
    \item there exists a constant $C$ such that $\|u\|_{W^{1,r}(Q)}\leq C \|g\|_{L^p(\mathbb{R}^n)}$,
    \item there exists a constant $C'$ such that $\|u\|_{W^{2,r}(Q)}\leq C' \|g\|_{W^{1,p}(\mathbb{R}^n)}$.
\end{enumerate}
\end{lemma}

\begin{proof}
 $u$ have the explicit form
\begin{align}
\label{8}
u(x,t)=\int_{\mathbb{R}^n}\frac{a}{t^{\frac n 2}}e^{-b\frac{|x-y|^2}{t}}g(y)dy,
\end{align}
where $a=(4\pi)^{-\frac n 2}$ and $b=\frac 1 4$.

Note that $u$ is the convolution between heat kernel $K(x,t):=\frac{a}{t^{\frac n 2}}e^{-b\frac{|x|^2}{t}}$ and $g(x)$, with Lemma \ref{Lemma A.1}, we derive
\begin{align}
\label{9a}
\|u(\cdot, t)\|_{r_0}\leq\|g\|_{p_0}\|K(\cdot,t)\|_{q_0}
\end{align}
for $p_0,q_0,r_0\in[1,\infty]$ satisfying $\frac 1 {p_0}+ \frac 1 {q_0}=\frac 1 {r_0}+1$.

Due to the uniform convergence of (\ref{8}), we have $\frac {\partial u}{\partial x_i}=(\frac {\partial}{\partial x_i} K(x,t))*g(x)$ for all $i\in[n]$.
Thus we have
\begin{align}
\label{(10)}
    \|\frac {\partial u(\cdot,t)}{\partial x_i}\|_{r'}\leq \|g\|_{p'} \|\frac {\partial}{\partial x_i} K(x,t)\|_{q'}.
\end{align}
for $p',q',r'\in[1,\infty]$ satisfying $\frac {1}{p'}+ \frac {1}{q'} =\frac {1}{r'}+1$.

It is enough to decide appropriate tuples of $(p_0,q_0,r_0),(p',q',r')$.

For $q\geq 1$,
\begin{align}
    \|K(\cdot,t)\|_q^q=\frac {a^q}{t^{\frac {nq}{2}}}\int_{\mathbb{R}^n}e^{-bq\frac {\|x\|^2}{t}}dx
    =\frac {a^q}{t^{\frac {nq-n}{2}}}\int_{\mathbb{R}^n}e^{-bq\|y\|^2}dy\ \ \quad (y=\frac {1}{\sqrt{t}}x).
\end{align}
Thus $\|K(\cdot,t)\|_q=\frac {C}{t^{\frac{nq-n}{2q}}}$, where $C$ is a constant.

As a result, for any $p,q,r\in[1,\infty]$ satisfying $\frac 1 p+ \frac 1 q=\frac 1 r+1$,
\begin{align}
\label{25a}
    \|u\|_{L^r(\mathbb{R}^n\times[0,T])}^r=\int_0^T\|u(\cdot,t)\|_r^rdt
    \leq \|g\|_p^r\int_0^T\|K(\cdot,t)\|_q^rdt=C^r\|g\|_p^r\int_0^T \frac {dt}{t^{\frac{nq-n}{2q}r}}.
\end{align}
 Here we have $\|u\|_r\leq C_1 \|g\|_p$ for a constant $C_1$ $\iff$ the integral in the R.H.S. of Eq. (\ref{25a}) converges $\iff \frac{nq-n}{2q}r<1$. Then we could decide appropriate tuple $(p_0,q_0,r_0)$ for (\ref{9a}): tuples that satisfy $\frac {n}{n+2}\frac 1 {p_0}<\frac 1 {r_0}\leq \frac 1 {p_0}$. (The second inequality comes from the constraint $q_0\in [1,\infty]$. It could be removed when these $L^p-$norms are calculated in a bounded domain).

 We handle $(p',q',r')$ in (\ref{(10)}) with exactly the same method and find that tuples which satisfy $\frac {n}{n+2}\frac 1 {p'}+\frac {1}{n+2}<\frac 1 {r'}\leq \frac 1 {p'}$ gives 
 \begin{align}
 \label{a29}
\left\|\frac{\partial u}{\partial x_i}\right\|_{r'}\leq C_2 \|g\|_{p'},
\end{align}
where $C_2$ is a constant.

 Finally, note that for any bounded set $\Omega$, and $1\leq q<p$, there is a constant $C$ such that $\|v\|_{L^q(\Omega)}\leq C\|v\|_{L^p(\Omega)}$ for all $v\in L^p(\Omega)$. 
 Together with the inequalities (\ref{25a}) and (\ref{a29}), this means that 
 for any compact set $Q\subset\mathbb{R}^n\times[0,T]$, $p\geq 1$, and $r<\frac{(n+2)p}{n+p}$, 
 \begin{align}
     \|u\|_{W^{1,r}(Q)}&=(\|u\|_{L^r(Q)}^r+\sum_{i=1}^n\|\frac{\partial u}{\partial x_i}\|_{L^r(Q)}^r)^{\frac 1 r}\\
     &\leq (C_3\|u\|_{L^{r_0}(Q)}^{r}+C_4\sum_{i=1}^n\|\frac{\partial u}{\partial x_i}\|_{L^{r'}(Q)}^r)^{\frac 1 r}\\
     &\leq(C_3\|u\|_{L^{r_0}(\mathbb{R}^n\times[0,T])}^r+C_4\sum_{i=1}^n\|\frac{\partial u}{\partial x_i}\|_{L^{r'}(\mathbb{R}^n\times[0,T])}^r)^{\frac 1 r}\\
     &\leq (C_5\|g\|^r_{L^p(\mathbb{R}^n)}+C_6\sum_{i=1}^n\|g\|^r_{L^p(\mathbb{R}^n)})^{\frac 1 r}\\
     &=C_7\|g\|_{L^p(\mathbb{R}^n)},
 \end{align}
 where $r_0,r'\in[p,\frac{(n+2)p}{n+p})\cap[r,+\infty)$ and all $C_i$ are constants.

 This gives the first statement.
 
 Next, we prove the second statement. 
 
Note that $\frac {\partial^2 u}{\partial x_i x_j}=(\frac {\partial}{\partial x_i} K(x,t))*\frac{\partial g(x)}{\partial x_j},\ \forall i,j\in[n]$.

With the same argument for (\ref{a29}), we obtain that for any $r'',\ p''$ satisfying $\frac {n}{n+2}\frac 1 {p''}+\frac {1}{n+2}<\frac 1 {r''}\leq \frac 1 {p''}$,
\begin{equation}
 \label{a29_2}
\left\|\frac{\partial^2 u}{\partial x_ix_j}\right\|_{L^{r''}(\mathbb{R}^n\times[0,T])}\leq C \left\|\frac{\partial g(x)}{\partial x_j}\right\|_{L^{p''}(\mathbb{R}^n)},\quad \forall i,j\in[n],
\end{equation}
where $C$ is a constant.

Therefore, 
 for any compact set $Q\subset\mathbb{R}^n\times[0,T]$, $p\geq 1$, and $r<\frac{(n+2)p}{n+p}$, 
 \begin{align}
     &\|u\|_{W^{2,r}(Q)}\\
     =&\left(\|u\|_{L^r(Q)}^r+\sum_{i=1}^n\left\|\frac{\partial u}{\partial x_i}\right\|_{L^r(Q)}^r+
     \sum_{i,j=1}^n\left\|\frac{\partial^2 u}{\partial x_ix_j}\right\|_{L^r(Q)}^r\right)^{\frac 1 r}\\
     \leq & \left(C_1\|u\|_{L^{r_0}(Q)}^{r}+C_2\sum_{i=1}^n\left\|\frac{\partial u}{\partial x_i}\right\|_{L^{r'}(Q)}^r+
     C_3\sum_{i,j=1}^n\left\|\frac{\partial^2 u}{\partial x_ix_j}\right\|_{L^{r''}(Q)}^{r}\right)^{\frac 1 r}\\
     \leq &\left(C_1\|u\|_{L^{r_0}(\mathbb{R}^n\times[0,T])}^r+C_2\sum_{i=1}^n\left\|\frac{\partial u}{\partial x_i}\right\|_{L^{r'}(\mathbb{R}^n\times[0,T])}^r+
     C_3\sum_{i,j=1}^n\left\|\frac{\partial^2 u}{\partial x_ix_j}\right\|_{L^{r''}(\mathbb{R}^n\times[0,T])}^{r}\right)^{\frac 1 r}\\
     \leq  &\left(C_4\|g\|^r_{L^p(\mathbb{R}^n)}+C_5\sum_{i=1}^n\|g\|^r_{L^p(\mathbb{R}^n)}
     +C_6\sum_{i,j=1}^n\left\|\frac{\partial g}{\partial x_j}\right\|^r_{L^p(\mathbb{R}^n)}
     \right)^{\frac 1 r}\\
     \leq  & C_7\|g\|_{W^{1,p}(\mathbb{R}^n)},
 \end{align}
where $r_0,r',r''\in[p,\frac{(n+2)p}{n+p})\cap[r,+\infty)$ and all $C_i$ are constants.

This completes the proof.
 \end{proof}

At last, we present it here a well-known result in functional analysis.

\begin{lemma}[Inverse function theorem in Banach space]
\label{Lemma A.6}
Let $X,Y$ be two Banach spaces, $ V \subset X$ be an open set, and $g\in C^1(V,Y)$ be a mapping. Assume $x_0\in V,\ y_0=g(x_0)$ and the inverse of the  Fréchet derivative $(g'(x_0))^{-1}\in \mathscr{L}(Y,X)$. Then there exists $r>0$ and $s>0$ such that $B_r(y_0)\subset g(V)$, $B_s(x_0)\subset V$  and $g:B_s(x_0)\to g(B_s(x_0))$ is a differmorphism.
\end{lemma}


\section{Derivation of a Class of Hamilton-Jacobi-Bellman (HJB) Equations}
\label{app:derive-HJB}

For the sake of completeness of the paper, we give the derivation of a class of Hamilton-Jacobi-Bellman (HJB) Equations as below.


To start with, we derive the general form of HJB Equation in stochastic control problem.


In stochastic control, the state function $\{X_t\}_{0\leq t \leq T}$ is a stochastic process, where $T$ is the time horizon of the control problem. The evolution of the state function is governed by the following stochastic differential equation:
\begin{equation}
\label{eq:general-control1}
\left\{
\begin{array}{ll}
    \diff X_s = m(s,X_s)\diff s+\sigma \diff W_s & s\in[t,T] \\
    X_t=x
\end{array}
\right.,
\end{equation}
where $m:[t, T]\times\mathbb{R}^n\to \mathbb{R}^n$ is the control function and $\{W_s\}$ is a standard $n$-dimensional Brownian motion.

Given a control function $m=(m_1(s,y),m_2(s,y),...m_n(s,y)),\ s\in[t,T]$, $y\in\mathbb{R}^n$
, its total cost is defined as $J_{x,t}(m)=\E\int_t^T r(X_s,m(s,X_s),s)\diff s+ g(X_T)$,
where $r:\mathbb{R}^n\times\mathbb{R}^n\times [0,T]\to\mathbb{R}$ measures the cost rate during the process and $g:\mathbb{R}^n\to\mathbb{R}$ measures the final cost at the terminal state. The expectation is taken over the randomness of the trajectories. 

We are interested in finding a control function that minimizes the total cost for a given initial state. Formally speaking, we define the \textit{value function} of the control problem (\ref{eq:general-control1}) as $u(x,t)=\min\limits_{m\in \mathcal{M}}J_{x,t}(m)$, where $\mathcal{M}$ denotes the set of possible control functions that we take into consideration. 




It is obvious that $u$ satisfies $u(x,T)=g(x)$. In addition, according to
dynamical programming principle, we have
\begin{align}
\label{B29}
u(x,t)=\min\limits_{m\in \mathcal{M}}\mathbb{E}(\int_t^{t+h} r(X_s,m(s,X_s),s)ds+ u(X_{t+h},t+h)),
\end{align}

With Ito's formula, we derive
\begin{align}
\label{b16}
u(X_{t+h},t+h)=u(x,t)+(\partial_t u+\frac 1 2 \sigma^2 \Delta u)h+\nabla u\cdot(m(t,x) h+\sigma (W_{t+h}-W_t))+o(h)
\end{align}

After taking expectation and some calculation, we derive from (\ref{b16}),
\begin{align}
0=(\partial_t u+\frac 1 2 \sigma^2 \Delta u)h+\min\limits_{m\in \mathcal{M}}\mathbb{E}(\int_t^{t+h} r(X_s,m(s,X_s),s)ds+\nabla u\cdot m(t,x) h)+o(h)\\
0=\partial_t u(x,t)+\frac 1 2 \sigma^2 \Delta u(x,t)+\min\limits_{m\in \mathcal{M}}( r(x,m(t,x),t)+\nabla u\cdot m(t,x))
\end{align}

Then we get HJB equation
\begin{equation}
        \begin{cases}
    \label{B.hjb}
\partial_t u(x,t)+\frac 1 2 \sigma^2 \Delta u(x,t)+\min\limits_{m\in \mathcal{M}}( r(x,m(t,x),t)+\nabla u\cdot m(t,x))=0\\
u(x,T)=g(x).
\end{cases}
\end{equation}

Next, we further simplify this equation in some special cases.

In practice, different components of the state have different meanings, and thus the effects of controlling corresponding components have different significance. Therefore, the cost function's dependence on each component of $m_t$ takes a very different form.

Based on this argument, we consider the case when $r(x,y)$ takes the form
\begin{align}
\label{costB}
    r(x,y,t)=\sum_{i=1}^n a_i|y_i|^{\alpha_i}-\varphi(x,t)
\end{align}
for some appropriate function $\varphi$ and $a_i\geq0,\alpha_i>1$ (if $\alpha_i\leq 1$, the minimizing term might be $-\infty$),$\forall i\in[n]$.

Denote $m(t,x)=(m_1(t,x),m_2(t,x),...m_n(t,x))$ as $y\in \mathbb{R}^n$, and $\frac{\partial u(x,t)}{\partial x_i}$ as $\partial_i u$, and suppose that $\mathcal{M}$ is so large that it includes the global minimizor of (\ref{B29}), then the third term in HJB equation (\ref{B.hjb}) could be written as
\begin{align}
\label{b20}
    \mathop{\min}_{y\in \mathbb{R}^n}( -\varphi(x,t)+\sum_{i=1}^n (a_i|y_i|^{\alpha_i}+y_i\partial_i u))=\varphi(x,t)+\sum_{i=1}^n\mathop{\min}_{y_i\in \mathbb{R}}(a_i|y_i|^{\alpha_i}+y_i\partial_i u).
\end{align}

With some simple computation, we get
\begin{align}
    \mathop{\min}_{y_i\in \mathbb{R}}(a_i|y_i|^{\alpha_i}+y_i\partial_i u)=\left(\frac{a_i}{(a_i\alpha_i)^{\frac{\alpha_i}{\alpha_i-1}}}-\frac{1}{(a_i\alpha_i)^{\frac{1}{\alpha_i-1}}}\right)|\partial_i u|^{\frac{\alpha_i}{\alpha_i-1}}.
\end{align}
As a result, HJB equation in this case is
\begin{equation}
\label{B.hjb.derive}
  \begin{cases}
\partial_t u(x,t)+\frac 1 2 \sigma^2 \Delta u(x,t)-\varphi(x,t)-\sum_{i=1}^n
A_i|\partial_i u|^{c_i}=0\\
u(x,T)=g(x),
\end{cases}  
\end{equation}

where $A_i={(a_i\alpha_i)^{-\frac{1}{\alpha_i-1}}}-{a_i}{(a_i\alpha_i)^{-\frac{\alpha_i}{\alpha_i-1}}}\in(0,+\infty)$ and $c_i={\frac{\alpha_i}{\alpha_i-1}}\in(1,+\infty)$.



\begin{remark}
\label{1rk}
After taking the transform $v(x,t):=u(x,T-t)$, the equation above becomes
\begin{equation}
\label{B.ini.hjb}
    \begin{cases}
\partial_t v(x,t)-\frac 1 2 \sigma^2 \Delta v(x,t)+\sum_{i=1}^n
A_i|\partial_i v|^{c_i}=-\varphi(x,T-t)\\
v(x,0)=g(x).
\end{cases}
\end{equation}

We will study this equation in the rest of the paper instead.
\end{remark}

\begin{remark}
\label{2rk}
The minimizer of (\ref{b20}) is $y_i^*=(\frac{|\partial_i u|}{a_i\alpha_i})^{\frac {1}{\alpha_i-1}}$ and this gives the i-th component of the optimal control $m^*$. Based on the fact that both the value function $u$ and the optimal control $m^*$ are of interest in applications, it is necessary to study this equation in $W^{1,p}$ space.

Moreover, in most cases, only a bounded domain $\Omega\subset{\mathbb{R}^n}$ is taken into consideration in both real applications and numerical experiments. Therefore, we study this equation in the space of $W^{1,p}(\Omega\times[0,T])$ for a bounded domain $\Omega$, instead of $W^{1,p}(\mathbb{R}^n\times[0,T])$.
\end{remark}

\begin{remark}
\label{rk3}
The form of cost function (\ref{costB}) we investigate in the paper is a generalization of the widely-used power-law cost (or utility) function, which is representative in optimal control. For example, in financial markets, we often face power-law trading cost in optimal execution problems \cite{forsyth2012optimal,schied2009risk}. The cost function in Linear–Quadratic–Gaussian control and Merton's portfolio model (constant relative risk aversion utility function in \cite{merton1975optimum}) is also of this form. Therefore, we believe our theoretical analysis for this class of HJB equation is relevant for practical applications. 
\end{remark}

\section{Proof of Theorem \ref{thm:stb0}}
\label{app:proof-upper-bound}
In this section, we give the proof of an equivalent statement of Theorem \ref{thm:stb0}.

In light of remark \ref{1rk}, it is equivalent to consider the stability property (as is defined in Definition \ref{def_stb}) for the following equation:
\begin{equation}
\label{eqC}
    \begin{cases}
\partial_t u(x,t)-\frac 1 2 \sigma^2 \Delta u(x,t)+\sum_{i=1}^n
A_i|\partial_i u|^{c_i}=h(x,t)\quad \ (x,t)\in\mathbb{R}^n\times[0,T]\\
u(x,0)=g(x),
\end{cases}
\end{equation}

where $A_i>0$, $c_i\in(1,\infty)$, and $h(x,t)$ corresponds to $-\varphi(x,T-t)$ in Eq. (\ref{B.ini.hjb}). Without loss of generality, we assume $\sigma=\sqrt{2}$ for simplicity in the discussion below.

We define operators $\mathcal{L}_0:=\frac{\partial}{\partial t}-\Delta$, $\tilde{\mathcal{L}}_{\mathrm{HJB}}u:=\mathcal{L}_0u+\sum_{i=1}^n A_i|\partial_i u|^{c_i}$ and $\tilde{\mathcal{B}}_{\mathrm{HJB}}u(x,t):=u(x,0)$ for clarity. We define $\bar{c}$ as $\max\limits_{i\in[n]}c_i$.

We start with the proof of some auxiliary results.

\begin{lemma}
\label{Lemma C.1.}
For every $c>1$, there exist $k\in\mathbb{N}$, $\{t_i\}_{i=1}^{k}$ satisfying $1\leq t_1 <t_2<...<t_{k}<c$, and $k$ power functions $f_1,...f_{k}$ whose orders are strictly less than $c$ and no smaller than 0, such that
\begin{align}
    (b+w)^c-b^c-w^c\leq \sum_{i=1}^{k}f_i(b)w^{t_i},\ \forall b,w \geq 0.
\end{align}
\end{lemma}

\begin{proof}
Obviously, the inequality holds when $c\in \mathbb{N}$ because of the binomial expansion. We will only consider the case when $c\notin\mathbb{N}$.

\textit{Step 1.}We prove that for any $b,w\geq 0$ such that $\max\{b,x\}\geq 1$,
$\ F_{b,w}(c):=(b+w)^c-b^c-w^c $ is monotone increasing.

Without loss of generality, suppose $b\geq w$. Then
\begin{align}
    F'(c)&=(b+w)^c\ln(b+w)-b^c\ln b-w^c\ln w\\
    &=(b+\eta)^{c-1}(1+c\ \ln(b+\eta))w-w^c\ln w
\end{align}
holds for an $\eta\in(0,w)$ (mean value theorem).

Thus
\begin{align}
    F'(c)\geq w^c(1+c\ \ln(b+\eta))-w^c\ln w>w^c(1+\ln(b+\eta)-\ln w)>w^c>0.
\end{align}
The inequalities rely on the assumption $b\geq 1$, which means $\ln(b+\eta)>0$, and the first inequality comes from $b\geq w$.

This completes the proof in this step.

\textit{Step 2.} We then construct $k,\{t_i\}_{i=1}^k,\{f_i\}_{i=1}^k$ stated in the lemma.

Set $n=\lceil c\rceil$. By virtue of the increasing property of $F(c)$ when max$\{b,w\}\geq 1$, we get
\begin{align}
    F_{b,w}(c)\leq F_{b,w}(n)=\sum_{i=1}^{n-1} \binom{n}{i}b^{n-i}\cdot w^i.
\end{align}

When $b,w$ satisfies max$\{b,w\}< 1$,
\begin{align}
    F_{b,w}(c)\leq (b+w)^c-b^c=c(b+\eta)^{c-1}w<c 2^{c-1}w
\end{align}
holds for an $\eta\in(0,w)$.
The equality is an application of mean value theorem, and the second inequality comes from the fact that $\eta\in(0,w)$.

To conclude, $(b+w)^c-b^c-w^c\leq c 2^{c-1}w+\sum_{i=1}^{n-1}\binom{n}{i}b^{n-i}\cdot w^i,\ \forall b,w\geq 0.$
Since $c\notin\mathbb{N}$, which means $n-1<c$, this completes the proof.
\end{proof}

\begin{lemma}
\label{Lemma C.2.}
Suppose $u^{*}$ is the exact solution to
$$
\begin{cases}
\tilde{\mathcal{L}}_{\mathrm{HJB}}u=h\quad (x,t)\in \mathbb{R}^n\times[0,T]\\
\tilde{\mathcal{B}}_{\mathrm{HJB}}u=g
\end{cases}.
$$
Fix a bounded open set $\Omega\subset\mathbb{R}^n$.
Suppose $u_1$ 
satisfies
$\tilde{\mathcal{B}}_{\mathrm{HJB}}u^{*}=\tilde{\mathcal{B}}_{\mathrm{HJB}}u_1$ and that $\mathrm{supp}(u_1-u^{*})\subset Q_T(\Omega)$. 
Recall $c_i$ are parameters in the operator $\tilde{\mathcal{L}}_{\mathrm{HJB}}$.
Let $p\in [2,\infty)$.

If $p\geq n\cdot \mathop{\max}\limits_{i\in[n]}\frac{c_i-1}{c_i}=(1-{\bar{c}}^{-1})n$, then there exists $\delta_0>0$ such that, when $\|\tilde{\mathcal{L}}_{\mathrm{HJB}}u_1-h\|_p<\delta_0$, we have $\|u^{*}-u_1\|_{2,p}\leq C\|\tilde{\mathcal{L}}_{\mathrm{HJB}}u_1-h\|_p$ for a constant $C$ independent of $u_1$.
\end{lemma}

\begin{proof}
Define $w=w_{u_1}:=u_1-u^{*}$, then $\mathrm{supp}(w)$ is compact and $w(x,0)=0$. 
We further define $f=f_{u_1}:=\tilde{\mathcal{L}}_{\mathrm{HJB}}u_1-h$.
Since $u_1=u^{*}$ in $\mathbb{R}^n\times[0,T]\backslash Q_T(\Omega)$ and $f=\tilde{\mathcal{L}}_{\mathrm{HJB}}u_1-\tilde{\mathcal{L}}_{\mathrm{HJB}}u^{*}$, we have $\mathrm{supp}(f)\subset Q_T(\Omega)$. 
The $W^{m,p}$ and $L^p$ norm in the rest of the proof is defined on the domain $Q_T(\Omega).$

Compute
\begin{align}
    f=\tilde{\mathcal{L}}_{\mathrm{HJB}}u_1-\tilde{\mathcal{L}}_{\mathrm{HJB}}u^{*}=\mathcal{L}_0 w+\sum_{i=1}^n A_i|\partial_i (u^{*}+w)|^{c_i}-
    A_i|\partial_i u^{*}|^{c_i}
\end{align}

Thus, for any $(x,t)$,
\begin{align}
    |\mathcal{L}_0w(x,t)|&=|f-\sum_{i=1}^n (A_i|\partial_i (u^{*}+w)|^{c_i}-
    A_i|\partial_i u^{*}|^{c_i})|\bigg|_{(x,t)}\\
    &\leq |f(x,t)|+\sum_{i=1}^n A_i\|\partial_i (u^{*}+w)|^{c_i}-
    |\partial_i u^{*}|^{c_i}|\bigg|_{(x,t)}\\
   & \leq |f(x,t)|+\sum_{i=1}^n A_i((|\partial_i u^{*}|+|\partial_i w|)^{c_i}-
    |\partial_i u^{*}|^{c_i})\bigg|_{(x,t)}
\end{align}
where the second inequality could be derived from the fact that $(a+b)^c-a^c\geq a^c-(a-b)^c$ for $a\geq b\geq 0$ and $c\geq 1$.

For $i\in[n]$, apply Lemma \ref{Lemma C.1.} for $c=c_i$ and we obtain $k_i$ and
$\{t_{ij}\}_{j=1}^{k_i}$, $\{f_{ij}\}_{j=1}^{k_i}$ satisfying corresponding properties. 


We have
\begin{align}
    |\mathcal{L}_0w(x,t)|&\leq|f(x,t)|+\sum_{i=1}^{n}A_i(|\partial_iw|^{c_i}+\sum_{j=1}^{k_i}f_{ij}
    (|\partial_i u^{*}|)|\partial_i w|^{t_{ij}})\bigg|_{(x,t)}.
\end{align}

With Lemma \ref{Lemma A.4} and triangle inequality, we obtain
\begin{align}
\label{32}
    \|w\|_{2,p}\leq C \|\mathcal{L}_0 w\|_p
    \leq C( \|f\|_p+\sum_{i=1}^{n}A_i(\||\partial_iw|^{c_i}\|_p+\sum_{j=1}^{k_i}\|f_{ij} (|\partial_i u^{*}|)|\partial_i w|^{t_{ij}}\|_p)).
\end{align}
We will handle each term respectively.

Using Lemma \ref{Lemma A.2}, we have
\begin{align}
\label{33}
    \||\partial_iw|^{c_i}\|_p=\|\partial_iw\|_{c_ip}^{c_i}
    \leq \|w\|_{1,c_ip}^{c_i}\leq \hat{ C_i}\|w\|_{2,\frac{nc_ip}{n+c_ip}}^{c_i}.
\end{align}

for constants $\hat{ C_i}$.

Using Lemma \ref{Lemma A.2} and Hölder inequality, we have
\begin{align}
    \||f_{ij} (|\partial_i u^{*}|)|\partial_i w|^{t_{ij}}\|_p
    &\leq \|f_{ij}(|\partial_i u^{*}|)\|_{\infty}\||\partial_i w|^{t_{ij}}\|_p\\
    &=
    \|f_{ij}(|\partial_i u^{*}|)\|_{\infty}\|\partial_i w\|_{t_{ij}p}^{t_{ij}}
    \leq \Tilde{C}_{ij}\|w\|_{2,\frac{nt_{ij}p}{n+t_{ij}p}}^{t_{ij}}.\label{34}
\end{align}
for constants $\Tilde{C}_{ij}$ (Since $\overline{Q_T(\Omega)}$ is compact, we can tell that $ \|f_{ij}(|\partial_i u^{*}|)\|_{\infty}<\infty$ and thus $\tilde{C}_{ij}$ are well-defined).

When $p\geq n\cdot \mathop{\max}\limits_{i\in[n]}\frac{c_i-1}{c_i}$, because of $t_{ij}<c_i$, we have $\frac{nc_ip}{n+c_ip}\leq p,\ \frac{nt_{ij}p}{n+t_{ij}p}\leq p$ for all $i,j$.

Note that $\Omega$ is bounded, so for $1\leq q<p$,there is a constant $C$ such that $\|v\|_{L^q(\Omega)}\leq\|v\|_{L^p(\Omega)}$ for all $v\in L^p(\Omega)$.

As a consequence, we can derive from Eq. (\ref{32},\ref{33},\ref{34}) that $M:=\|w\|_{2,p}$ satisfies the inequality
\begin{align}
\label{(35)}
    K_0 M-\sum_{i=1}^{n}(K_i M^{c_i}+\sum_{1\leq j\leq k_i, t_{ij}>1} K_{ij}M^{t_{ij}})\leq \|f\|_p,
\end{align}
where all $K_i$ and $K_{ij}$ are positive constants depending only on $p,n,u^{*}$ and $\Omega$.

For clarity, We define  the L.H.S. of (\ref{(35)}) as a function $F$ with variable $M$.

With the observations (i)$ F(0)=0$, (ii) $F'(0)>0, $ (iii)$F''(M)<0$, (iv)$F(+\infty)=-\infty$, we could tell that 
$F(M)$ has a unique zero $m_0$ in $\mathbb{R}^{+}$. We could further tell that for any non-negative number $C\leq \max\limits_{M\in[0,m_0]}F(M)$, solving $F(M)\leq C\ (M\geq 0)$ derives $M\in[0,a]\cap[b,\infty]$ for some $0<a<b$ depending on $C$ and that $a\to 0,b\to m_0$ monotonously as $C$ decreases to $ 0$. 
Note that there exists $\delta>0$ such that $a\leq\frac {2}{K_0}C$ for $\forall C\in[0,\delta]$. 
In order to prove $\|w\|_{2,p}=O(\|f\|_p)$, it suffices to show that $\|w\|_{2,p}$ (i.e. $M$ in the discussion above) would not fall in the second interval providing $C$ (or $\|f\|_p$, correspondingly) is sufficiently small. We will prove by contradiction.

Note that $\tilde{\mathcal{L}}_{\mathrm{HJB}}$ is a continuous injection from $W_0^{2,p}(Q_T(\Omega))$ to $L^{p}(Q_T(\Omega))$
, and that there exists $r_0^*>0$ such that $\tilde{\mathcal{L}}_{\mathrm{HJB}}$ is a differmorphism from $B_{r_0}(u^{*})$ (in $W^{2,p}_0(Q_T(\Omega))$) to $\tilde{\mathcal{L}}_{\mathrm{HJB}}(B_{r_0}(u^{*}))\supset B_{r_1}(h)$ (in $L^p(Q_T(\Omega))$) for any $r_0\in (0,r_0^*)$ and any  $r_1\in(0,r_1^*) $ with $r_1^*$ depending on $r_0$ (this comes from an application of Lemma \ref{Lemma A.6}).


Select $r_0<\frac {m_0}{2}$ and determine the corresponding $r_1^*$.
By the property of $b$, there exists $\delta_0\in(0,\min\{\frac {r_1^*}2,\delta\})$ such that for any 
$C<\delta_0$, the corresponding $b$ is larger than $\frac 3 4 m_0$. If there exists 
$w\in W_0^{m,p}(Q_T(\Omega))$ satisfying $\|f\|_p=\|\tilde{\mathcal{L}}_{\mathrm{HJB}}(u^{*}+w)-\tilde{\mathcal{L}}_{\mathrm{HJB}}u^{*}\|_{p}:= C<\delta_0 $ while $\|w\|_{2,p}>b$ ($b$ depends on $C$), 
then there will also be a $w'\in B_{r_0}(0)$ such that $\tilde{\mathcal{L}}_{\mathrm{HJB}}(u^{*}+w')-\tilde{\mathcal{L}}_{\mathrm{HJB}}u^{*}=f$. We could tell from the difference between their norm that $w\neq w'$. This contradicts the property of injection.

The proof is completed.
\end{proof}

\begin{lemma}
\label{Lemma C.3.}
Suppose $u^{*}$ follows Lemma \ref{Lemma C.2.}, and $\Omega$ is a fixed bounded open set in $\mathbb{R}^n$.
Suppose $u_1$ 
satisfies
$\tilde{\mathcal{L}}_{\mathrm{HJB}}u^{*}=\tilde{\mathcal{L}}_{\mathrm{HJB}}u_1$ and that $\mathrm{supp}(u_1-u^{*})\subset Q_T(\Omega)$. 
Let $q\in [1,\infty)$.

If $\bar{c}\leq 2$ and $q> 
\frac {(\bar{c}-1)n^2}{(2-\bar{c})n+2}
$, there exists $\delta_0>0$ such that when $\|\tilde{\mathcal{B}}_{\mathrm{HJB}}u_1-\tilde{\mathcal{B}}_{\mathrm{HJB}}u^{*}\|_q<\delta_0$, 
we have $\|u^{*}-u_1\|_{1,r}\leq C\|\tilde{\mathcal{B}}_{\mathrm{HJB}}u_1-\tilde{\mathcal{B}}_{\mathrm{HJB}}u^{*}\|_q$ for a constant $C$ independent of $u_1$, where
$r<\frac{(n+2)q}{n+q}$.
\end{lemma}

\begin{proof}
Define $w=w_{u_1}:=u_1-u^{*}$ and $f=f_{u_1}:=\tilde{\mathcal{B}}_{\mathrm{HJB}}u_1-\tilde{\mathcal{B}}_{\mathrm{HJB}}u^{*}$. Let $w_1$ be the solution to
$$
\begin{cases}
\mathcal{L}_0 u=0, \quad (x,t)\in \mathbb{R}^n \times [0,T]\\
\tilde{\mathcal{B}}_{\mathrm{HJB}}u=f.
\end{cases}
$$
The $W^{m,p}$ and $L^p$ norm in the rest of the proof is defined on the domain $Q_T(\Omega)$.

Since the conditions $\bar{c}\leq 2$ and $q> 
\frac {(\bar{c}-1)n^2}{(2-\bar{c})n+2}
$ hold, we have $[(\bar{c}-1)n,\frac{(n+2)q}{n+q})\neq \varnothing$. Thus we could choose $r'\in [(\bar{c}-1)n,\frac{(n+2)q}{n+q})\cap[r,\infty)$. Since $Q_T(\Omega)$ is bounded, it suffices to bound $\|u_1-u^*\|_{r'}$.

To start with, from Lemma \ref{Lemma A.5}, we have $\|w_1\|_{1,r'}\leq C\|f\|_q$.

Then we bound the difference between $w_1$ and $w$. Define $v=w_1-w$, then $v$ satisfies
$$\begin{cases}
\mathcal{L}_0 v= \sum_{i=1}^n A_i|\partial_i (u^{*}+w)|^{c_i}-
    A_i|\partial_i u^{*}|^{c_i} \\
\tilde{\mathcal{B}}_{\mathrm{HJB}}v=0.
\end{cases}
$$
By Lemma \ref{Lemma A.2} and Lemma \ref{Lemma A.4} , we get $\|v\|_{1,r'}\leq C \|v\|_{2,\frac{nr'}{n+r'}}\leq C' \|\mathcal{L}_0 v\|_{\frac{nr'}{n+r'}}$.

Therefore we have $\|w\|_{1,r'}\leq \|w_1\|_{1,r'}+\|v\|_{1,r'}\leq C\|f\|_q+C'\|\mathcal{L}_0v\|_{\frac{nr'}{n+r'}}$.

Next, we give an estimation for $\|\mathcal{L}_0v\|_{\frac{nr'}{n+r'}}$.

Following from the proof in Lemma \ref{Lemma C.2.}, 
 we obtain $\{k_i\}_{i=1}^n\subset\mathbb{N}$, $\{t_{ij}\}_{1\leq i\leq n,1\leq j\leq k_i}\subset \mathbb{R}$, and power functions $\{f_{ij}\}_{1\leq i\leq n,1\leq j\leq k_i}$.

With similar computations, we have

\begin{align}
\label{70}
    \|\mathcal{L}_0 v\|_{\frac{nr'}{n+r'}}
    \leq & \sum_{i=1}^{n}A_i(\||\partial_iw|^{c_i}\|_{\frac{nr'}{n+r'}}+\sum_{j=1}^{k_i}\|f_{ij} (|\partial_i u^{*}|)|\partial_i w|^{t_{ij}}\|_{\frac{nr'}{n+r'}})\\
    \leq & 
    \sum_{i=1}^{n}A_i(\|\partial_i w\|^{c_i}_{\frac{c_inr'}{n+r'}}+\sum_{j=1}^{k_i}\|f_{ij} (|\partial_i u^{*}|)\|_{\infty}\|\partial_i w\|^{t_{ij}}_{\frac{nt_{ij}r'}{n+r'}}).
\end{align}
Because of $r'\geq (\bar{c}-1)n$ and $t_{ij}<c_i$, we have $\frac{nc_ir'}{n+r'}\leq r',\ \frac{nt_{ij}r'}{n+r'}\leq r'$ for all $i,j$.

Thus, due to the fact that $Q_T(\Omega)$ is bounded, all $\|\partial_i w\|_{\frac{c_inr'}{n+r'}}$ and $\|\partial_i w\|_{\frac{nt_{ij}r'}{n+r'}}$ could be bounded by $C_{i,j}\|w\|_{1,r'}$, where $C_{i,j}$ are constants.

With similar methods applied in Lemma \ref{Lemma C.2.}, we could prove this lemma.
\end{proof}

\begin{lemma}

\label{Theorem C.4.}
We denote by $u^{*}$ the exact solution to
$$
\begin{cases}
\tilde{\mathcal{L}}_{\mathrm{HJB}}u(x,t)=h(x,t)&(x,t)\in\mathbb{R}^n\times[0,T],\\
\tilde{\mathcal{B}}_{\mathrm{HJB}}u(x,t)=g(x)& x\in\mathbb{R}^n.
\end{cases}$$
Fix $\Omega$, which is an arbitrary bounded open set in $\mathbb{R}^n$. 
For two functions $\hat{f}_1(x,t),\ \hat{f}_2(x)$, denote by $u_1$ the solution to
$$
\begin{cases}
\tilde{\mathcal{L}}_{\mathrm{HJB}}u(x,t)=h(x,t)+\hat{f}_1(x,t),\ in\ \mathbb{R}^n\times[0,T]\\
\tilde{\mathcal{B}}_{\mathrm{HJB}}u(x,t)=g(x)+\hat{f}_2(x),\ in\ \mathbb{R}^n.
\end{cases}$$

For $p,q\geq 1$, let $r_0=\frac{(n+2)q}{n+q}$. Assume the following inequalities hold for $p,q$ and $r_0$:
\begin{equation}
\label{eq:main-thm-cond2}
    p\geq \max\left\{2, \left(1-\frac{1}{\bar{c}}\right)n\right\};~
    q> \frac{(\bar{c}-1)n^2}{(2-\bar{c})n+2};~
    \frac{1}{r_0}\geq \frac{1}{p}-\frac{1}{n},
\end{equation}
Further assume that $\bar{c}\leq 2$ and $\mathrm{supp}(u_1-u^*)\subset Q_T(\Omega)$.

Then for $\forall r\in[1,r_0)$, there exists $\delta_0>0$ such that, when $\|\hat{f}_1\|_p<\delta_0$ and $\|\hat{f}_2\|_q<\delta_0$, $\|u_1-u^*\|_{1,r}\leq C(\|\hat{f}_1\|_p+\|\hat{f}_2\|_q)$ for a constant $C$ independent of $u_1$.


\end{lemma}

\begin{proof}
It is straight-forward to define $u_2$ as the solution to
$$
\begin{cases}
\tilde{\mathcal{L}}_{\mathrm{HJB}}u(x,t)=h(x,t)+\hat{f}_1(x,t)&(x,t) \in\mathbb{R}^n\times[0,T]\\
\tilde{\mathcal{B}}_{\mathrm{HJB}}u(x,t)=g(x)& x\in\mathbb{R}^n
\end{cases}$$

 and bound $\|u^{*}-u_2\|_{1,r},\ \|u_2-u_1\|_{1,r}$ respectively.

From Lemma \ref{Lemma C.2.}, there exists $\delta_1>0, C_1>0$ such that $\|\hat{f}_1\|_p<\delta_1$ implies $\|u^{*}-u_2\|_{2,p}\leq C_1\|\hat{f}_1\|_p$. And from Lemma \ref{Lemma C.3.}, there exists $\delta_2>0, C_2>0$ such that $\|\hat{f}_2\|_q<\delta_2$ implies $\|u_2-u_1\|_{1,r}\leq C_2\|\hat{f}_2\|_q.$ By virtue of the condition $\frac 1 r>\frac 1 {r_0}\geq \frac 1 p-\frac 1 n$, with Lemma \ref{Lemma A.2} and the fact that we are considering $\|u^{*}-u_2\|_{1,r},\ \|u_2-u_1\|_{1,r}$ on a compact domain, providing $\|\hat{f}_1\|<\delta_1$ and $\|\hat{f}_2\|<\delta_2$, we derive 
\begin{align}
 \|u^{*}-u_1\|_{1,r}&\leq\|u^{*}-u_2\|_{1,r}+\|u_2-u_1\|_{1,r}\\
 &\leq
C\|u^{*}-u_2\|_{2,p}+\|u_2-u_1\|_{1,r}\\
&\leq CC_1\|\hat{f}_1\|_{p}+C_2\|\hat{f}_2\|_{q},  
\end{align}
where $C$ is a constant.

This concludes the proof.
\end{proof}

Finally, we give the proof of an equivalent statement of Theorem \ref{thm:stb0}. 



\begin{theorem}
\label{thm_stb}
Let $\hat{f}_1,\hat{f}_2,u^{*}$ and $u_1$ follow from Lemma \ref{Theorem C.4.}. Let $p,q,r_0$ satisfy the conditions in Lemma \ref{Theorem C.4.}. Assume $\bar{c}\leq 2$. For any bounded open set $Q\subset \mathbb{R}^n\times[0,T]$, it holds that 
for any $r\in[1,r_0)$, 
there exists $\delta>0$ and a constant $C$ independent of $u_1,\ \hat{f}_1$ and $\hat{f}_2$, such that $\max\{\|\hat{f}_1\|_{L^p(\mathbb{R}^n\times[0,T])},\|\hat{f}_2\|_{L^q(\mathbb{R}^n)}\}<\delta$ implies $\|u_1-u^*\|_{W^{1,r}(Q)}\leq C(\|\hat{f}_1\|_{L^p(\mathbb{R}^n\times[0,T])}+\|\hat{f}_2\|_{L^q(\mathbb{R}^n)})$.

\end{theorem}

\begin{proof}

Since $Q$ is bounded, there exists $R>0$ such that $Q\subset Q(R)$. 

Let $\hat{u}_1$ be the constraint of $u_1$ in $Q(R)$. Construct an extension $v$ of $\hat{u}_1$ to $\mathbb{R}^n\times\mathbb{R}_{\geq0}$ such that
\begin{enumerate}[(i)]
    \item  $v=u^{*}$ in $(\mathbb{R}^n\times\mathbb{R}_{\geq0})\backslash Q(2R)$,
    \item $\|\tilde{f}_1\|_p
 \leq C' \|\hat{f}_1\|_{L^p(Q)}$ and $
 \|\tilde{f}_2\|_
 \leq C' \|\hat{f}_2\|_{L^q(B_R(0))}$ for a constant $C'$ depending only on $n,R,p,q,Q$, where $\tilde{f}_1:=\tilde{\mathcal{L}}_{\mathrm{HJB}}v-\hat{f}_1,\ \tilde{f}_2:=\tilde{\mathcal{B}}_{\mathrm{HJB}}v-\hat{f}_2$.
\end{enumerate}
 Note that $\mathrm{supp}(\tilde{f}_1)\subset \overline {Q(2R)}$ and $\mathrm{supp}(\tilde{f}_2)\subset B_{2R}(0)$, the existence of $v$ is obvious. 
 
 From Lemma \ref{Theorem C.4.}, there exists $C>0$ and $\delta>0$ such that  $\|\tilde{f}_1\|_p<\delta$ and $\|\tilde{f}_2\|_q<\delta$ imply $\|v-u^{*}\|_{1,r}\leq C(\|\tilde{f}_1\|_p+\|\tilde{f}_2\|_q)$. 
 Thus
 \begin{align}
 \|u_1-u^{*}\|_{W^{1,r}(Q)}=\|v-u^{*}\|_{W^{1,r}(Q)}\leq \|v-u^{*}\|_{1,r}\leq C(\|\tilde{f}_1\|_p+\|\tilde{f}_2\|_q)\\
 \leq
 CC'(\|\hat{f}_1\|_{L^p(Q)}+\|\hat{f}_2\|_{L^q(B_R(0))})\leq CC'(\|\hat{f}_1\|_p+\|\hat{f}_2\|_q).
 \end{align}
The proof is completed.
\end{proof}

\section{Proof of Theorem \ref{thm:lower-bound}}
\label{app:proof-lower-bound}
In this section, we give the proof of Theorem \ref{thm:lower-bound}. 

Based on remark \ref{1rk}, it is equivalent to consider Eq. (\ref{eqC}). 
We will show that the following equation satisfies the properties stated in Theorem \ref{thm:lower-bound},
\begin{equation}
\label{d.eq}
\begin{cases}
 \partial_t u-\Delta u+ |Du|^2=0\ \  in\ \mathbb{R}^n\times[0,T]\\
 u(x,0)=g(x).
\end{cases}    
\end{equation}

This equation is a special case for Eq. (\ref{eqC}) with $A_i=1,\ c_i=2,\ \forall i\in[n]$ and $h(x,t)\equiv0$.


Denote by $u^{*}$ the exact solution to the equation above. The notations $\mathcal{L}_0,\tilde{\mathcal{L}}_{\mathrm{HJB}},\tilde{\mathcal{B}}_{\mathrm{HJB}}$ in the following discussion have the same meaning as in section \ref{app:proof-upper-bound}.

We prove some auxiliary results first.

\begin{lemma}
\label{Lemma D.1.}
For $p\in [2,2n),$ and an open set or a parabolic region $\mathfrak{A}$, we denote the function space $W^{2,p}(\mathfrak{A})$ as $X$ and $L^{\frac{np}{2n-p}}(\mathfrak{A})$ as $Y$.
 For any $u\in X$, we have $\|u-u'\|_X\geq A\sqrt{\|\tilde{\mathcal{L}}_{\mathrm{HJB}}u-\tilde{\mathcal{L}}_{\mathrm{HJB}}u'\|_Y+B}-C$ holds for $\forall u'\in X$, where $A,B,C$ are positive constants depending on $u$.
\end{lemma}


\begin{proof}
We divide the proof into two steps.

\textit{Step 1.}
We check that $\tilde{\mathcal{L}}_{\mathrm{HJB}}$ as an operator mapping from $X$ to $Y$ is Fréchet-differentiable.

For any $u,v\in X,\ t\in \mathbb{R}$, since $v\in X$, which means $|Dv|^2\in W^{1,\frac p 2}\subset Y$, we have
\begin{align}
\|\tilde{\mathcal{L}}_{\mathrm{HJB}}(u+tv)-\tilde{\mathcal{L}}_{\mathrm{HJB}}u-t\cdot (\mathcal{L}_0v+2Du\cdot Dv )\|_Y=\|t^2|Dv|^2\|_Y=o(t)\ \ (t\to 0).
\end{align}
Therefore, $d\tilde{\mathcal{L}}_{\mathrm{HJB}}(u,v)= \mathcal{L}_0v+2Du\cdot Dv $ by definition.

Define operator $A(u)\in\mathscr{L}(X,Y)$ as $A(u)v=dL(u,v)$.
For $u,u'\in X$ and any $v\in X$, note that
\begin{align}
\|A(u')v-A(u)v\|_Y=2\|D(u'-u)\cdot Dv\|_Y\leq C_0\|D(u'-u)\cdot Dv\|_{1,\frac p 2}\\
\leq C_0\|D(u'-u)\|_{1,p}\|Dv\|_{1,p}\leq C_0 \|u'-u\|_X\|v\|_X
\end{align}
for a constant $C_0$, where the second inequality comes from Cauchy-Schwarz inequality.
Thus we have $\|A(u)-A(u')\|_{\mathscr{L}(X,Y)}\leq C_0\|u-u'\|_X$, which means $A$ is continuous with regard to $u$.
As a result, $\tilde{\mathcal{L}}_{\mathrm{HJB}}$ is Fréchet-differentiable and $\tilde{\mathcal{L}}_{\mathrm{HJB}}'(u)=A(u)$.
Moreover, we derive $\|L'(u)\|\leq\|L'(0)\|+\|L'(0)-L'(u)\|\leq C_0\|u\|_X+C_1$.\\

\textit{Step 2.}
For any $u,u'\in X$, let $y=\tilde{\mathcal{L}}_{\mathrm{HJB}}u,\ y'=\tilde{\mathcal{L}}_{\mathrm{HJB}}u'$.

Define $u_\eta=(1-\eta)u+\eta u'$ and $y_\eta=\tilde{\mathcal{L}}_{\mathrm{HJB}}u_\eta$ for $\eta\in[0,1]$. Fix a number $m\in\mathbb{N}$.

From the property proved in step 1, for any $\eta \in[0,1]$, there exists $r_\eta \in(0,\frac 1 m)$ such that
\begin{align}
\|\tilde{\mathcal{L}}_{\mathrm{HJB}}v-\tilde{\mathcal{L}}_{\mathrm{HJB}}u_{\eta}\|\leq2\|L'(u_{\eta})\|\cdot\|u_{\eta}-v\|\leq 2(C_0\|u_{\eta}\|+C_1)\|v-u_{\eta}\|,\ \ \forall v\in B_{r_\eta}(u_{\eta}).
\end{align}
Note that $\{B_{\frac {r_\eta}2 }(u_{\eta}):\eta \in[0,1]\}$ is an open cover of $\{u_{\eta}:\eta \in[0,1]\}$. Because of the compactness of $\{u_{\eta}:\eta \in[0,1]\}$, we obtain an increasing finite sequence $\{\eta _i\}_{i=0}^{N}$ with $\eta _0=0,\eta _N=1$ such that either $u_{\eta _i}\in B_{r_{\eta _{i-1}}}(u_{\eta _{i-1}})$ or $u_{\eta _{i-1}}\in B_{r_{\eta _{i}}}(u_{\eta _{i}})$ for $\forall i\in [N]$. This means that
\begin{align}
\|y_{\eta _{i}}-y_{\eta _{i-1}}\|\leq 2(C_0\|u_{\eta _j}\|+C_1)\|u_{\eta _{i}}-u_{\eta _{i-1}}\|,\ j\in\{i-1,i\}
\end{align}
holds for $\forall i\in[N]$.

Therefore
\begin{align}
\label{eqn:4}
\|y'-y\|\leq \sum_{i=1}^{N} \|y_{\eta _i}-y_{\eta _{i-1}}\|
\leq \sum_{i=1}^N 2(C_0\|u_{\eta _j}\|+C_1)\|u_{\eta _{i}}-u_{\eta _{i-1}}\|
\end{align}
Note that this inequality holds for every $m$.
As $m\to\infty$, R.H.S. of (\ref{eqn:4}) converges to
\begin{align}
&\|u-u'\|\int_{0}^1 2(C_0\|u+s(u'-u)\|+C_1)ds.\\
=&\|u-u'\|(C_0\|u+\theta(u'-u)\|+C_1),\ \ (\theta\in (0,1))\\
\label{44}
\leq &\|u-u'\|(C_0(\|u\|+\|u'-u\|)+C_1).
\end{align}
The equality comes from mean value theorem for integral, and the inequality comes from triangular inequality.

Combining Eq. (\ref{eqn:4}) and (\ref{44}) together completes the proof.
\end{proof}

\begin{lemma}
\label{Lemma D.2.}
Let $u_1$ 
satisfies
$\tilde{\mathcal{B}}_{\mathrm{HJB}}u^{*}=\tilde{\mathcal{B}}_{\mathrm{HJB}}u_1$ and that $\mathrm{supp}(u_1-u^{*})$ is compact in $\mathbb{R}^n\times\mathbb{R}^+$.
Define $f(x,t):=\tilde{\mathcal{L}}_{\mathrm{HJB}}u_1$.
Let $p\in [2,\infty),m\in\mathbb{N}$.
If $p\geq \frac n 2$ then there exists $\delta_0>0$ such that $\|f\|_{m,p}<\delta_0$ implies 
$\|u^{*}-u_1\|_{m+2,p}=O(\|f\|_{m,p})$.
\end{lemma}

\begin{proof}
When $m=0$, this statement is a direct consequence of Lemma \ref{Lemma C.2.}.

When $m> 0$, for every multi-index $\alpha$ with $|\alpha|\leq m$, operate $D^{\alpha}$
on both sides of $\tilde{\mathcal{L}}_{\mathrm{HJB}}u_1=f$ and $\tilde{\mathcal{L}}_{\mathrm{HJB}}u^{*}=0$. We then obtain
\begin{align}
    \label{46}
    \mathcal{L}_0 D^{\alpha}u_1+ D^{\alpha}|Du_1|^2&=D^{\alpha}f\\
    \label{47}
    \mathcal{L}_0 D^{\alpha}u^{*}+ D^{\alpha}|Du^{*}|^2&=0.
\end{align}

Define $w:=u_1-u^{*}$, and compute the difference between (\ref{46}) and (\ref{47}), we get
\begin{align}
\mathcal{L}_0D^{\alpha}w=D^{\alpha}f-\sum_{i=1}^{n}(2D^{\alpha}(\partial_i u^{*}\partial_i w)+D^{\alpha}(\partial_iw)^2).
\end{align}
With similar methods used in Lemma \ref{Lemma C.2.}, we could bound $\|D^{\alpha}w\|_{2,p}$ with $\|D^{\alpha}f\|_p$, based on which we complete the proof.
\end{proof}

Finally, we show that Eq. (\ref{d.eq}) satisfies the properties stated in Theorem \ref{thm:lower-bound}, which will conclude the proof for Theorem \ref{thm:lower-bound}.

\begin{theorem}
\label{thm:lower-bound_apdx}
For any $\varepsilon>0,A>0,r\geq 1,m\in\mathbb{N}$ and $p\in\left[1,\frac n 4\right]$, there exists a function $u\in C^{\infty}(\mathbb{R}^n\times(0,T])$ which satisfies the following conditions:
\begin{itemize}
    \item $\|\tilde{\mathcal{L}}_{\mathrm{HJB}}u\|_{L^p(\sR^n\times[0,T])}<\varepsilon$,  $\tilde{\mathcal{B}}_{\mathrm{HJB}}u=\tilde{\mathcal{B}}_{\mathrm{HJB}}u^{*}$, and $\mathrm{supp}(u-u^{*})$ is compact.
    \item $\|u-u^{*}\|_{W^{m,r}(\sR^n\times[0,T])}>A$.
\end{itemize}
\end{theorem}



\begin{proof}
Since $L^1(\Omega)$ has the weakest topology in function spaces $W^{m,r}(\Omega)$ when $\Omega$ is bounded, it is enough to consider the case for $r=1,\ m=0.$

Set ${p_0}=\frac 5 9 n,\ p_1=\frac{11}9 n$ and $p_2=\frac{11} 7 n$.\\ 

\textit{Step 1.}
We construct two families of functions $\{v_{a,c}\},\{F_{a,c}\}$ 
as the basis of proof.

For any $a,c>0$, define $f_{a,c}(x)=c|x|^{-0.7},x\in \overline{B_1(0)}\backslash B_{a}(0)$
in $\mathbb{R}^n$. We could extend it to a $C^{\infty}$ function $\Tilde{f}_{a,c}(x)$ defined on $\mathbb{R}^n$ such that (i)$\|\Tilde{f}_{a,c}\|_{\infty}<\|{f}_{a,c}\|_{\infty}+\min\{1,c\}$ and (ii)$\mathrm{supp}(\Tilde{f}_{a,c})\subset B_{1.1}(0)$. 

We could further construct a $C^{\infty}$ function $\hat{f}_{a,c}(x,t)$ such that
\begin{enumerate}[(i)]
    \item $\mathrm{supp}(\hat{f}_{a,c})\subset B_{1.1}(0)\times (0,T]$,
    \item $\hat{f}_{a,c}(x,t)=\Tilde{f}_{a,c}(x),\forall t\in[\frac T 2,T]$, 
    \item $\|\hat{f}_{a,c}(x,t)\|_{L^{\infty}(\mathbb{R}^n)\times[0,T]}\leq \|\Tilde{f}_{a,c}\|_{L^{\infty}(\mathbb{R}^n)}$.
\end{enumerate}

Define $u_{a,c}$ as the solution to
$$
\begin{cases}
\tilde{\mathcal{L}}_{\mathrm{HJB}}u=\hat{f}_{a,c} \ \ in\ \mathbb{R}^n\times[0,T]\\
\tilde{\mathcal{B}}_{\mathrm{HJB}}u=g.
\end{cases}
$$
Select a function $w_{a,c}\in C^{\infty}(\mathbb{R}^{n}\times\mathbb{R})$ with compact support $\mathrm{supp}(w_{a,c})\subset\mathbb{R}^{n}\times(0,T]$, such that $\|u_{a,c}-u^{*}-w_{a,c}\|<\epsilon_c$ in  $W^{3,4{p_0}}(\mathbb{R}^n\times [0,T]),W^{2,4p_1}(\mathbb{R}^n\times [0,T])$ and $W^{2,4p}(\mathbb{R}^n\times [0,T])$,
where $\epsilon_c$ is an small value depending on $c$ and is to be decided later.

We define $v_{a,c}=u^{*}+w_{a,c}$ and $F_{a,c}=\tilde{\mathcal{L}}_{\mathrm{HJB}}v_{a,c}$.\\

\textit{Step 2.}
We show that
 $\{v_{a,c}\}$ and $\{F_{a,c}\}$ have following properties:

\begin{enumerate}[(i)]
    \item $\mathrm{supp}(v_{a,c}-u^{*})$ is compact in $\mathbb{R}^{n}\times(0,T]$.
    \item $\tilde{\mathcal{B}}_{\mathrm{HJB}}v_{a,c}=\tilde{\mathcal{B}}_{\mathrm{HJB}}u^{*}$.
    \item $\mathrm{supp}(F_{a,c})$ is compact in $\mathbb{R}^{n}\times(0,T]$.
    \item There exists a constant $M<\infty$ such that $\|F_{a,c}\|_q<cM$ and $\|F_{a,c}\|_{1,{p_0}}<cM$.
    \item For any $c>0$, $\|F_{a,c}\|_{p_2}\to\infty$ as $a\to 0$.
\end{enumerate}

(i) and (ii) comes directly from the construction of $v_{a,c}$.

Because $\mathrm{supp}(v_{a,c}-u^{*})$ is close, for any $(x,t)\in (\mathbb{R}^{n}\times[0,T])\backslash \mathrm{supp}(v_{a,c}-u^{*})$, there exists $r>0$ such that $(B_{r}(x,t)\cap (\mathbb{R}^{n}\times[0,T]))\subset (\mathbb{R}^{n}\times[0,T])\backslash \mathrm{supp}(v_{a,c}-u^{*})$
, which means $v_{a,c}=u^{*}$ in $B_{r}(x,t)\cap (\mathbb{R}^{n}\times[0,T])$ and thus $\tilde{\mathcal{L}}_{\mathrm{HJB}}v_{a,c}=\tilde{\mathcal{L}}_{\mathrm{HJB}}u^{*}=0$. This gives (iii).

Due to the fact that the function $|x|^{-0.7}\in L^p(B_2(0))\cap W^{1,{p_0}}(B_2(0))$, there exists a constant $M<\infty$ such that $\|\hat{f}_{a,1}\|_p<M-1$ and $\|\hat{f}_{a,1}\|_{1,{p_0}}<M-1$ holds for any $a$ and any construction of $\hat{f}_{a,1}$ based on $f_{a,1}$.
Due to the linearity of norms, we derive  $\|\hat{f}_{a,c}\|_p<c(M-1)$ and $\|\hat{f}_{a,c}\|_{1,{p_0}}<c(M-1)$.

It is easy to check that $\tilde{\mathcal{L}}_{\mathrm{HJB}}$ is a continuous mapping from $W^{3,4{p_0}}(\Omega)$ to $W^{1,{p_0}}(\Omega)$, from $W^{2,4p_1}(\Omega)$ to $L^{p_2}(\Omega)$ and from $W^{2,4p}(\Omega)$ to $L^p(\Omega)$ for any compact set $\Omega\subset\mathbb{R}^n\times [0,T]$. Therefore, 
$\|F_{a,c}-\hat{f}_{a,c}\|$ is small
in $W^{1,{p_0}}(\mathbb{R}^n\times [0,T]),\ L^{p_2}(\mathbb{R}^n\times [0,T]),\ $ and $L^p(\mathbb{R}^n\times [0,T])$.

Since $\||x|^{-0.7}\|_{L^{p_2}(B_1(0))}=+\infty$, by the construction of $\hat{f}_{a,c}$ we have $\|\hat{f}_{a,c}\|_{p_2}\to +\infty$ as $a \to 0$.

As a result of the continuity of $\tilde{\mathcal{L}}_{\mathrm{HJB}}$, we could guarantee $\|F_{a,c}\|_p<cM$, $\|F_{a,c}\|_{1,{p_0}}<cM$ and $\|F_{a,c}\|_{p_2}>\frac 1 2 \|\hat{f}_{a,c}\|_{p_2}$ by choosing $\epsilon_c$ sufficiently small previously. This gives (iv) and (v).\\

\textit{Step 3.} We give an estimation for $\|v_{a,c}-u^{*}\|_1$, i.e., $\ \|w_{a,c}\|_1$).

We mention at the beginning of this part that all $C_i$ appeared below are positive constants.

For any $\epsilon>0$,
set $c$ to $\frac 1 {2M} \min\{\epsilon,\delta_0\}$, where $\delta_0$ follows from an application of Lemma \ref{Lemma D.2.} for the case $p=p_0$.
Then for any $a>0$, $\|F_{a,c}\|_{p}<\epsilon$ and we obtain $\|w_{a,c}\|_{3,{p_0}}\leq C_0 \|F_{a,c}\|_{1,{p_0}}$.

In Lemma \ref{Lemma A.3}, we choose $j=2,\ k=3,\ \theta=\frac{n+\frac{13}{11}}{n+\frac 6 5},\  r={p_0},\ q=1,\ p=\frac{11}9 n$ and derive
$ \|\nabla^2 w_{a,c}\|_{p_1}\leq C_1\|\nabla^3 w_{a,c}\|_{p_0}^{\theta}\|w_{a,c}\|_1^{1-\theta}.$
Since $\|\nabla^3 w_{a,c}\|_{p_0}\leq \|w_{a,c}\|_{3,{p_0}}$, we get
\begin{align}
    \|w_{a,c}\|_1\geq\left(\frac{\|\nabla^2 w_{a,c}\|_{p_1}}{C_1 \|w_{a,c}\|_{3,{p_0}}^{\theta}}\right)^{\frac 1 {1-\theta}}
    \geq C_2 \left(\frac{\|\nabla^2 w_{a,c}\|_{p_1}}{ \|F_{a,c}\|_{1,{p_0}}^{\theta}}\right)^{\frac 1 {1-\theta}}
    \geq\frac{C_2}{(C_3M)^{\frac{1-\theta}{\theta}}}\|\nabla^2w_{a,c}\|_{p_1}^{\frac 1 {1-\theta}},
\end{align}
where the last inequality comes from property (iv) in Step 2.

By virtue of property (i) in Step 2, we have $\|\nabla^2w_{a,c}\|_{p_1}\geq C_4 \|w_{a,c}\|_{2,p_1}$, which is an application of Poincaré inequality. Together with Lemma \ref{Lemma D.1.},
\begin{align}
   \|w_{a,c}\|_1 \geq\frac{C_2}{(C_3M)^{\frac{1-\theta}{\theta}}}\|\nabla^2w_{a,c}\|_{p_1}^{\frac 1 {1-\theta}}
   \geq C_5 \|w_{a,c}\|_{2,p_1}^{\frac 1 {1-\theta}}
   \geq C_5 \left(C_6\sqrt{\|F_{a,c}\|_{p_2}}-C_7\right)^{\frac 1 {1-\theta}}.
\end{align}

Since R.H.S. above goes to $+\infty$ as $a\to 0$ due to property (v), for any $A>0$, there exists $a_0>0$ such that $\|w_{a_0,c}\|_1>A$.

Setting $u=v_{a_0,c}=u^*+w_{a_0,c}$ completes the proof.
\end{proof}

\section{Improved Theorem \ref{thm:stb0}}
\label{new-thm}
In this section, we give the stability result for Eq. (\ref{eqC}) (Theorem \ref{thm_stb_new}). Different from Theorem \ref{thm_stb}, the constraint $\bar{c}\leq 2$ is released here.

 The notations $u^*,\ \bar{c},\  \mathcal{L}_0,\tilde{\mathcal{L}}_{\mathrm{HJB}},\tilde{\mathcal{B}}_{\mathrm{HJB}}$ in the following discussion come from \ref{app:proof-upper-bound}.

The proof of Theorem \ref{thm_stb_new} is quite similar to that of Theorem \ref{thm_stb}.

We begin with some auxiliary results.

\begin{lemma}
\label{lemmaE}
Suppose $\Omega$ is a fixed bounded open set in $\mathbb{R}^n$.
Suppose $u_1$ 
satisfies
$\tilde{\mathcal{L}}_{\mathrm{HJB}}u^{*}=\tilde{\mathcal{L}}_{\mathrm{HJB}}u_1$ and that $\mathrm{supp}(u_1-u^{*})\subset Q_T(\Omega)$.
Let $q\in [1,\infty)$.

If $q>\frac{(\bar{c}-1)n^2}{n+2\bar{c}} 
$, then there exists $\delta_0>0$ such that when $\|\tilde{\mathcal{B}}_{\mathrm{HJB}}u_1-\tilde{\mathcal{B}}_{\mathrm{HJB}}u^{*}\|_{1,q}<\delta_0$, 
we have $\|u^{*}-u_1\|_{2,r}\leq C\|\tilde{\mathcal{B}}_{\mathrm{HJB}}u_1-\tilde{\mathcal{B}}_{\mathrm{HJB}}u^{*}\|_{1,q}$ for a constant $C$ independent of $u_1$, where
$r<\frac{(n+2)q}{n+q}$.
\end{lemma}

\begin{proof}
The proof is almost the same as that for Lemma \ref{Lemma C.3.}.

Following its proof,
we define $w,\ f,\ w_1$ and $v$ similarly.
And the $W^{m,p}$ and $L^p$ norm in the rest of the proof will also be defined on the domain $Q_T(\Omega)$.

Since $q>\frac{(\bar{c}-1)n^2}{n+2\bar{c}}$, we have $[(1-{\bar{c}}^{-1})n,\frac{(n+2)q}{n+q})\neq \varnothing$. 

Thus we could choose $r'\in [(1-{\bar{c}}^{-1})n,\frac{(n+2)q}{n+q})\cap[r,\infty)$.

Since $Q_T(\Omega)$ is bounded, it suffices to bound $\|u_1-u^*\|_{r'}$.

From Lemma \ref{Lemma A.5}, we have $\|w_1\|_{2,r'}\leq C\|f\|_{1,q}$.
And from Lemma \ref{Lemma A.4} , we get $\|v\|_{2,r'}\leq C' \|\mathcal{L}_0 v\|_{r'}$.

Therefore we have $\|w\|_{2,r'}\leq \|w_1\|_{2,r'}+\|v\|_{2,r'}\leq C\|f\|_{1,q}+C'\|\mathcal{L}_0v\|_{r'}$.

Next, we give an estimation for $\|\mathcal{L}_0v\|_{r'}$.

Following from the proof in Lemma \ref{Lemma C.2.}, 
 we obtain $\{k_i\}_{i=1}^n\subset\mathbb{N}$, $\{t_{ij}\}_{1\leq i\leq n,1\leq j\leq k_i}\subset \mathbb{R}$, and power functions $\{f_{ij}\}_{1\leq i\leq n,1\leq j\leq k_i}$.

With similar computations, we have
\begin{align}
\label{70new}
    \|\mathcal{L}_0 v\|_{r'}
    \leq & \sum_{i=1}^{n}A_i(\||\partial_iw|^{c_i}\|_{r'}+\sum_{j=1}^{k_i}\|f_{ij} (|\partial_i u^{*}|)|\partial_i w|^{t_{ij}}\|_{r'})\\
    \leq & 
    \sum_{i=1}^{n}A_i(\|\partial_i w\|^{c_i}_{c_ir'}+\sum_{j=1}^{k_i}\|f_{ij} (|\partial_i u^{*}|)\|_{\infty}\|\partial_i w\|^{t_{ij}}_{t_{ij}r'}).
\end{align}

Due to the fact that $Q_T(\Omega)$ is bounded, all $\|\partial_i w\|_{c_ir'}$ and $\|\partial_i w\|_{t_{ij}r'}$ could be bounded by $C_{i,j}\|w\|_{1,\bar{c}r'}$, where $C_{i,j}$ are constants.

Moreover, since $r'\geq(1-{\bar{c}}^{-1})n$, we could tell from Lemma \ref{Lemma A.2} that
\begin{align}
   \|w\|_{1,\bar{c}r'}\leq \hat{C} \|w\|_{2,\frac{n\bar{c}r'}{n+\bar{c}r'}}\leq \hat{C}'\|w\|_{2,r'} ,
\end{align}
where $\hat{C}$ and $\hat{C}'$ are constants.

With similar methods applied in Lemma \ref{Lemma C.2.}, we could prove this lemma.
\end{proof}

\begin{lemma}
\label{Theorem C.4._new}

Fix $\Omega$, which is an arbitrary bounded open set in $\mathbb{R}^n$. 
For two functions $\hat{f}_1(x,t),\ \hat{f}_2(x)$, denote by $u_1$ the solution to
$$
\begin{cases}
\tilde{\mathcal{L}}_{\mathrm{HJB}}u(x,t)=h(x,t)+\hat{f}_1(x,t),\ in\ \mathbb{R}^n\times[0,T]\\
\tilde{\mathcal{B}}_{\mathrm{HJB}}u(x,t)=g(x)+\hat{f}_2(x),\ in\ \mathbb{R}^n.
\end{cases}$$

For $p,q\geq 1$, let $r_0=\frac{(n+2)q}{n+q}$. Assume the following inequalities hold for $p,q$ and $r_0$:
\begin{equation}
\label{eq:main-thm-cond3}
    p\geq \max\left\{2, \left(1-\frac{1}{\bar{c}}\right)n\right\};~
    q> \frac{(\bar{c}-1)n^2}{n+2\bar{c}}.
\end{equation}
Further assume $\mathrm{supp}(u_1-u^*)\subset Q_T(\Omega)$.

Then for $\forall r\in[1,\min\{r_0,p\})$, there exists $\delta_0>0$ such that, when $\|\hat{f}_1\|_p<\delta_0$ and $\|\hat{f}_2\|_q<\delta_0$, $\|u_1-u^*\|_{2,r}\leq C(\|\hat{f}_1\|_p+\|\hat{f}_2\|_{1,q})$ for a constant $C$ independent of $u_1$.
\end{lemma}
\begin{proof}
The proof follows as in Lemma \ref{Theorem C.4.} by replacing the use of Lemma \ref{Lemma C.3.} with Lemma \ref{lemmaE}.
\end{proof}

\begin{theorem}
\label{thm_stb_new}
Let $\hat{f}_1,\hat{f}_2$ and $u_1$ follow from Lemma \ref{Theorem C.4._new}. 
For $p,q,r\geq 1$, let $r_0=\frac{(n+2)q}{n+q}$. Assume the following inequalities hold for $p,q,r$ and $r_0$:
\begin{equation}
\label{eq:main-thm-cond1}
    p\geq \max\left\{2, \left(1-\frac{1}{\bar{c}}\right)n\right\};~
    q> \frac{(\bar{c}-1)n^2}{n+2\bar{c}};~
    \frac 1 r >\frac 1 {\min\{r_0,p\}}-\frac 1 n.
\end{equation}

Then for any bounded open set $Q\subset \mathbb{R}^n\times[0,T]$, 
there exists $\delta>0$ and a constant $C$ independent of $u_1,\ \hat{f}_1$ and $\hat{f}_2$, such that $\max\{\|\hat{f}_1\|_{L^p(\mathbb{R}^n\times[0,T])},\|\hat{f}_2\|_{W^{1,q}(\mathbb{R}^n)}\}<\delta$ implies $\|u_1-u^*\|_{W^{1,r}(Q)}\leq C(\|\hat{f}_1\|_{L^p(\mathbb{R}^n\times[0,T])}+\|\hat{f}_2\|_{W^{1,q}(\mathbb{R}^n)})$.

\end{theorem}

\begin{proof}
By replacing the use of Lemma \ref{Theorem C.4.} in the proof for Theorem \ref{thm_stb} with Lemma \ref{Theorem C.4._new}, we can bound $\|u_1-u^*\|_{W^{2,r'}(Q)}$ with 
$\|\hat{f}_1\|_{L^p(\mathbb{R}^n\times[0,T])}$ and $\|\hat{f}_2\|_{W^{1,q}(\mathbb{R}^n)}$ for any $r'\in [1,\min\{r_0,p\})$.

We could further bound $\|u_1-u^*\|_{W^{1,r}(Q)}$ with the help of Lemma \ref{Lemma A.2}. 

This concludes the proof.
\end{proof}

\section{Experimental Settings}
\label{app:exp-settings}
\paragraph{Hyperparameters.} The hyperparameters used in our experiment is described in Table \ref{tab:settings}.

\begin{table}[h]
    \caption{\textbf{Derailed experimental settings} of Section \ref{sec:exp}.}
    \label{tab:settings}
        \begin{center}
            \begin{tabular}{lcc}
                \toprule
                 & $n=100$ & $n=250$ \\
                \midrule
                \textit{Model Configuration} \\
                \midrule
                \textbf{Layers} & \multicolumn{2}{c}{4} \\
                \textbf{Hidden dimension} & \multicolumn{2}{c}{4096} \\
                \textbf{Activation} & \multicolumn{2}{c}{$\mathrm{tanh}$} \\
                \midrule
                \textit{Hyperparameters} \\
                \midrule
                \textbf{Toal iterations} & 5000 & 10000 \\
                \textbf{Domain Batch Size} $N_1$ & 100 & 50 \\
                \textbf{Boundary Batch Size} $N_2$ & 100 & 50 \\
                \textbf{Inner Loop Iterations $K$} & \multicolumn{2}{c}{20} \\
                \textbf{Inner Loop Step Size $\eta$} & \multicolumn{2}{c}{0.05} \\
                \textbf{Learning Rate} & \multicolumn{2}{c}{$7\mathrm{e}-4$} \\
                \textbf{Learning Rate Decay} & \multicolumn{2}{c}{Linear}\\
                \textbf{Adam $\varepsilon$} & \multicolumn{2}{c}{$1\mathrm{e}-8$} \\
                \textbf{Adam($\beta_1$, $\beta_2$)} & \multicolumn{2}{c}{(0.9, 0.999)} \\
                \bottomrule
            \end{tabular}
    \end{center}
\end{table}

\paragraph{Training data.} In all the experiments, the training data is sampled \textit{online}. Specifically, in each iteration, we sample $N_1$ i.i.d. data points, $(x^{(1)},t^{(1)}), \cdots, (x^{(N_1)},t^{(N_1)}) $, from the domain $\sR^n\times [0,T]$, and $N_2$ i.i.d. data points, $(\tilde x^{(1)},T), \cdots, (\tilde x^{(N_2)},T)$, from the boundary $\sR^n\times \{T\}$, where $(x^{(i)},t^{(i)}) \sim \mathcal{N}(\vzero,\mI_{n})\times \mathcal{U}(0,1)$ and $\tilde x^{(j)} \sim \mathcal{N}(\vzero,\mI_{n})$.

\paragraph{Evaluation metrics.} We use $L^1$, $L^2$, and $W^{1,1}$ relative errors to evaluate the quality of the learned solution. 

$L^1$ and $L^2$ relative errors are two popular evaluation metrics, which are defined as 
\begin{align}
 \frac{\sum_{j=1}^S|u^*(x_j)-u_\theta(x_j)|^p}{\sum_{j=1}^S |u^*(x_j)|^p},\ \ p=1,2,
\end{align}
where $u_\theta$ is the learned approximate solution, $u^*$ is the exact solution and $\{x_j\}_{j=1}^S$ are $S$ i.i.d. uniform samples from the domain $[0,1]^n\times[0,T]$.

Since the gradient of the solution to HJB equations plays an important role in applications, we also evaluate the solution using $W^{1,1}$ relative error, which is defined as
\begin{align}
\frac{\sum_{j=1}^S (|u^*(x_j)-u_\theta(x_j)|+\sum_{i=1}^n|\partial_{x_i}u^*(x_j)- \partial_{x_i}u_\theta(x_j)|)}{\sum_{j=1}^S(|u^*(x_j)|+\sum_{i=1}^n|\partial_{x_i}u^*(x_j) |)}.
\end{align}

\section{More experiments and visualizations}
\label{app:more-exp}

\subsection{More instance of HJB Equations}

\begin{figure*}[ht]
\begin{minipage}{0.3\linewidth}
    \centering
    \includegraphics[width=1\linewidth]{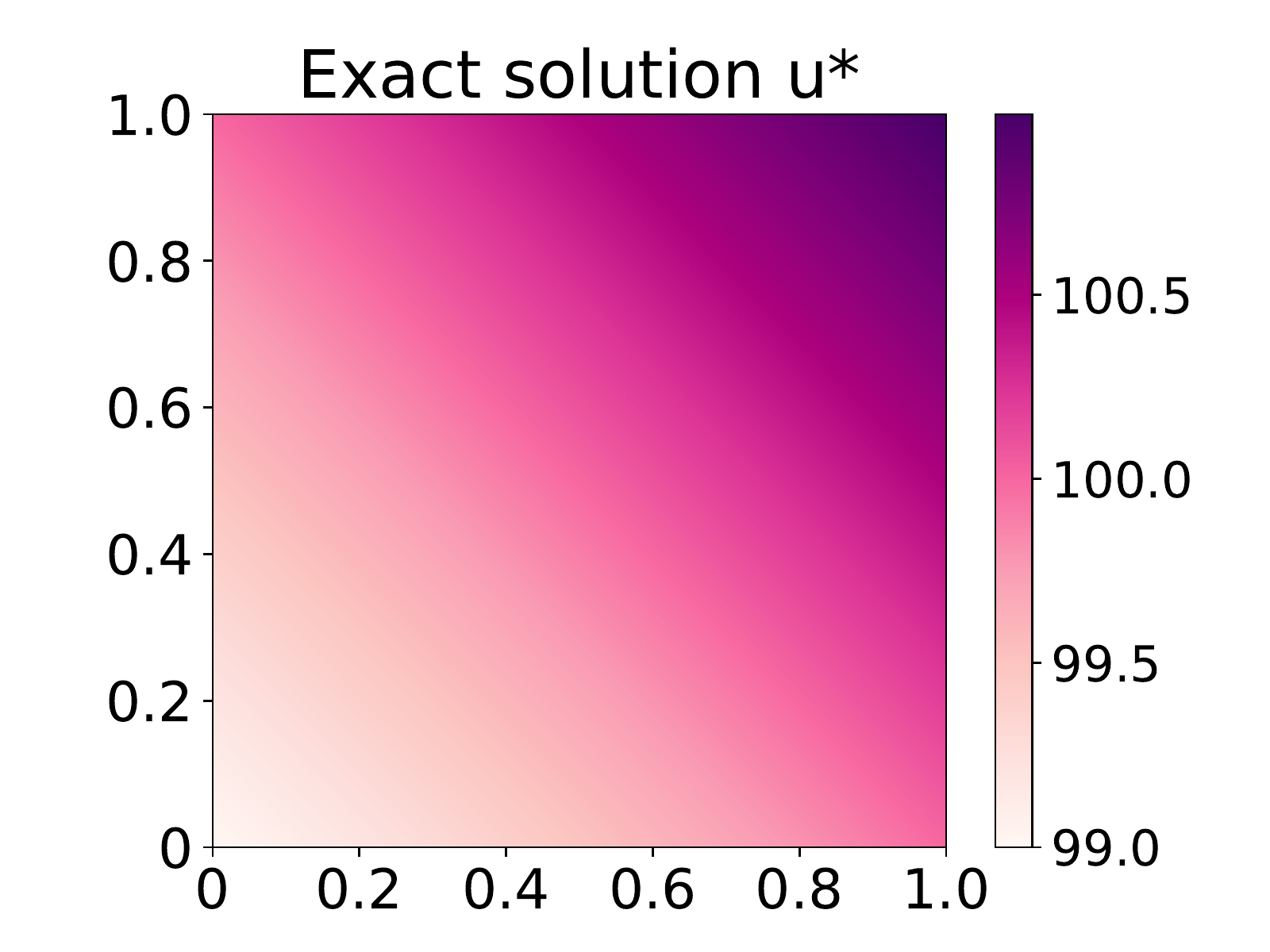}
\end{minipage}\hfill
\begin{minipage}{0.3\linewidth}
    \centering
    \includegraphics[width=1\linewidth]{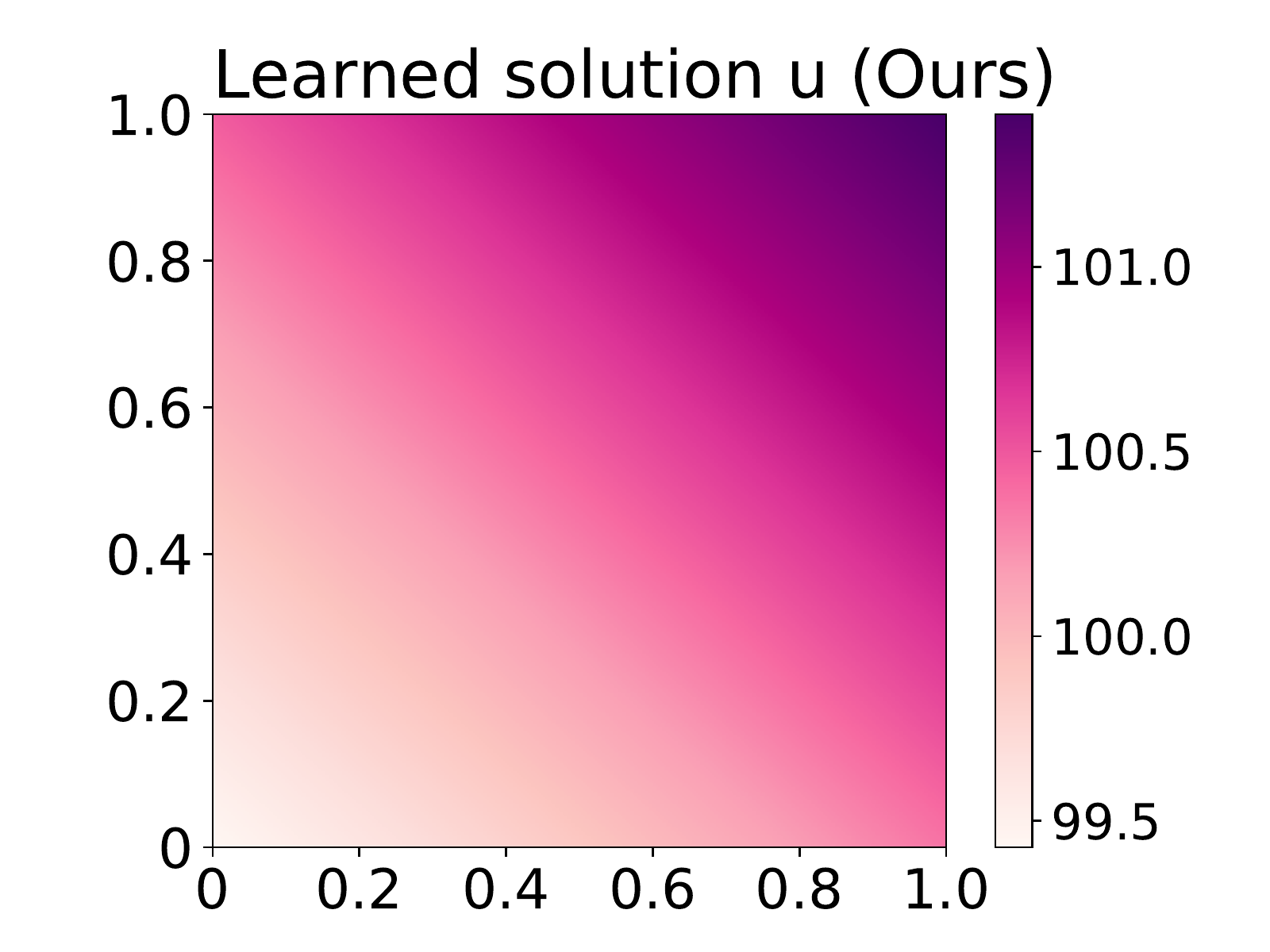}
\end{minipage}\hfill
\begin{minipage}{0.3\linewidth}
    \centering
    \includegraphics[width=1\linewidth]{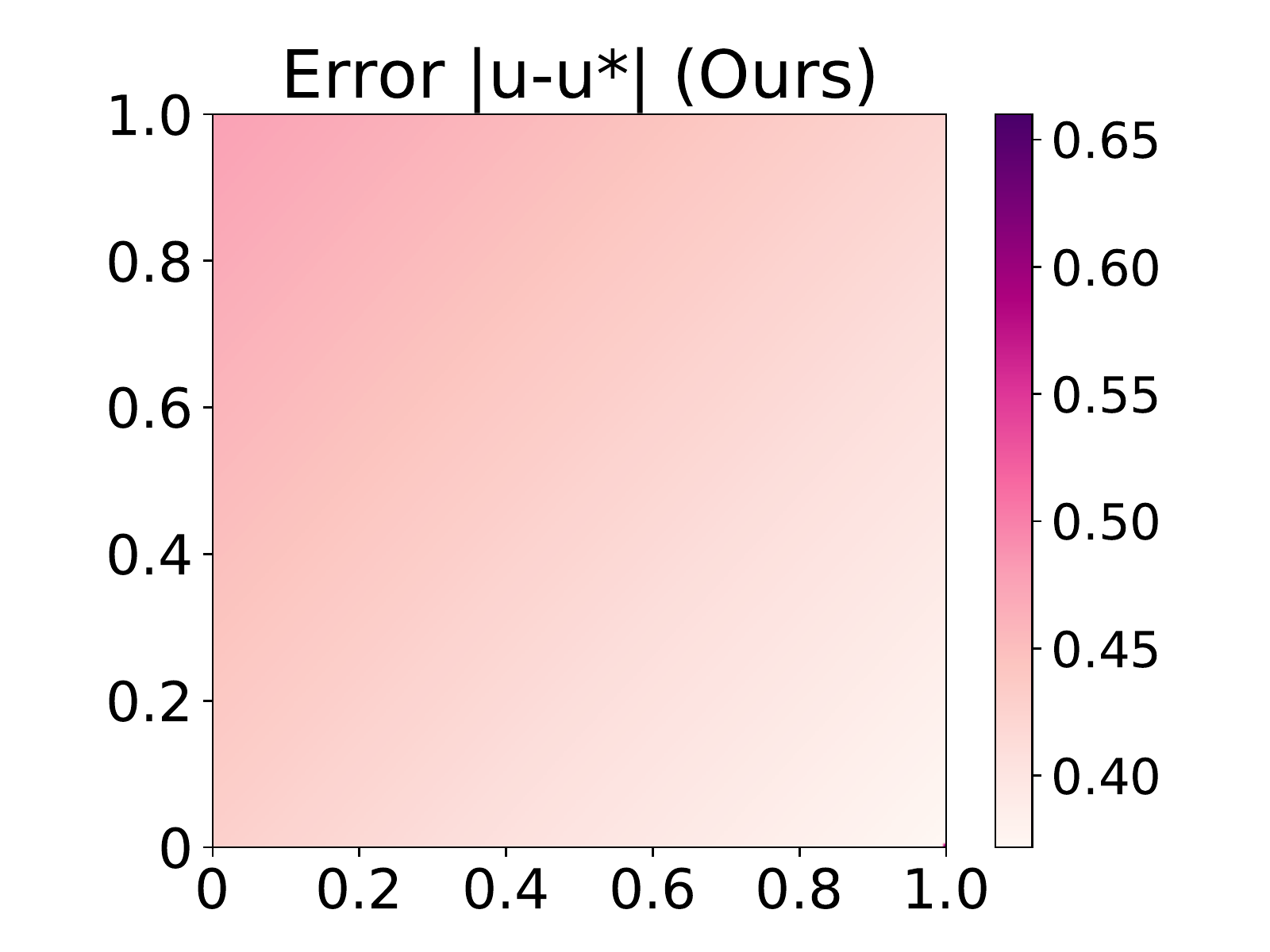}
\end{minipage}
\caption{\textbf{Visualization for the solution snapshot of Eq. (\ref{eq:hjb-exp2})}. $c$ is set to 1.25.}
\label{fig:1.25hjb}
\end{figure*}

\begin{figure*}[ht]
\begin{minipage}{0.3\linewidth}
    \centering
    \includegraphics[width=1\linewidth]{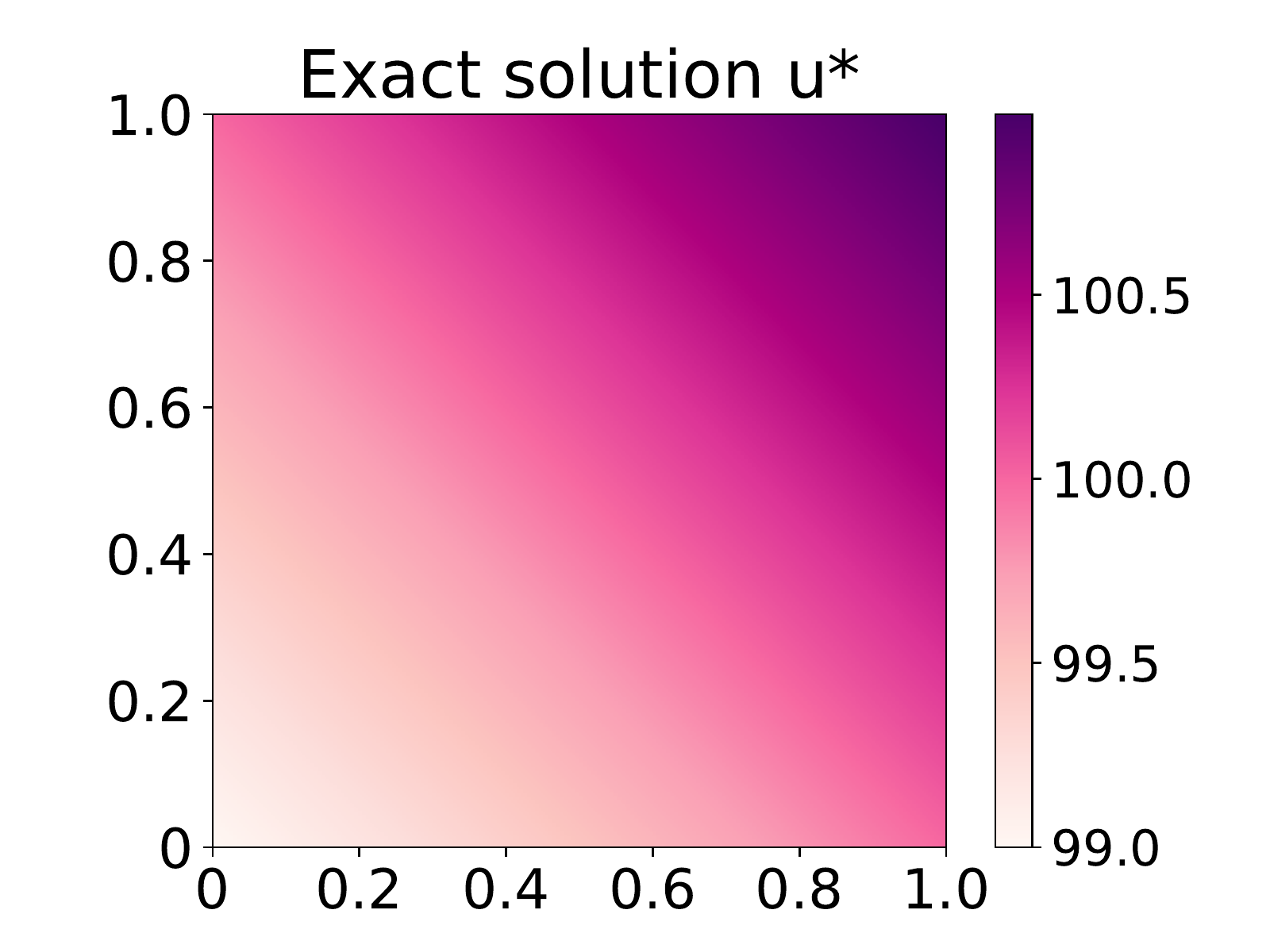}
\end{minipage}\hfill
\begin{minipage}{0.3\linewidth}
    \centering
    \includegraphics[width=1\linewidth]{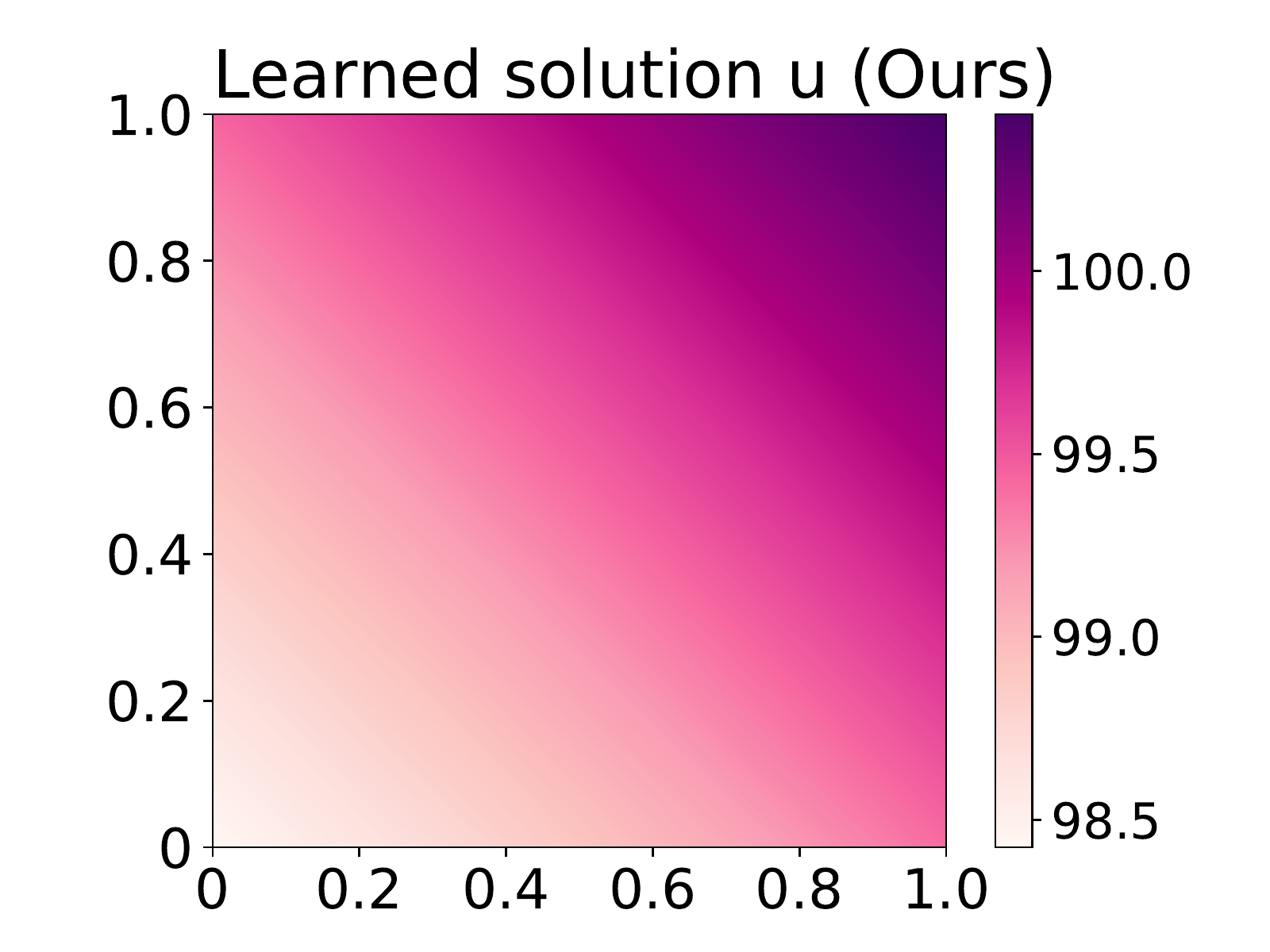}
\end{minipage}\hfill
\begin{minipage}{0.3\linewidth}
    \centering
    \includegraphics[width=1\linewidth]{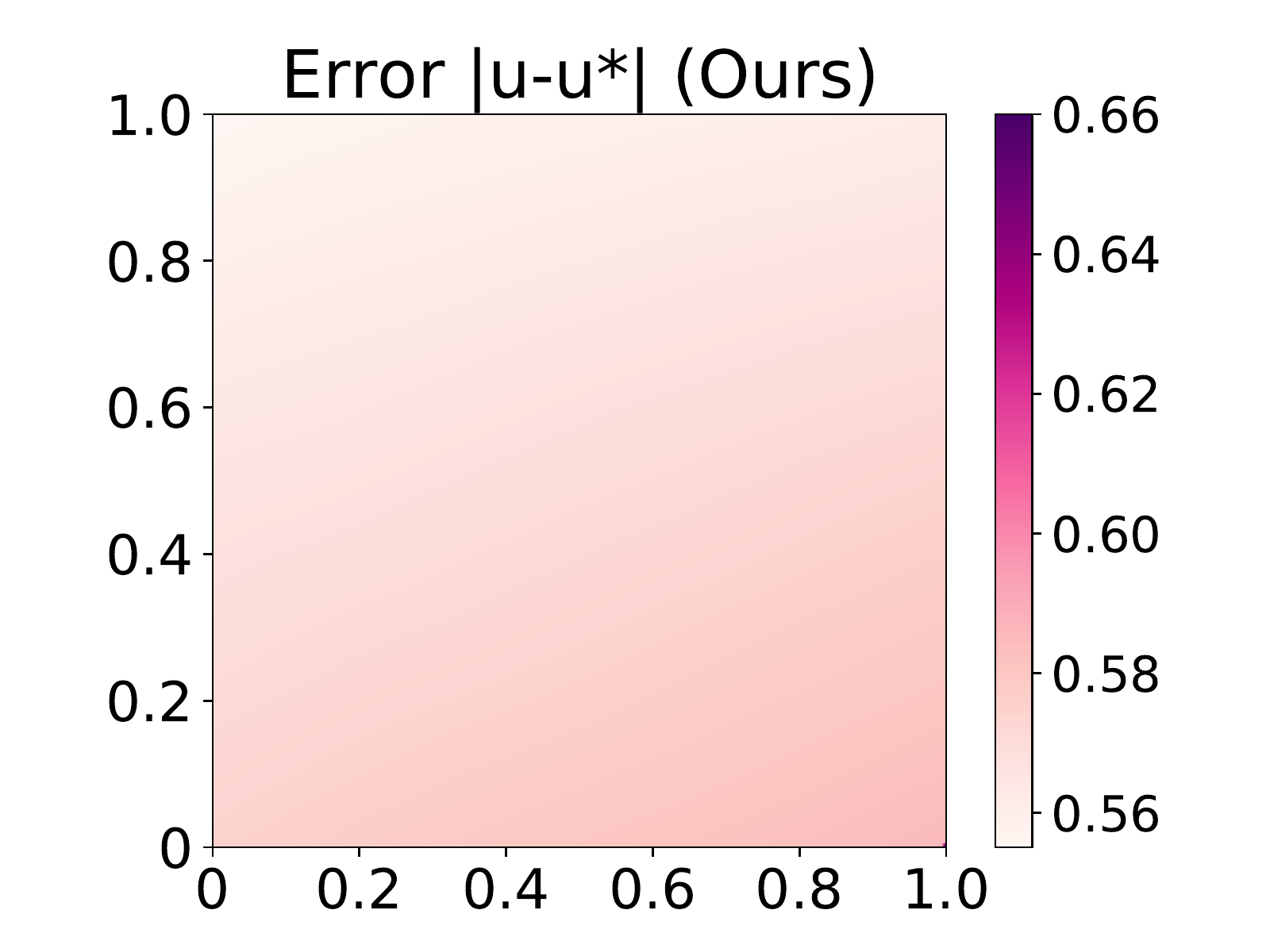}
\end{minipage}
\caption{\textbf{Visualization for the solution snapshot of Eq. (\ref{eq:hjb-exp2})}. $c$ is set to 1.5.}
\label{fig:1.5hjb}
\end{figure*}

\begin{figure*}[ht]
\begin{minipage}{0.3\linewidth}
    \centering
    \includegraphics[width=1\linewidth]{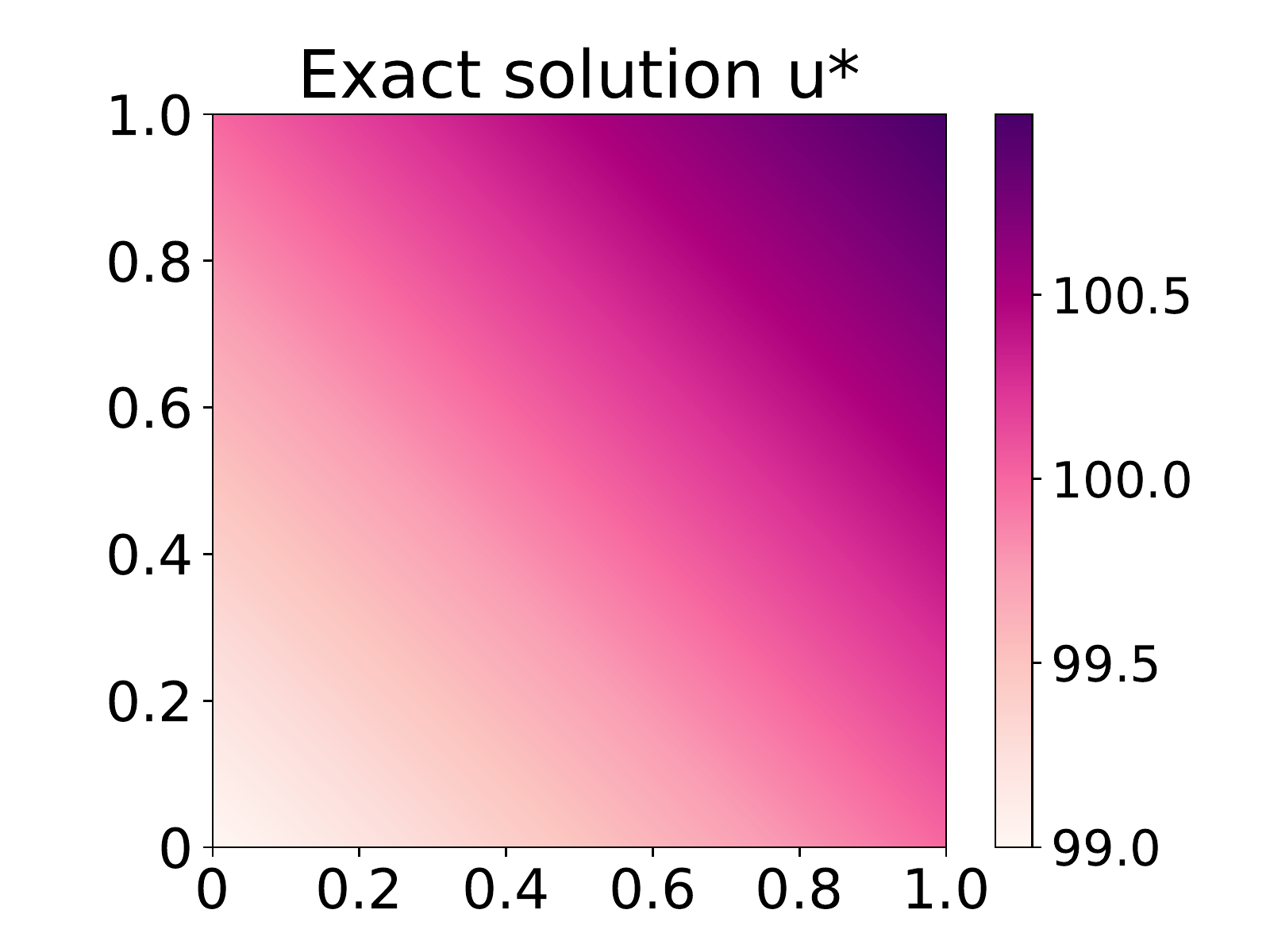}
\end{minipage}\hfill
\begin{minipage}{0.3\linewidth}
    \centering
    \includegraphics[width=1\linewidth]{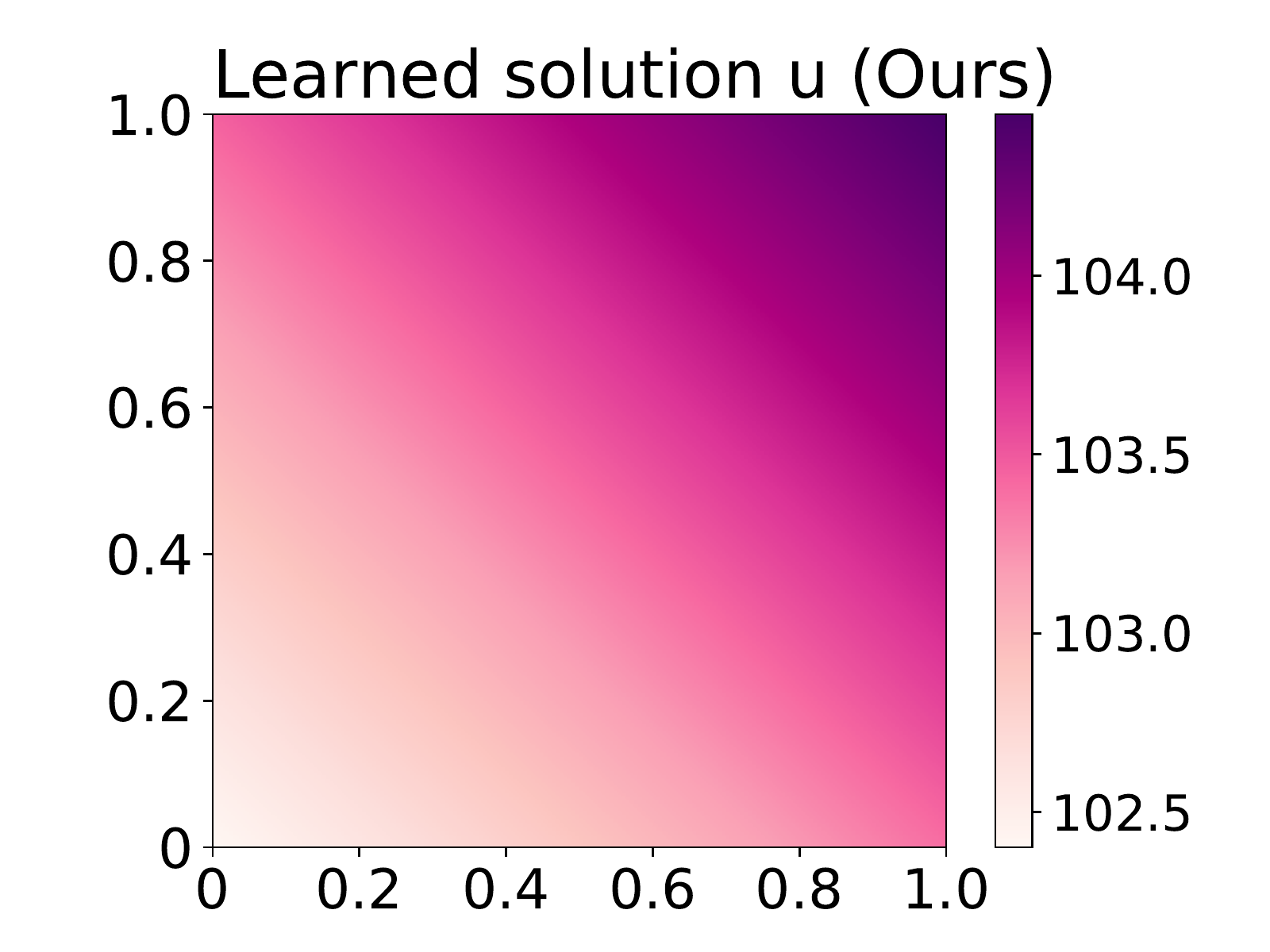}
\end{minipage}\hfill
\begin{minipage}{0.3\linewidth}
    \centering
    \includegraphics[width=1\linewidth]{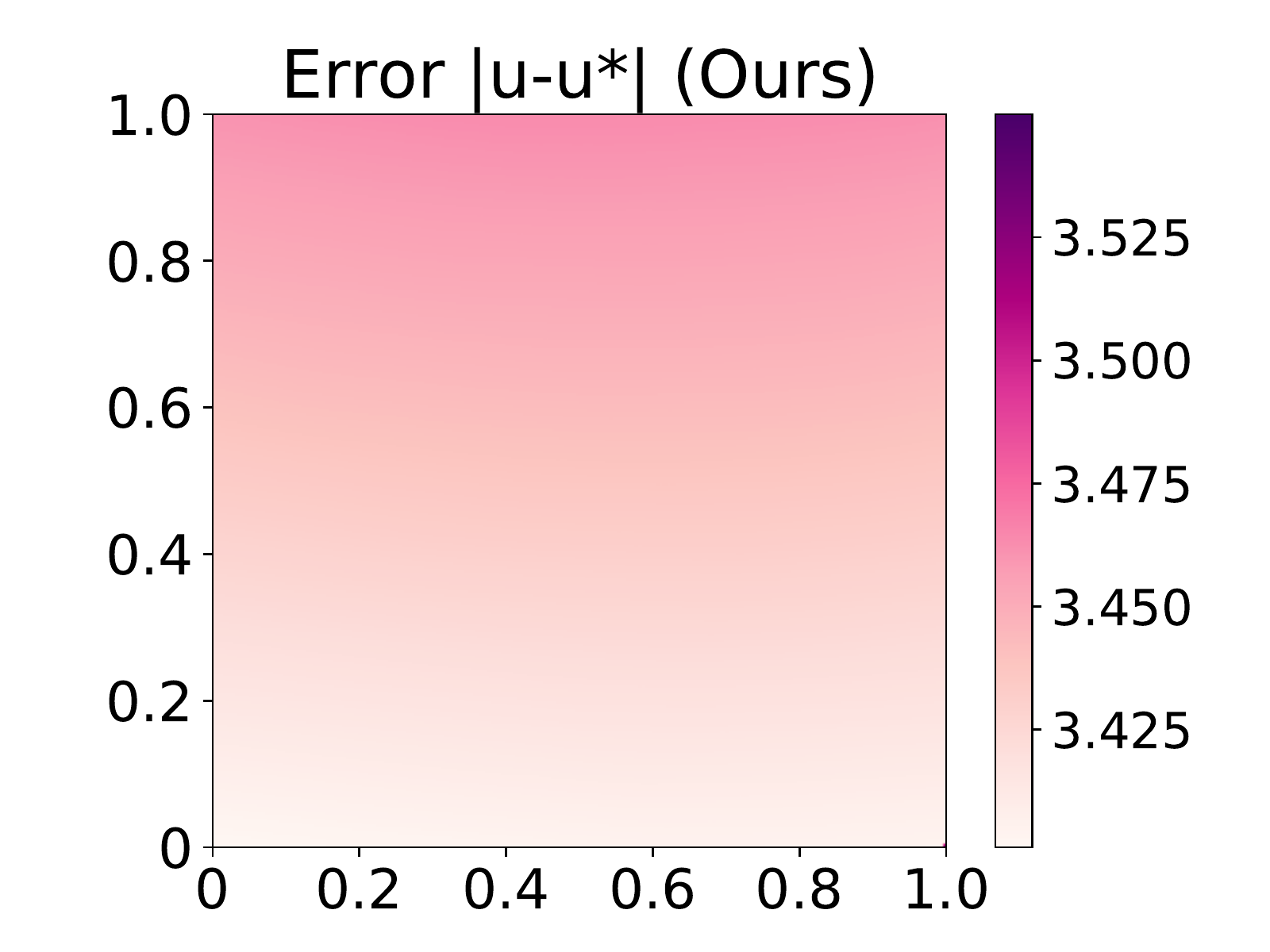}
\end{minipage}
\caption{\textbf{Visualization for the solution snapshot of Eq. (\ref{eq:hjb-exp2})}. $c$ is set to 1.75.}
\label{fig:1.75hjb}
\end{figure*}
To demonstrate the power of our method in solving general HJB Equations beyond classical LQG problems, we consider a more complicated HJB Equation as below:
\begin{equation}
\label{eq:hjb-exp2}
\begin{cases}
\displaystyle{
\partial_t u(x,t)+\Delta u(x,t)-\frac{1}{n}\sum_{i=1}^n |\partial_{x_i}u|^{c}=-2
} & (x,t)\in\sR^n\times[0,T]\\
\displaystyle{u(x,T)=\sum_{i=1}^n x_i} &x\in\sR^n
\end{cases},
\end{equation}

Eq. (\ref{eq:hjb-exp2}) has a unique solution $u(x,t)=x_1+\cdots+x_n+T-t$. We consider to solve Eq. (\ref{eq:hjb-exp2}) for different valued of $c$ using our method. We choose $c=1.25, 1.5$ and 1.75 in the experiment. 
The neural network used for training is a 5-layer MLP with 4096 neurons and $\mathrm{ReLU}$ activation in each hidden layer. 
The training recipe, including the optimizer, learning rate, batch size, and the total iterations are the same as those in Appendix \ref{app:exp-settings}. The number of inner-loop iterations $K$ is set to 5, and the inner-loop step size $\eta$ is searched from $\{2\mathrm{e}-1, 2\mathrm{e}-2, 2\mathrm{e}-3\}$.

Again, we examine the quality of the learned solution $u(x,t)$ by visualizing its snapshot on a two-dimensional space. Specifically, we consider the bivariate function $u(x_1, x_2,1, 1,\cdots,1; 0)$ and use a heatmap to show its function value given different $x_1$ and $x_2$. Figure \ref{fig:1.25hjb}-\ref{fig:1.75hjb} shows the ground truth $u^*$, the learned solutions $u$ of our method, and the point-wise absolute error $|u-u^*|$ given different values of $c$.

From the above visualization, we can see that our method can solve Eq. (\ref{eq:hjb-exp2}) for different values of $c$ effectively. Specifically, when $c=1.25$ or 1.5, the point-wise absolute error is less than 0.5 for most of the area shown in the figures. When $c=1.75$, the point-wise absolute error seems slightly larger, but it's still negligible compared with the scale of the learned solution. Thus, PINNs trained with our method fit the solution of Eq. (\ref{eq:hjb-exp2}) well, given different values of $c$.

We also compare our models with other baselines on these equations. The evaluation metric is $L^1$ relative error in the domain $[0,1]^n\times[0,1]$. The results are shown in Table \ref{tab:exp-hjb-c-apdx}. It's clear that our models outperform all the baselines on all these equations, showing the efficacy of our approach.

\begin{table}[tb]
\caption{\textbf{Experimental results of solving the high dimensional HJB equations.} $c$ is the parameter in Eq. (\ref{eq:hjb-exp2}). The dimensionality $n$ is 100. Performances are measured by the $L^1$ relative error in the domain $[0,1]^n\times[0,T]$. Best performances are indicated in \textbf{bold}.}

\label{tab:exp-hjb-c-apdx}
\centering
\begin{tabular}{cccc}\toprule
{Method} & {$c=1.25$}  &{$c=1.5$} & {$c=1.75$} \\
                        \midrule
Original PINN \cite{raissi2019physics}     & 1.11\%    & 3.82\%  &2.73\%   \\ 
Adaptive time sampling \cite{wight2020solving}  & 1.18\%	& 2.34\% & 7.94\% \\
Learning rate annealing  \cite{wang2021understanding}   & 0.98\%  &  1.13\%   &  1.06\%   \\
Curriculum regularization \cite{krishnapriyan2021characterizing}   &   6.27\% & 0.37\% &   3.51\%  \\ 
\midrule
Adversarial training (ours)  & \textbf{0.61\%}	& \textbf{0.15\%}  & \textbf{0.29\%}	\\ \bottomrule     
\end{tabular}
\end{table}

\subsection{Tracing loss and error during the training}

To give a more comprehensive comparison between original PINN and our method, we trace the loss and error during the training. 

\begin{table}[h]
\caption{\textbf{Error/loss-vs-time result of original PINN for Eq. (\ref{eq:lqg})}.}
\label{err-time-pinn}
\centering
\begin{tabular}{cccccc}\toprule
{\textbf{Iteration}}& {\textbf{1000}} &{\textbf{2000}}&{\textbf{3000}}&{\textbf{4000}}&{\textbf{5000}}\\\midrule
$L^2$ Loss      & 0.098 & 0.088 & 0.070 & 0.584 & 0.041 \\
$L^1$ Relative Error      & 6.18\% & 5.36\% & 3.86\% & 3.94\% & 3.47\% \\
$W^{1,1}$ Relative Error   & 17.53\% & 17.67\% & 14.83\% & 14.40\% & 11.31\% \\  \bottomrule      
\end{tabular}
\end{table}

\begin{table}[h]
\caption{\textbf{{Error/loss-vs-time result of our method for Eq. (\ref{eq:lqg})}.} }
\label{err-time-ours}
\centering
\begin{tabular}{cccccc}\toprule
{\textbf{Iteration}}& {\textbf{1000}} &{\textbf{2000}}&{\textbf{3000}}&{\textbf{4000}}&{\textbf{5000}}\\\midrule
$L^{\infty}$ Loss       & 11.841 & 9.352 & 2.404 & 1.605 & 0.711 \\
$L^1$ Relative Error    & 15.22\% & 4.26\% & 0.97\% & 1.10\% & 0.27\% \\
$W^{1,1}$ Relative Error   & 21.91\% & 18.62\% & 5.14\% & 4.96\% & 2.22\% \\  \bottomrule      
\end{tabular}
\end{table}

It is clear that for the original PINN approach, the $L^2$ loss drops very quickly during training, while its $W^{1,1}$ relative error remains high. This result indicates the optimization is successful in this experiment, and that the stability property of the PDE leads to the high test error. By contrast, our proposed training approach enables the test error goes down steadily during training, which aligns with the theoretical claims.

\section{Discussions on training with $L^p$ loss}
\label{app:supp-table2}
As is shown in the left panel of Table \ref{tab:ablate}, directly optimizing $L^p$ loss with large ${p}$ fails to achieve a good approximator. This might seem to contradict our theoretical analysis in section \ref{sec:theory}. However, there is actually \textbf{no contradiction between our theorems and empirical results}. Theorem \ref{thm:stb0} focuses on the approximation ability, which indicates that if we have a model whose $L^p$ loss is small, it will approximate the true solution well. The empirical results in Table \ref{tab:ablate} demonstrate the optimization difficulty of learning such a model. 

Intuitively, we randomly sample points in each training iteration in the domain/boundary to calculate the loss. When $p$ is large, most sampled points will hardly contribute to the loss, which leads to inefficiency and makes the training hard to converge. In Algorithm \ref{alg:main}, we adversarially learn the points with large loss values, making all of them contribute to the model update (Step 8), significantly improving the model training. 

Technically, directly applying Monte Carlo to compute $L^p$ loss in experiments will lead to large variance estimations. For a function $f$,
\begin{equation*}
    \int |f|^p \mathrm{d}x=\frac 1 N \sum_{i=1}^N |f(X_i)|^p+\mathcal O\left(\sqrt{\frac{\mathrm{Var} |f(X)|^p}{N}}\right),
\end{equation*}
where $\{X_i\}_{i=1}^N$ are i.i.d. samplings in the domain.

Thus, $||f||_p$ suffers from an $\mathcal O((\mathrm{Var} |f(X)|^p/N)^{1/2p})$ error.

As $p\to\infty, \mathrm{Var} |f(X)|^p\sim ||f||_{\infty}^{2p}$. Therefore, the errors for estimating both Eq.(\ref{eq:pinn-loss},\ref{eq:pinn-loss2}) and the $L^p$ norm of the residual are very large when $p$ is large.

\end{document}